\documentclass[runningheads]{llncs}
\usepackage[T1]{fontenc}
\usepackage{xcolor,colortbl}
\usepackage[prefix=]{xcolor-material}
\usepackage{todonotes}

\newcommand{\Land}{\bigwedge}

\newcommand{\pmaps}{\hookrightarrow}

\newcommand{\booleans}{\mathbb{B}}
\newcommand{\reals}{\mathbb{R}}

\newcommand{\naturals}{\mathbb{N}}

\newcommand{\xcount}{\ensuremath{\mathsf{XCount}}}

\newcommand{\exactcount}{\ensuremath{\mathsf{ExactCount}}}

\newcommand{\bitmask}{\mathsf{BitMasks}}
\newcommand{\bitmaskgen}{\mathsf{BitMaskGen}}
\newcommand{\glbmask}{\mathsf{GlobalMaskSet}}
\newcommand{\probs}{\mathsf{Probs}}
\newcommand{\thresh}{\mathsf{Thresh}}
\newcommand{\EncodeADD}{\mathsf{EncodeADD}}
\newcommand{\ToBDD}{\mathsf{ToBDD}}
\newcommand{\ganak}{\mathsf{Ganak}}
\newcommand{\approxmc}{\mathsf{ApproxMC}}
\newcommand{\xbase}{\mathsf{ADD\text{-}baseline}}
\newcommand{\sol}{\mathsf{Sol}}
\newcommand{\dist}{\mathsf{Dist}}

\newcommand{\ToADD}{\mathsf{ToADD}}

\newcommand{\DiffBDD}{\varphi}
\newcommand{\DiffADD}{\textsc{DiffADD}}
\newcommand{\DiffSum}{\textsc{DiffSum}}

\usepackage{amsmath,amssymb}
\usepackage{algorithm}
\usepackage{algpseudocode}

\algtext*{EndIf}
\algtext*{EndFor}
\algtext*{EndWhile}
\algtext*{EndRepeat}

\usepackage{graphicx}
\usepackage{subcaption}
\usepackage{forest}

\usepackage{float}

\usepackage{placeins}

\usepackage{caption}
\usepackage{graphicx}
\usepackage{booktabs}
\usepackage{multirow}
\usepackage{wrapfig}
\begin{document}
%
\title{Quantifying Sensitivity for Tree Ensembles: A symbolic and compositional approach}
\author{Ajinkya Naik\inst{1}\thanks{Corresponding author.} \and Chaitanya Garg\inst{1} \and S. Akshay\inst{1} \and Ashutosh Gupta\inst{1} \and Kuldeep S. Meel\inst{2}\thanks{The last three authors' names are diplayed in alphabetical order.  Part of this work was done while Meel was visiting IIT Bombay. 
}}
\institute{Indian Institute of Technology Bombay, India \and University of Toronto, Canada}

%
\titlerunning{Quantifying Sensitivity for Tree Ensembles}
%
%
\authorrunning{A. Naik, C. Garg, S. Akshay, A. Gupta, K.S. Meel}
%
%
\maketitle              

\begin{abstract}
Decision tree ensembles (DTE) are a popular model for a wide range of AI classification tasks, used in multiple safety critical domains, and hence verifying properties on these models has been an active topic of study over the last decade. 
One such verification question is the problem of sensitivity, which asks, given a DTE, whether a small change in subset of features can lead  to misclassification of the input. 

In this work, our focus is to build a quantitative notion of sensitivity, tailored to DTEs, by discretizing the input space of the model and enumerating the regions which are susceptible to sensitivity. We propose a novel algorithmic technique that can perform this computation efficiently, within a certified error and confidence bound. Our approach is based on encoding the problem as an algebraic decision diagram (ADD), and further splitting it into subproblems that can be solved efficiently and make the computation compositional and scalable. We evaluate the performance of our technique over benchmarks of varying size in terms of number of trees and depth, comparing it against the performance of model counters over the same problem encoding. Experimental results show that our tool $\xcount$ achieves significant speedup over other approaches and can scale well with the increasing sizes of the ensembles.


\end{abstract}
\section{Introduction}
\label{sec:introduction}
Artificial Intelligence (AI)  models are now ubiquitous in their use in almost all walks of life. While complex models such as transformers, variants of neural networks, and the like, now scale to millions of variables, a major concern has been the difficulty in interpretability and verifiability of such models. At the other end of the spectrum, decision trees are classical models which are known to be interpretable, but are not really scalable. A middle path that has proven to be effective is to train collections of decision trees, called Decision Tree Ensembles (DTEs), that tread the fine line between expressivity and scalability. Since the seminal paper of~\cite{xgboost}, introduced the XGBOOST training algorithm that has since been used, adapted, analyzed and improved over the last decade in different applications and settings~\cite{survey2023}. Today different training algorithms for DTEs including XGBOOST, random forests~\cite{sklearn}, LGBMs~\cite{lgbm} etc are routinely used in many applications~\cite{telecom,WaterResources,health1}, especially in the financial sector~\cite{banking,bank2} due to the ease of modeling tabular data using DTEs. 

Given its many applications, a rich line of work has considered formal verfication of different properties over DTEs~\cite{kantchelian,Calzavara2022,veritas2021,chen,ChenZS0BH19,veritas_multi,WangZCBH20,Ignatiev22,Ignatiev25}. These properties include, evasion, local robustness, sensitivity, fairness, explainability and more. A recent work in~\cite{ahmad2025sensitivity} considered the problem of sensitivity, which asks if a given DTE is sensitive to a set of features, i.e., does there exist a pair of inputs (called a sensitive pair) such that they differ only on sensitive features (by potentially a small margin), but their outputs are classified differently (with a high confidence). Sensitivity is closely related to individual fairness or causal discrimination~\cite{fairness,dwork2012fairness,KDD25}, since any such sensitive pair is a witness to unfairness, when sensitive features are seen as protected features. This problem is quite different from the well-studied local robustness problem, which asks if a model is robust around a given instance. The main difference is that while local robustness can be solved by searching around the neighborhood of a given input, sensitivity (much like the harder global robustness problem), requires searching across the entire input space (of a dimension that is related to the number of non-sensitive features). In~\cite{ahmad2025sensitivity}, the authors show NP-hardness of the problem and use a pseudo-Boolean encoding showing that their resulting tool could check and produce sensitive pairs for reasonably large DTE benchmarks. In~\cite{varshney2026data}, an optimized MILP encoding was proposed for solving this problem, that not only extends to larger and multi-class tree ensembles, but also finds sensitive pairs that are close to the data-distribution. However, a shortcoming in both these works, was that they only check presence of a single sensitive pair, which does not necessarily capture sensitivity of the model. Indeed, what is needed is to {\em quantify the sensitivity} of a DTE, ie., to count the number of the sensitive pairs (or estimate the volume of the sensitive region) in a DTE. In fact, it would be even more useful to count the number of input regions for which we witness sensitivity of the model, i.e., existence of sensitive pairs, one inside this region and another outside, as this captures a quantitative measure of sensitivity of the model. This is the problem that we address in this paper.

Quantitative verification problems are evidently harder to tackle than plain checking existence or satisfiability. Even for the simple setting of Boolean formulae, this is complexity-theoretically captured by the difference between $\# P$-hardness vs NP-hardness, with real difference being seen in tools; SAT-solvers routinely scale to millions of variables, while model counters (that count the number of satisfying assignments) are much more conservative. Most exact model counters, including the state-of-the-art d4~\cite{LMB24} and Ganak~\cite{SM2025,SRSM19} work by knowledge compilation, i.e., compiling the formula into a form in which counting is easy. Approximate model counters such as ApproxMC~\cite{YM23,SGM20} often scale better at the cost of providing PAC-guarantees rather than exact guarantees. Recent work has considered quantitative notions of robustness for neural networks~\cite{DowningBultan24,fairquant25}, but to the best of our knowledge, we do not know any work that deals with quantitative versions of sensitivity or fairness in decision tree ensembles.

While moving from counting models in Boolean formulae to counting sensitive pairs or regions in DTEs, there are three questions that immediately present themselves. First, what should be the quantitative notion of sensitivity in DTEs? Second, what could be a knowledge representation form for such a quantitative version of sensitivity? Third, can we obtain algorithms at scale at the cost of getting an approximate count with formal guarantees? In this paper, we address these three questions and develop a new compositional algorithm that combines ideas from exact and approximate model counting, knowledge compilation and sensitivity verification tailored for decision tree ensembles. We implement the resulting algorithm and present experimental results to show the effectiveness of our solution. Our contributions are as follows:
\begin{enumerate}
\item We exploit the inherent structure of DTEs to formulate a discretized quantitative version of the problem of counting sensitive regions. We can then use Algebraic decision diagrams (ADDs)~\cite{bahar1997algebric}, a quantitative extension of Binary decision diagrams (BDDs)~\cite{bryant1986graph} as a natural model to obtain the exact count. However, compiling a DTE to ADD is rather cumbersome and does not lead to efficient algorithms. 
\item As our main algorithmic contribution, we develop a novel compositional technique, that subdivides the problem into subproblems for each of which we can obtain the count separately and in parallel. However, these counts may overlap, and the critical difficulty we have to overcome is to obtain a correct total count from overlapping subproblem counts. Our key contribution is to design a splitting and merging strategy that gives an approximate count with strong theoretical soundness guarantees.
  To do this, we adapt a technique to approximate the volume of the union of sets in the streaming model from database theory~\cite{MVC21,KuldeepIJCAI23}.
\item We implement our algorithm and perform experiments on a large suite of 3194 benchmark instances coming from 10 different datasets, with trees ranging from 10 to 100, depth varying from 3 to 6, with at least 8 features. We compare our approach against the ADD-monolithic baseline, the state-of-the-art tools for exact and approximate model counting. Our experimental results show that outperform the other approaches by an order of magnitude, while maintaining errors within the prescribed margin.
\item We also provide an application and validation of our counting approach and model, by showing that it can be used to check effectiveness of regularization techniques in DTEs. More precisely, we validate that when a given DTE is regularized, the count of sensitivity decreases. 
\end{enumerate}

Our paper is structured as follows. In Section~\ref{sec:preliminaries}, we start with some preliminaries followed by the problem statement and example in Section~\ref{sec:problem}. Our algorithm is presented in Section~\ref{sec:algorithm} and we present our experimental results in Section~\ref{sec:experiments}. Finally, we provide a case-study regarding utility of counting sensitivity in DTE which also helps establish validity of our results in Section~\ref{sec:motivation} followed by a brief conclusion in Section~\ref{sec:conclusion}. Several additional details and further experiments are given in the Appendix.
\medskip

\noindent{\it Related Work.} For neural networks, the question of global robustness and individual fairness was dealt with recently in~\cite{athavale24}, where the authors consider Deep Neural Networks (DNNs) and focus on a linear approximation of the softmax, extending a state-of-the-art local robustness neural network verifier to handle global fairness. An approach using an SMT-based approach was developed in~\cite{fairify23}, where the authors exploit the fact that several neurons in the network remain inactive to prune the network and make it amenable for formal verification. Neither of these approaches consider quantitative notions. In a very recent work~\cite{fairquant25} the authors quantify the degree of fairness in DNNs by computing the percentages of inputs whose classification outputs can be certified as fair or falsified as unfair. But their method is incomplete as it uses counter-example guided refinement and stops after some number of iterations declaring the remaining inputs as undecided. Sound but incomplete approaches for arbitrary ML models were also considered in~\cite{UAI20} for linear and kernelized polynomial models using semi-definite programming for verification. Different from all these works, our approach is deeply embedded in the DTE formalism, uses ADD based compilation and probabilistic reasoning to remain sound and complete with Probably Approximately Correct (PAC) guarantees.

As mentioned earlier, for DTEs, the literature has mainly focussed on local robustness verification, with different techniques such as abstract-interpretation~\cite{Ranzato_Zanella_2020}, dynamic programming based~\cite{WangZCBH20} and clique-based approaches~\cite{veritas2021,veritas_multi}. For sensitivity and individual fairness~\cite{ahmad2025sensitivity} uses a pseudo-Boolean encoding, while ~\cite{varshney2026data} uses an MILP encoding.
 However, none of these approaches consider quantitative notions of sensitivity and hence are rather different from our work. 

Finally, normal forms including BDDs, ADDs and variants of decomposable Negation normal forms have often been explored for model counting in the literature~e.g.,\cite{DBLP:conf/aaai/DudekPV20,ColnetIJCAI24}. Recent work for DNF model counting~\cite{KuldeepIJCAI23} built on ideas of approximating the volume of union of sets from~\cite{MVC21} while also using compositionality. Some model counting based ideas were exploited for group fairness verification in~\cite{GBM22}, but these do not lift to sensitivity.

\section{Preliminaries and Notation}
\label{sec:preliminaries}
In this section, we establish the notations used throughout the paper. We denote the set of real numbers by $\reals$, natural numbers by $\naturals$ and Booleans \{0,1\} by $\booleans$. Vectors are represented by bold lowercase letters (e.g., $\mathbf{x}$), while their individual components are indicated by subscripts (e.g., $x_i$). The set of satisfying assignments of a Boolean formula $\varphi$ is denoted by $\sol$($\varphi$), with its cardinality $|\sol$($\varphi)|$. We consider the standard supervised learning setting where an input vector $\mathbf{x}$ is drawn from a feature space $\mathcal{D} \subseteq \reals^m$ and mapped to an output $y \in \reals$ by a learned function $h: \mathcal{D} \to \reals$.
\medskip

\noindent{\bf Decision Tree Ensembles.}
Let $F = \{1, \dots, k\}$ be a set of $k$ feature indices. An input instance is as a vector $\mathbf{x} = (x_1, \dots, x_k)$, where each $x_f$ corresponds to the value of feature $f \in F$. A \textit{Decision Tree Ensemble} (DTE) is a predictive model composed of a collection of $n$ decision trees, denoted by $\mathcal{T} = \{T_1, \dots, T_n\}$~\cite{breiman2001random}. Each individual tree $T_i: \mathcal{D} \to \reals$ is a rooted binary structure where:
\begin{itemize}
    \item Each \textbf{internal node} $u$ performs a split based on a guard predicate $g(\mathbf{x}) \equiv (x_f \le \theta)$, where $f \in F$ is the splitting feature and $\theta \in \reals$ is a threshold.
    \item Each \textbf{leaf node} $l$ stores a constant prediction value $v_l \in \reals$.
\end{itemize}

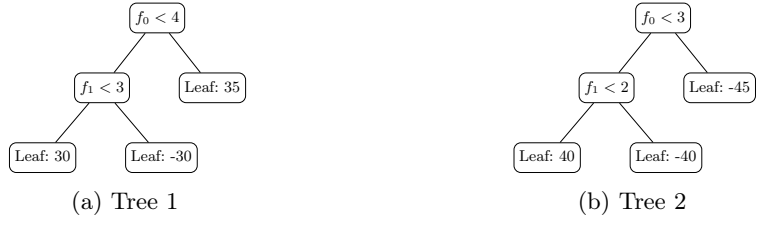
\begin{figure}[t]
\centering
\begin{subfigure}{0.45\textwidth}
\centering
\scalebox{0.65}{
\begin{forest}
for tree={
    draw,
    rounded corners,
    align=center,
    s sep=10mm,
    l sep=8mm
}
[$f_0 < 4$
    [$f_1 < 3$
        [Leaf: 30]
        [Leaf: -30]
    ]
    [Leaf: 35]
]
\end{forest}}
\caption{Tree 1}
\end{subfigure}
\hfill
\begin{subfigure}{0.45\textwidth}
\centering
\scalebox{0.65}{
\begin{forest}
for tree={
    draw,
    rounded corners,
    align=center,
    s sep=10mm,
    l sep=8mm
}
[$f_0 < 3$
    [$f_1 < 2$
        [Leaf: 40]
        [Leaf: -40]
    ]
    [Leaf: -45]
]
\end{forest}}
\caption{Tree 2}
\end{subfigure}
\caption{Two-tree ensemble represented as decision trees.}
\label{fig:example}
\end{figure}

Figure~\ref{fig:example} illustrates a simple ensemble of two decision trees. Each internal node contains a guard predicate (in this case there are 4 i.e $f_0 < 3$, $f_0 < 4$, $f_1 < 2$, $f_1 < 3$) that directs the traversal based on the input features, while each leaf node provides a constant output value (six leaf nodes with values 30, -30, 35, 40, -40, -45). Decision trees are trained on tabular data and are widely used in various applications, including classification and regression tasks. They are particularly favored for their interpretability and ability to capture complex feature interactions.
The output of a single tree, $T(\mathbf{x})$, is determined by traversing the tree from the root to a leaf according to the guard predicates. Usually a single decision tree is not sufficient to achieve high performance on tasks such as regression or classification, hence the need for ensembles in which the output of multiple trees is combined. There are a number of different ways of aggregating the outputs of individual trees. In this work, we focus on additive ensembles, such as Gradient Boosted Trees (GBTs)~\cite{lgbm}, where the final ensemble prediction $\mathcal{T}(\mathbf{x})$ is the sum of the individual tree predictions: 
 $   \mathcal{T}(\mathbf{x}) = \sum_{i=1}^n T_i(\mathbf{x})$.
Without any modification, our analysis method will work for random forest ensembles~\cite{breiman2001random} as well, where the final prediction is obtained by averaging the outputs of individual trees.
\medskip

\noindent{\bf Model Sensitivity.}
Beyond predictive accuracy, we are concerned with the changes in model's decisions based on changes to certain number of feature values of the input. We define our problem through the lens of \textit{individual fairness}, a principle requiring that a model yields similar outputs for similar individuals~\cite{dwork2012fairness}. 
To formalize this, we partition the feature set $F$ into two disjoint subsets:
\begin{itemize}
    \item
    \textbf{Sensitive Features ($S \subseteq F$)}, which model protected attributes such as race, gender, or age, and 
    \item
    \textbf{Non-Sensitive Features ($\bar{S} = F \setminus S$)}, which include all other attributes, such as income or education level. 
\end{itemize}

For instance, in the ensemble shown in Figure~\ref{fig:example}, if we consider $f_0$ as the sensitive feature and $f_1$ as the non-sensitive feature, then $S = \{f_0\}$ and $\bar{S} = \{f_1\}$.
For an input vector $\mathbf{x}$, we use $\mathbf{x}_{\bar{S}}$ to denote the components of $\mathbf{x}$ corresponding to non-sensitive features, and similarly $\mathbf{x}_S$ for sensitive features.

\begin{definition}[Sensitivity]
    An ensemble $\mathcal{T}$ is \textbf{sensitive} with respect to a subset $S$ of features $F$ if there is a pair of inputs $\mathbf{x}, \mathbf{x}' \in \mathcal{D}$, such that $\dist(\mathbf{x}_{\bar{S}} ,\mathbf{x}'_{\bar{S}})=0$ and $\dist(\mathbf{x}, \mathbf{x}') \le \delta$, the following holds:
 $       |\mathcal{T}(\mathbf{x}) - \mathcal{T}(\mathbf{x}')| > \epsilon$, 
    where $\epsilon, \delta \in \reals$ are user-defined tolerance thresholds, and $\dist$ is a chosen distance metric.
\end{definition}

\noindent{\bf Symbolic Representations: Algebraic Decision Diagrams.}
Our framework relies on \textit{Algebraic Decision Diagrams} (ADDs)~\cite{bahar1997algebric}, which generalize Binary Decision Diagrams (BDDs)~\cite{bryant1986graph} to represent functions $f: \{0,1\}^n \to \mathbb{R}$. 
Let $X=\{x_1,\ldots,x_n\}$ be a finite set of Boolean variables and let
$D$ be a finite set of terminal values, real numbers in our case.
An \emph{Algebraic Decision Diagram} (ADD)
over $X$ and $D$ is a rooted directed acyclic graph
$A=(V,E,r,\mathsf{var},\mathsf{low},\mathsf{high},\mathsf{val})$,
where $V$ is a finite set of nodes, $E$ is a finite set of edges and $r\in V$ is the root, each non-terminal node $v$ is labeled by a Boolean variable
    $\mathsf{var}(v)\in X$ and has exactly two outgoing edges:
    a low edge $\mathsf{low}(v)$, followed when $\mathsf{var}(v)=0$,
    and a high edge $\mathsf{high}(v)$, followed when $\mathsf{var}(v)=1$; each terminal node $t$ is labeled by a value
    $\mathsf{val}(t)\in D$; and along every root-to-terminal path, variables appear according
    to a fixed variable ordering. 
The ADD represents a function
$f_A : \{0,1\}^n \to D$.
For an assignment $\alpha\in\{0,1\}^n$, the value $f_A(\alpha)$ is obtained
by starting at the root and recursively following the low edge whenever
the current variable evaluates to $0$, and the high edge whenever it
evaluates to $1$, until a terminal node is reached. The label of that
terminal node is the value of $f_A(\alpha)$.

The core advantage of ADDs in our context is the efficient \textsc{Apply} operator, which performs pointwise arithmetic on the symbolic structures. We construct the symbolic representation of the entire ensemble, $\mathcal{T}_{\text{ADD}}$, by summing the ADDs of individual trees ($T_{i, \text{ADD}}$) directly: $\mathcal{T}_{\text{ADD}} = \sum_{i=1}^M T_{i, \text{ADD}}$.
\medskip

\noindent{\bf Notion of Distance.} In this work, we will use Hamming distance as the measure of similarity between a pair of inputs. Hamming distance for a pair of vectors of equal length is defined as the number of features for which the two of them differ:
 $   HammingDist(\mathbf{x},\mathbf{x'}) = \sum_{i=1}^k \mathbf{1}[x_i \neq x'_i].$
%
As the inputs to the models are fixed length vectors, choosing Hamming distance as the metric is natural.


\section{Problem Statement}
\label{sec:problem}
In this section, we develop a quantitative notion of sensitivity that captures how much of the input space is sensitive. Decision tree ensembles, by their design, split the input space of the model into regions such that the output of the model for all points belonging to a region is the same. This natural property of these models can be used to ask questions of a quantitative variety as the input space is discretized into a finite number of regions. This property motivates our approach towards quantifying sensitivity over the discrete space of input regions. 

As the model under focus is a decision tree model, we have a set of guard predicates which can be considered essentially as a set of feature-threshold pairs. For a feature $f \in F$, let ordered set $G_f$ be the list of sorted thresholds $\theta_1, \theta_2, \ldots, \theta_{m_f}$ occurring in the guards of the decision nodes of the ensemble involving feature $f$, where $m_f = |G_f|$. Further, for feature $f$, let $B_f$ be the ordered set $(b_{f,i})_{i=1}^{m_f}$ of Boolean variables. Also, let $b_{f,i}$ represent the decision $x_f < \theta_i$.
The Boolean variables corresponding to feature $f$ must satisfy the {monotonicity constraint}:
$  b_{f,\theta_1} \Rightarrow b_{f,\theta_2} \Rightarrow \dots \Rightarrow b_{f,\theta_{m_f}}$.
  This  ensures that Boolean assignments correspond to valid numeric orderings.
    
    Let $B$ be the ordered set of all Boolean variables corresponding to all features in $F$. Formally, $B = (B_1, B_2, \dots , B_k)$. 
    Let $ M = |B|=\sum_{f \in F}m_f$. Thus, any 
    $\mathbf{b} \in \booleans^{M}$ is an $M$ dimensional vector $(\mathbf{b}_{f,i})_{i \in [1,m_f], f \in F}$. Let $\mathbf{b}_{S}$, $\mathbf{b}_{F/S}$ denote $\mathbf{b}\!\upharpoonright_{i \in [1,m_f], f \in S}$ and 
    $\mathbf{b}\!\upharpoonright_{i \in [1,m_f], f \in F/S}$ respectively, i.e, $\mathbf{b}$ projected over a Boolean variables belonging to sensitive and non-sensitive features respectively.

\begin{definition}
    We define a mapping function $\phi : \reals^{k} \pmaps \booleans^{M}$ that maps a input vector $\mathbf{x}$ (consisting of $|F|$ feature values) to
    an output Boolean vector $\mathbf{b}$ (consisting of $M$ Boolean variables) as follows 

\[
\phi(\mathbf{x})_{f,i}=
\begin{cases}
1, & \text{if } x_f < \theta_i,\\
0, & \text{otherwise.}
\end{cases}
\]


 We also define a valuation map $V : \booleans^{M} \to \reals$ be defined as: 
$  V(\mathbf{b}) =  \mathcal{T}(\mathbf{x}) $ for some $\mathbf{x}$ such that $\mathbf{b} = \phi(\mathbf{x})$. Note that $V$ is well-defined.  
\end{definition}

From the definition of sensitivity, we want to check if our model is sensitive to a subset of features $S$ i.e perturbation in the input with respect to these features, all else being the same, leads to a significant difference in the model output.
To operationalize this, we consider pairs of inputs $\mathbf{x}$ and $\mathbf{x}'$ and their corresponding Boolean encodings $\mathbf{b} = \phi(\mathbf{x})$ and $\mathbf{b}' = \phi(\mathbf{x}')$. The following three conditions must be satisfied for $\mathbf{x}$ and $\mathbf{x}'$ to witness a sensitivity violation:

\paragraph{(1) Equality of non-sensitive features.}
All Boolean
variables of non-sensitive features must be identical between $\mathbf{b}$ and $\mathbf{b'}$:
\begin{equation}
\label{cond:equality}
\Land_{i \in [0,m_f],f\in F\setminus S} \mathbf{b}_{f,{i}} = \mathbf{b}'_{f,{i}} \tag{C1}
\end{equation}

\paragraph{(2) Sensitive similarity.}
The sensitive features must differ as quantified by a Hamming
distance threshold $d$:
\begin{equation}
\label{cond:distance}
\mathrm{HammingDist}(\mathbf{b}, \mathbf{b}') \le d \tag{C2}
\end{equation}

\paragraph{(3) Prediction-gap condition.}
A fairness violation occurs when the ensemble’s predictions for $x$ and $x'$
differ by at least a user-specified margin $G$:
\begin{equation}
\label{cond:gap}
|V(\mathbf{b}) - V(\mathbf{b}')| > G \tag{C3}
\end{equation}


Having defined the discretized notion of sensitivity above, we lift it to capture a quantitative version that measures sensitivity in the input space. As each input $\textbf{x}$ is mapped to a Boolean vector $\textbf{b}$, we have essentially divided the input space into a finite set of regions, each corresponding to a Boolean vector. Input data points which belong to the same region have the same output when evaluated over the model. Hence, a region can be considered as the discrete object over which we can define a notion of quantitative sensitivity. We propose an approach where we want to enumerate all regions for which there exists at least one other region such that this pair satisfies the notion of sensitivity defined as above. Therefore, our goal is to enumerate the number of distinct Boolean vectors $\textbf{b}$ that admit at least one counterpart $\textbf{b}'$ satisfying all three conditions above. 

\begin{definition}
Given a DTE $\mathcal{T}$ over a set of input features $F$, with $S \subseteq F$, $d \in \naturals$ and $G \in \reals$ the \textbf{counting objective} can be defined as :

\[
\boxed{
\begin{aligned}
C(S,d,G) = 
\#\Big\{
\mathbf{b} \in \{0,1\}^{M}
\;\Big|\;
\exists\, \mathbf{b}' \in \{0,1\}^{M} :
\big[
&\Land_{i \in [0,m_f],f\in F\setminus S} \mathbf{b}_{f,{i}} = \mathbf{b}'_{f,{i}} \\
&\Land \mathrm{HammingDist}(\mathbf{b}, \mathbf{b}') \le d,\\
&\Land |V(\mathbf{b}) - V(\mathbf{b}')| > G
\big]
\Big\}.
\end{aligned}
}
\]

The quantity $C$ counts the number of Boolean encodings (hence input regions)
for which there exists at least one assignment (subject to conditions above) such that there is a significant
difference in model output for the two assignments. This quantity serves as a symbolic
measure of fairness.

\end{definition}

Since every Boolean vector $\mathbf{b}$ uniquely identifies a discrete region of the input space, we use the terms interchangeably. We can then define a \textbf{sensitive region} as any region $\mathbf{b}$ that admits a witness $\mathbf{b}'$ satisfying the constraints established in the counting objective. Thus, our primary objective is to compute the number of sensitive regions. 

Consider again the two-tree ensemble as presented earlier in Figure~\ref{fig:example}. This model has 
two features $f_0$ and $f_1$. Let $f_0$ be the sensitive feature.
There are a total of four guards - $f_0 < 3$, $f_0 < 4$, $f_1 < 2$,
$f_1 < 3$. These can be represented by four Boolean variables - 
$b_{f_0,3}$, $b_{f_0,4}$, $b_{f_1,2}$, $b_{f_1,3}$. As $f_0 < 3$ being 
true implies $f_0 < 4$, the Boolean variables $b_{f_0,3}$, $b_{f_0,4}$ must
satisfy the constraint $b_{f_0,3} \Rightarrow b_{f_0,4}$. Likewise for 
Boolean variables for feature $f_1$ i.e $b_{f_1,2} \Rightarrow b_{f_1,3}$.
So, essentially the Boolean assignments that variables ($b_{f_0,3}$, $b_{f_0,4}$) and ($b_{f_1,2}$, $b_{f_1,3}$) can take are - \{(0,0),(0,1),(1,1)\}. Therefore, over the 4 Boolean variables there are a total 3*3 = 9 assignments possible. Any input $x$ to the model maps to one these 9 assignments, and each assignment maps to an output $\mathcal{T}(x)$. \\
The objective is to count the number of assignments from the set of all possible assignments (9 in this case) for which there exists another corresponding assignment such that they differ by at most $d$ Boolean variables belonging to sensitive feature, all other Boolean variables being equal, and their outputs differ by a given gap $G$ . For instance, for this example let's assume $d = 1$ and $G = 80$. In this example, out of 9, there are two assignments ($b_{f_0,3}=0$, $b_{f_0,4}=1$, $b_{f_1,2}=1$, $b_{f_1,3}=1$) and
($b_{f_0,3}=1$, $b_{f_0,4}=1$, $b_{f_1,2}=1$, $b_{f_1,3}=1$), which differ in the bit $b_{f_0,3}$, and have outputs -15 and 70 respectively, and hence are witnesses of each other for satisfying the sensitivity condition stated above. Hence $C(S=\{f_0\},d=1,G=80)$ for this example amounts to 2, out of 9 assignments.

\medskip

\begin{algorithm}[!htbp]
\caption*{$\exactcount$($\mathcal{T},S,d,G$)}
\label{alg:addcount}
\begin{algorithmic}[1]

\State $A_{\mathrm{sum}} \gets \mathbf{0}$

\ForAll{$T_i \in \mathcal{T}$}
    \State $A_i \gets \ToADD(T_i)$
    \State $A_{\mathrm{sum}} \gets A_{\mathrm{sum}} + A_i$
\EndFor

\State $A_1 \gets A_{\mathrm{sum}}$ , $A_2 \gets A_{\mathrm{sum}}$

\ForAll{$f \in S$}
    \State $A_2 \gets \mathsf{ReplaceBits}(A_2, \mathrm{Bits}(f))$ 
\EndFor

\State $\Delta \gets A_2 - A_1$

\State $\Phi_d \gets \mathsf{AtMostDistance}(\mathrm{Bits}(f), \mathrm{Bits}(f'), d)$
\Comment{BDD constraint allowing at most $d$ bit differences}

\State $B_G \gets \ToBDD(\Delta, G)$
\Comment{Map each leaf $v$ to $1$ iff $v>G$}

\State $B_G \gets B_G \land \Phi_d$

\State $C \gets |\sol(B_G)|$

\State \Return $C$

\end{algorithmic}
\end{algorithm}

\noindent{\bf ADD based approach for exact count.} 
An ADD based approach for computing the exact count $C$ is outlined in Algorithm $\exactcount$
For the computation of $C$, we can exploit the tree-like 
structure of the ensemble. This can be done by encoding the trees 
of the ensemble as ADDs (in line 2-4 of the Algorithm), so that operations over trees can be efficiently done 
over a data structure that captures the tree-like properties of our 
input. Each non-leaf node of the ADD forms an ADD variable/bit,
representing the decision node of the tree ensemble. This conversion is achieved through the $\ToADD$ function. The 
complete ensemble can be represented by taking the sum of all ADDs, 
we refer to the resultant ADD as $A_{\mathrm{sum}}$
Post computation of $A_{\mathrm{sum}}$, the monotonicity constraints between 
bits of each feature must be enforced. Post this computation, 
two copies of $A_{\mathrm{sum}}$ can be created ($A_1$,$A_2$), followed by replacing ADD variables i.e bits of features $f \subseteq S$
in one of the copies with their corresponding copies (let the copies be represented by the set $\mathrm{Bits}(f')$), referred to in the algorithm by $\mathsf{ReplaceBits}$ function (line 7 of the algorithm), taking the 
difference of the two copies, and applying the constraint that only upto
$d$ bits of the $f \subseteq S$ can differ (represented by $\mathsf{AtMostDistance}$, line 9 of the algorithm) . After these operations, 
the resultant object that we get is a ADD that represents the 
difference between two outputs of the ensemble when a subset of 
features is allowed to change by a certain amount. The leaf nodes of this ADD 
represents all the possible difference values that can be obtained 
between the two models. Then we convert this ADD to a BDD with the following rule -
leaf nodes of ADD having values greater than $G$ take value  
1 in the BDD or else 0, with internal nodes remaining the same (referred to as $\ToBDD$ function in the algorithm, in line 10 of the algorithm), followed by applying the distance constraints. Counting the satisfying assignments over the resultant BDD 
achieves our original task of counting the regions.



\section{An Approximate Counting Algorithm}
\label{sec:algorithm}


In this section, we describe the primary technical contributions of our
work : a procedure $\xcount$ that efficiently counts the 
number of regions in the input space of a given decision tree 
ensemble that are sensitive with respect to given feature set, bit 
distance and output threshold. We first give a technical 
overview of the overall algorithm and then describe each of its
components in detail.

\subsection{Overview}
The input to our algorithm is a decision tree ensemble 
$\mathcal{T}=\{T_i\}_{i=1}^{n}$ having $n$ binary trees over
a feature space $F$ of $k$ features. The goal is to compute $C(S,d,G)$ i.e enumerate the 
count of the regions of input space that are sensitive for a given
feature set $S \subseteq F$, given a maximum perturbation 
of $d$ bits and a difference margin of $G$. Encoding the trees of the 
decision tree ensemble as algebraic decision diagrams (ADDs) is central to 
our approach for computing the count in an efficient way.
\medskip

\noindent{\it Subproblem division and approximate counting.}
We gave an outline of an ADD-based approach in the previous section. The problem with that approach is that as the number of 
trees in the ensemble increase or the trees of greater depth 
are present, the computation of sum-ADD overwhelms memory 
and is much slower to compute. Therefore, we adopt a divide-and-conquer 
approach - as the change is restricted to the variables of the sensitive features 
(by a given bit-distance), we break the 
computation into  subproblems, where each subproblem is a pair 
of bit vectors, with each bit vector being an assignment over 
variables of sensitive features, and the two differing in $d$ bits.
Moreover, as these bits are fixed for each subproblem, certain 
branches of the ADDs become unreachable, hence the computations can be 
performed on a pruned version of ADDs. Furthermore, instead of exact counting 
over each subproblem, we adopt an approximate counting approach where we compute an estimate
$\widehat{C}$ of the exact count $C(S,d,G)$ by taking 
the count of the union of all solutions over all subproblems within a 
($\epsilon, \delta$) guarantee where $\widehat{C}$ we get 
is within a factor of ($1\pm\epsilon$) of the actual count with at least ($1-\delta$) confidence. 
\medskip

\noindent{\it Application of approximate counting.}
To illustrate how the probabilistic soundness guarantee with ($\epsilon, \delta$) values applies in practice, consider a financial institution evaluating a credit scoring model for deployment. A regulator may require that the model exhibit sensitivity on at most, say, $1\%$ of the input space. Our method returns an estimate together with a guarantee: with probability at least $1-\delta$, the estimate is within a factor of $(1 \pm \varepsilon)$ of the true value.
Thus, the institution can conclude, with high confidence, that the model
satisfies the $1\%$ sensitivity threshold without performing exact, and
potentially intractable, counting. Conversely, if the estimate exceeds the
threshold even within the error margin, the model can be safely rejected.

More generally, since $\varepsilon$ and $\delta$ are user-controlled
hyperparameters, practitioners can tune the strength of the guarantee depending
on available computational resources and the application---for instance, using
stricter confidence levels for regulatory compliance and more relaxed ones
during model development.

\subsection{Description of the algorithm}
In Algorithm~\ref{alg:xcount} we lay out the main framework $\xcount$ along with the procedures used by it. 
Each decision node of the tree is of the form ($x_f < \theta$). These nodes form a set of predicates over which we can 
define a set of boolean variables.
We also maintain  the count of guards for each feature, denoted by $m_f$. Out of these, the variables belonging to the sensitive feature set $S$ is the one 
over which we want to check if changing $d$ bits from them leads to 
a swing in the output of the ensemble by $G$. These bits are subject to monotonicity constraints, hence if there $n$ number of bits of a particular   feature, then there are $(n+1)$ boolean evaluations possible. A set of these evaluations, referred to as bit-masks ($mask$ in the algorithm), in the form of bit-vectors is computed 
in advance by $\bitmaskgen$ (Algorithm~\ref{alg:bitmaskgen}) and stored in $\bitmask$.
If more than one feature given as sensitive, then we take 
a cross-product of the bit-masks of all such features - by taking all permutations of bit-masks over all features and appending them, and storing all such bit-masks 
as a global set. These operations are carried out in line 1-4 of Algorithm~\ref{alg:xcount}.
\begin{algorithm}[tb]
\caption{$\xcount$($\mathcal{T},S,d,G,\epsilon,\delta$)}
\label{alg:xcount}
\begin{algorithmic}[1]
\State $n \gets |\mathcal{T}|$, $k \gets |F|$, $\bitmask \gets \emptyset$, $\probs \gets \emptyset$
\For{$f \in S$}
  \State $\bitmask \gets (f,\bitmaskgen(m_f))$
\EndFor
\State $\glbmask \gets \prod_{f \in S} \bitmask(f)$
\For{$mask_1, mask_2 \in \glbmask$}
  \If{$\mathsf{HammingDist}(mask_1,mask_2) \leq d$}
    \State $\probs \gets (mask_1,mask_2)$
  \EndIf
\EndFor
\State $X \gets \emptyset$, $p \gets 1$, $\thresh \gets 
\max\!\left(
12 \cdot \frac{\ln(24/\delta)}{\epsilon^2},
\, 6\bigl(\ln \tfrac{6}{\delta} + \ln |\mathsf{Probs}|\bigr)
\right)$
\For{$pair \in \mathsf{Probs}$}
  \State $\DiffBDD \gets \mathsf{ProcessSubproblem}(\mathcal{T},pair,G)$
  \State $t \gets |\sol(\DiffBDD)|$
  \For{$\mathbf{b} \in X$}
    \If{$\mathbf{b}_{F/S} \models \DiffBDD$}
      \If{$\mathbf{b}_{S}=pair.mask_1 \lor \mathbf{b}_{S}=pair.mask_2$}
        \State remove $\mathbf{b}$ from $X$
      \EndIf
    \EndIf
  \EndFor
  \State $N_k \gets \mathsf{Poisson}(t.p)$
  \While{$N_k + |X| \geq \thresh$}
    \State Remove each assignment from $X$ with 0.5 probability
    \State $p \gets p/2$
    \State $N_k \gets \mathsf{Binomial}(N_k,0.5)$
  \EndWhile
  \State $Samples \gets \mathsf{SampleBDD}(\DiffBDD,N_k,pair)$
  \State $X.\mathsf{Append}$($Samples$)
\EndFor
\State \Return $\widehat{C} \gets |X| / p$
\end{algorithmic}
\end{algorithm}

\begin{algorithm}[htb]
\caption{$\mathsf{ProcessSubproblem}$($\mathcal{T},p,G$)}
\label{alg:processsubp}
\begin{algorithmic}[1]
\State $\mathcal{T'} \gets \emptyset$,$\mathcal{T''} \gets \emptyset$
\For{$i=1$ to $|\mathcal{T}|$}
    \State $\mathcal{T'}.\mathsf{Append}(\mathsf{PruneTree}(T_i.root,p.mask_1))$
    \State $\mathcal{T''}.\mathsf{Append}(\mathsf{PruneTree}(T_i.root,p.mask_2))$
\EndFor
\State $\DiffSum \gets 0$
\label{line:diffsum_init}
\For{$i=1$ to $|\mathcal{T}|$}
  \State $\DiffADD \gets \EncodeADD(T^{'}_i.root) - \EncodeADD(T^{''}_i.root)$
  \label{line:diffadd}
  \State $\DiffSum \gets \DiffSum + \DiffADD$
  \label{line:diffsum_update}
\EndFor
\State $\DiffBDD \gets \ToBDD(\DiffSum,G)$
\State \Return $\DiffBDD$
\end{algorithmic}
\end{algorithm}

\medskip

\noindent{\bf Breaking-up the problem.}
In line 5-10 of Algorithm~\ref{alg:xcount}, we break the original problem into subproblems. A pair of bitmasks from the set $\bitmask$, which lie within a Hamming  distance of $d$ are the pair of assignments of interest to the sensitivity problem. Hence any such tuple forms a subproblem, and all such tuples form the set $\probs$. Each subproblem is processed by the function $\mathsf{ProcessSubproblem}$. The function $\mathsf{PruneTree}$ uses the bit-mask input to prune the branches of the trees of the ensemble which are rendered unreachable due to fixing the nodes according to the values of bit-mask. Procedure $\mathsf{EncodeADD}$ encodes each tree as an ADD, by coverting each tree node into an ADD node and each leaf value to terminal node (leaf) of the ADD . We take a pair-wise difference of  each tree in the ensemble to get $\DiffADD$s that encodes the difference and add all such ADDs to get $\DiffSum$ ADD, which encodes the sum of all differences, essentially encoding the difference between the output of the two ensembles. This holds true as the difference between two sums over the ADDs is equal to the sum of the pair-wise differences between the corressponding ADDs of the two ensembles. The procedure $\mathsf{ToBDD}$ (Algorithm~\ref{alg:tobdd}) converts this ADD to a BDD by keeping the internal nodes same and values of the leaves into boolean values based on whether they exceed $G$. This resultant  BDD, labeled as $\DiffBDD$, is the object over which  the approximate satisfying assignments are computed. Exact implementation of the functions above are provided in Appendix~\ref{app:algo}.
\medskip

\noindent{\bf Merging the subproblems.}
In line 11-24 of Algorithm~\ref{alg:xcount},
computes the union of set of satisfying assignments over all subproblems, each being a BDD, with a ($\epsilon,\delta$) guarantee. The threshold parameter $\thresh$ is pre-computed as:
\[
      \thresh = \max\!\left(
        \frac{12 \ln (24/\delta)}{\varepsilon^2},
        ~6\bigl(\ln(6/\delta) + \ln |\mathsf{Probs}|\bigr)
      \right)
    \] 
The set of satisfying assignments is initially empty before we 
start iterating through the subproblems. Each assignment is over 
$M$ bits which together form a bit-vector $\mathbf{b} \in \booleans^M$. The constituent bits $\mathbf{b}_{f,i}$ of $\mathbf{b}$ can be divided into two disjoint sets $\mathbf{b}_S$ and $\mathbf{b}_{F/S}$, as defined in Section~\ref{sec:problem}. The $\DiffBDD$ we get is formed over nodes belonging to non-sensitive features, as we had fixed the nodes belonging to sensitive features when it was formed. So essentially a satisfying assignment of the BDD paired with either of the bits $pair.mask_1$ or $pair.mask_2$ forms the complete satisfying assignment. At the start of processing each subproblem, we first remove assignments from $X$ that satisfy the current BDD and whose values over bits of sensitive features also match either of the current bit-mask pair. Then we compute the number of solutions $N_k$ that is to be sampled from from the $\DiffBDD$, whose distribution is simulated by Poisson distribution. At any point of time, we want to store $\thresh$ number of solutions in $X$, if $N_k + |X|$ is larger than $\thresh$, then $p$ is halved and $N_k$ is adjusted by 
samping from Binomial($N_k$,0.5) and removing each solution from $X$ with 0.5 probability. After adjusting, we have to sample $N_k$ solutions of $\DiffBDD$ uniformly at random. This is done by the sub-routine $\mathsf{SampleBDD}$ (Algorithm~\ref{alg:samplebdd}) which samples assignments from a given BDD in two steps, as follows.

For each node $u$ in the BDD, we compute $p(u)$ the probability that sub-BDD rooted at $u$
    evaluates to 1 under a uniform random assignment of variables. For terminal nodes,
    $p(0) = 0$ and $p(1) = 1$. For non-terminal nodes with children
    $u_{\text{then}}$ and $u_{\text{else}}$, we have $
      p(u)=\frac{p(u_{\text{then}})+p(u_{\text{else}})}{2}$.
    To draw a sample, we start at the root of the BDD and repeatedly
    choose the next edge (\emph{then} or \emph{else}) using the
    probabilities above. This determines truth assignments for all
    variables encountered along the path. Variables that do not appear
    on the path are assigned uniformly at random. The result is a full
    assignment $x$.
Each assignment is paired with a bitmask choice and inserted into a
$X$. If $p$ denotes
the effective sampling probability, we obtain the final estimate as $
  \widehat{C} = \frac{|X|}{p}.$

A full running example of our algorithm is presented in Appendix~\ref{app:example_run}.

\subsection{Theoretical Guarantees}
\label{sec:xcount-correctness}

We now establish the correctness of {\xcount}, formalized as follows:

\begin{theorem}\label{thm:xcount-correct}
	$\Pr\left[(1-\epsilon)C(S,d,G) \le \widehat{C} \le (1+\epsilon)C(S,d,G)\right] 
	\ge 1-\delta$, where $\widehat{C}$ is the output of {\xcount}.
\end{theorem}

\begin{proof}
	We adapt the proof technique from~\cite{KuldeepIJCAI23} and provide the main derivation below, with more details in Appendix~\ref{app:xcproof}. Let $m = |\probs|$ be the number of subproblems. For each $i\in[m]$, let 
	$U_i \subseteq \{0,1\}^F$ denote the set of assignments from the $i$-th 
	subproblem, and define $S_i = \bigcup_{t=1}^i U_t$ so that 
	$|S_m| = C(S,d,G)$.
	
	For $j\ge 0$, let $p_j = 2^{-j}$ and let $X^{(i)}_j$ denote the random 
	multiset after processing $i$ subproblems with sampling probability $p_j$. 
	Since the algorithm draws $N_t \sim \mathrm{Poisson}(|U_t|p_j)$ samples from 
	each $U_t$, each element $s\in S_i$ appears independently with 
	$\mathrm{Poisson}(p_j)$ multiplicity, giving 
	$|X^{(i)}_j| \sim \mathrm{Poisson}(|S_i|p_j)$.
	
	Let $E^{(i)}_j$ denote the event that the algorithm has sampling probability 
	$p_j$ after $i$ subproblems, and define 
	$A^{(i)}_j = \{|X^{(i)}_j| \not\in [|S_i|p_j(1-\epsilon), 
	|S_i|p_j(1+\epsilon)]\}$. Let $j^\star$ be the smallest $j$ with 
	$p_j < \thresh/(4|S_m|)$. Since $\widehat{C} = |X^{(m)}_j|/p_j$ on event 
	$E^{(m)}_j$, we have
	\begin{align*}
		\Pr[\widehat{C}\not\in[(1-\epsilon)C(S,d,G),(1+\epsilon)C(S,d,G)]]
		\le \sum_{j=0}^{j^\star-1}\Pr[A^{(m)}_j] 
		+ \Pr\!\left[\bigcup_{j\ge j^\star}E^{(m)}_j\right].
	\end{align*}
	
	For $j<j^\star$, we have $|X^{(m)}_j| \sim \mathrm{Poisson}(\lambda_j)$ with 
	$\lambda_j = |S_m|p_j \ge \thresh/4$. By standard Poisson concentration, 
	$\Pr[|X-\lambda|\ge \epsilon\lambda] \le 2\exp(-\epsilon^2\lambda/3)$ for 
	$X\sim\mathrm{Poisson}(\lambda)$. Thus 
	$\Pr[A^{(m)}_j] \le 2\exp(-\epsilon^2\thresh/12)$, and 
	$\sum_{j=0}^{j^\star-1}\Pr[A^{(m)}_j] \le 2j^\star 
	\exp(-\epsilon^2\thresh/12) \le \delta/6$ by our choice of $\thresh$.
	
	The event $\bigcup_{j\ge j^\star}E^{(m)}_j$ occurs only if 
	$|X^{(i)}_{j^\star-1}|\ge \thresh$ for some $i\in[m]$. Since 
	$|X^{(i)}_{j^\star-1}| \sim \mathrm{Poisson}(|S_i|p_{j^\star-1}) 
	\le \mathrm{Poisson}(\thresh/2)$, we have 
	$\Pr[|X^{(i)}_{j^\star-1}|\ge \thresh] \le \exp(-\thresh/6) \le \delta/(6m)$. 
	By union bound, $\Pr[\bigcup_{j\ge j^\star}E^{(m)}_j] \le \delta/6$.
	
	Substituting both bounds into the decomposition above yields 
	$\Pr[\widehat{C}\not\in[(1-\epsilon)C(S,d,G),(1+\epsilon)C(S,d,G)]] 
	\le \delta/6 + \delta/6 < \delta$, which establishes the theorem.\qed
\end{proof}


\section{Experiments}
\label{sec:experiments}
We now evaluate the performance and accuracy of our proposed method, $\xcount$, on a set of decision tree ensemble models trained on tabular datasets. 

\noindent{\bf Implementation details.} We implemented our tool $\xcount$ in C++, using the CUDD package for the implementation of ADDs and BDDs. The tool takes following command inputs - a decision tree ensemble in the form of a JSON file, the input feature(s) of the ensemble to test for sensitivity ($S$), the precision that is required of the leaf values of the ensemble, the gap threshold ($G$), the bit distance allowed between sensitive features ($d$), the confidence parameter ($\delta$) and error parameter ($\epsilon$).
The output of the tool is the count estimate $\widehat{C}$, within a factor of ($1\pm\epsilon$) with a confidence of (1-$\delta$) of the actual count $C(S,d,G)$ as outlined in section~\ref{sec:algorithm}.
\medskip

\noindent{\bf Benchmarks} We evaluated our approach on ten standard tabular datasets: \textit{Diabetes}, \textit{Adult}, \textit{Covtype}, \textit{Protein-structure}, \textit{Mnist}, \textit{Webspam}, \textit{Wine-quality}, \textit{Supersymmetry}, \textit{Higgs}, and \textit{Fashion-MNIST}. To construct a comprehensive benchmark suite, we trained multiple distinct ensemble models for each dataset. We varied the ensemble size from 10 to 100 in increments of 10 (10 sizes), and tree depths from 3 to 6 (4 depths). For each configuration of (Dataset $\times$ Ensemble Size $\times$ Depth), we generated multiple instances by sweeping over different feature thresholds and decision points naturally arising from the data. 
For each model generated, we randomly chose a set of 8 features present in to model as the features to be tested for sensitivity analysis. This leads to total benchmark instances of 10*4*8 $=$ 320 for each dataset (there are 316, 319, 319 instances for \textit{Adult}, \textit{Supersymmetry} and \textit{Protein structure}, respectively, as smaller ensemble models of these datasets have features less than 8, in such case we have picked all features present in the model without any randomization). Table~\ref{tab:setup_datasets} summarizes the characteristics of these datasets, along with the min and max number of guards for each datatset, as this captures both depth and number of trees. 

\begin{table}[h!]
    \centering
    \small
    \scalebox{0.8}{
    \begin{tabular}{lccccccc}
        \hline
        \textbf{Dataset} & \textbf{Benchmark Instances} & \textbf{Guards (Min)} & \textbf{Guards (Max)} \\ \hline
        Diabetes & 320 & 41 & 1176\\ 
        Adult & 316 & 18 & 712\\
        Covtype & 320 & 36 & 1301\\ 
        Mnist & 320 & 58 & 4597\\ 
        Webspam & 320 & 36 & 1156\\
        Wine-quality & 320 & 41 & 998\\
        Supersymmetry & 319 & 67 & 2561\\ 
        Higgs & 320 & 70 & 2559\\ 
        Fashion-MNIST & 320 & 51 & 2513\\ 
        Protein structure & 319 & 49 & 1696\\ \hline
    \end{tabular}%
    }
    \caption{Benchmark instances for each dataset}
    \label{tab:setup_datasets}
\end{table}


All experiments were conducted on a high-performance server equipped with an \textbf{AMD EPYC 7742 64-Core Processor} and {629GiB of RAM}, running \textbf{Ubuntu 22.04.5 LTS}. To ensure fair comparison and resource isolation, we enforced a 4GB memory limit per instance and a strict timeout of 1800 seconds.
\medskip

\noindent{\bf Comparisons baselines.}
To establish a comparative baseline for quantitative verification of decision tree ensembles, we developed  a module that encodes the verification problem into Conjunctive Normal Form (CNF). This encoding enables us to leverage standard,  state of the art model counting engines--specifically $\ganak$  and $\approxmc$ to solve the generated instances. We also compare the performance against an $\xbase$. Thus we compare against:
\begin{itemize}
    \item {CNF + $\ganak$} : We encode the verification problem into CNF and use $\ganak$, a projected model counter, to obtain exact counts.
    \item {CNF + $\approxmc$} : Using the same CNF encoding, we employ $\approxmc$ to obtain approximate counts with $(\epsilon, \delta)$-guarantees.
    \item $\xbase$: A naive exact counting approach where the full ensemble is encoded into a single monolithic Algebraic Decision Diagram (ADD).
\end{itemize}
For consistency, we utilized a gap threshold of 2, bit-distance of 1, $\epsilon$ to 0.1, $\delta$ to 0.1, and set the leaf value precision to 3 decimal places across all experiments.
However, for the datasets \textit{Higgs} and \textit{Supersymmetry} we used a gap threshold of 0.05, due to the leaf values of these two datasets being 
much smaller to record difference of more than 2, leading to no possible solutions in case a gap threshold of 2 is selected. The value of a meaningful gap threshold to be given as input to the model is closely related 
to the scale of values in the leaves, as it captures the absolute difference between the output of the ensemble (i.e the sum of leaves) for two different inputs. 
Therefore, for any meaningful analysis of sensitivity, the gap threshold given must be proportionate to leaf values of the model. Experimental evaluation of varying gap thresholds over the performance of $\xcount$ can be found in section~\ref{app:expts} of the appendix. 
\medskip

\noindent{\it Our results.} We present our results in four parts: (i) we discuss the results of our performance in comparison to the baseline; (ii) we check the accuracy of our approximate count; (iii)  we demonstrate scalability of the tool varying size of the ensemble, and (iv) we highlight the impact of different optimizations in our tool through ablation studies.

\noindent{\bf Performance results.} We conducted a detailed experimental evaluation to compare $\xcount$ against $\ganak$ and $\approxmc$. Table~\ref{tab:par2scores} gives the overall statistics regarding the completed number of runs by each tool and the corresponding PAR-2 scores. Figure~\ref{fig:cactus_plots_all} presents the cactus plots for the ten datasets, visualizing the number of benchmark instances solved (x-axis) within a given time limit (y-axis). The plots for each datasets can be referred to in Appendix~\ref{app:expts}.

\begin{figure}[t]
    \centering
    \setlength{\tabcolsep}{4pt}
    \begin{minipage}[t]{0.51\textwidth}
        \vspace{0pt}
        \centering
        \includegraphics[width=\linewidth]{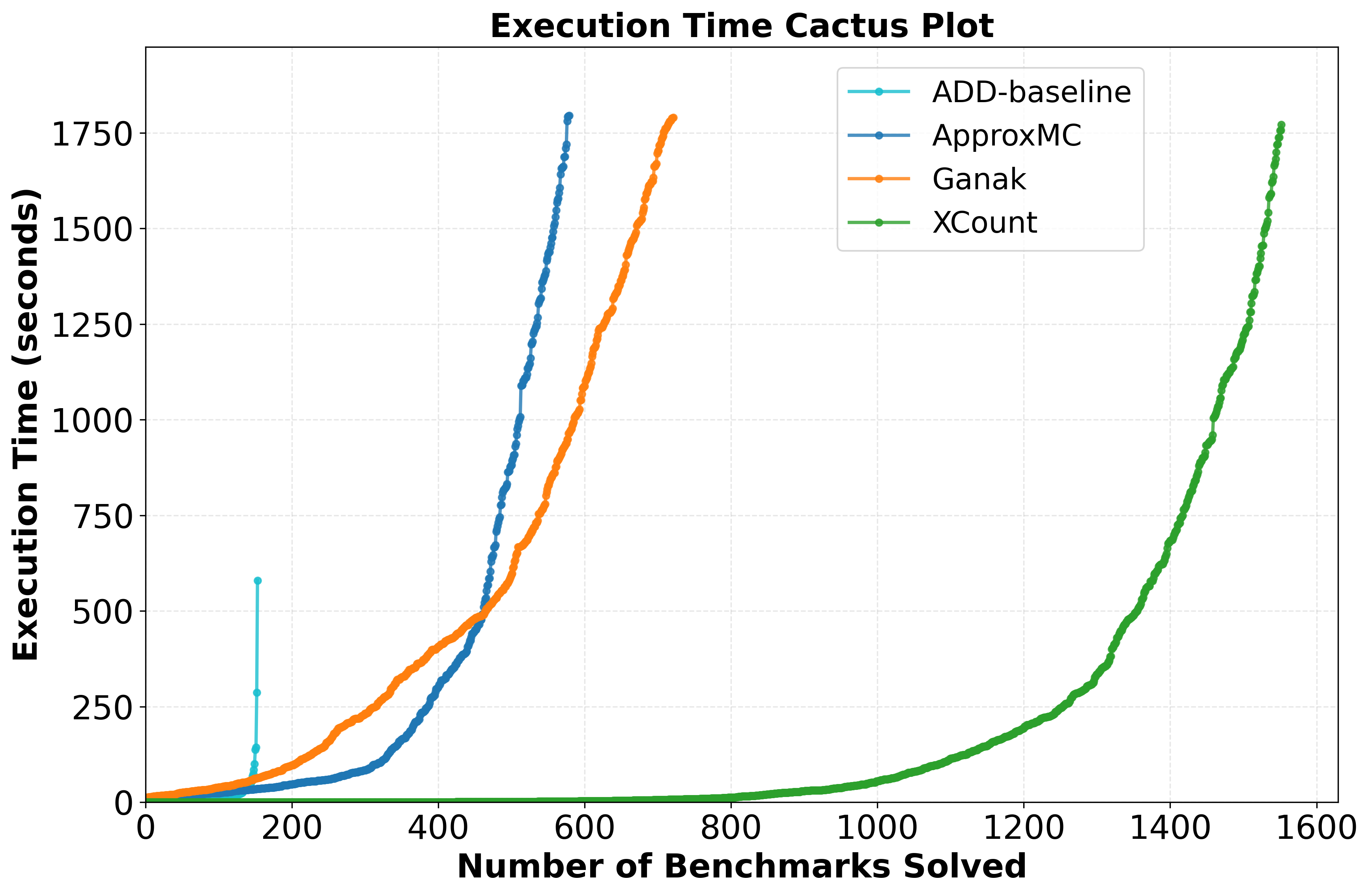}
        \caption{Cactus plots comparing execution time across all benchmarks.}
        \label{fig:cactus_plots_all}
    \end{minipage}
\quad 
    \begin{minipage}[t]{0.45\textwidth}
        \vspace{0pt}
        \centering
        \small
        \setlength{\tabcolsep}{2.5pt}
        \renewcommand{\arraystretch}{0.9}
        \begin{tabular}{lcc}
            \toprule
            \textbf{Tool} & \textbf{Solved} & \textbf{PAR2} \\ \hline
            $\xbase$ & 153 & 3428.461 \\
            $\approxmc$ & 579 & 3005.374 \\
            $\ganak$ & 721 & 2901.421 \\
            $\xcount$ & 1552 & 1937.474 \\ \hline            
        \end{tabular}
        \captionof{table}{Aggregate statistics for each configuration. \textbf{PAR2} is the Penalized Average Runtime - if instance is solved then its runtime is counted, and if it times out then twice the timeout (2*1800=3600 s) value is used.}
        \label{tab:par2scores}
    \end{minipage}
\end{figure}

As shown in Figure~\ref{fig:cactus_plots_all}, $\xcount$ demonstrates significantly higher instance completion rates compared to the generic CNF-based pipelines.  For smaller instances, $\xbase$ is faster compared to CNF + $\ganak$/$\approxmc$ pipelines but scales poorly with respect to size of the ensemble, whereas the latter is able to scale. However, $\xcount$ performs better than both of these, both in terms of timing and scalability.

{\em In summary, out of 3194 benchmark instances, $\xcount$ is able to handle 1552, about 1.15 times more than the next best approach,
as well as clocking a PAR-2 score of 1937.474, an improvement of 33\% over the next best approach and 43\% improvement over the 
ADD baseline approach.} 

\medskip

\noindent{\bf Accuracy wrt approximation guarantees.} We quantified the deviation of our approximate counts against the exact counts provided by the CNF + $\ganak$ pipeline (for instances where $\ganak$ did not time out). Figure~\ref{fig:error_plots_acc} depicts the relative error distribution. The majority of solved instances exhibit negligible error, confirming that $\xcount$ maintains high fidelity to the exact count while offering improved scalability.
\begin{figure}[h!]
    \centering
    \includegraphics[width=0.7\textwidth]{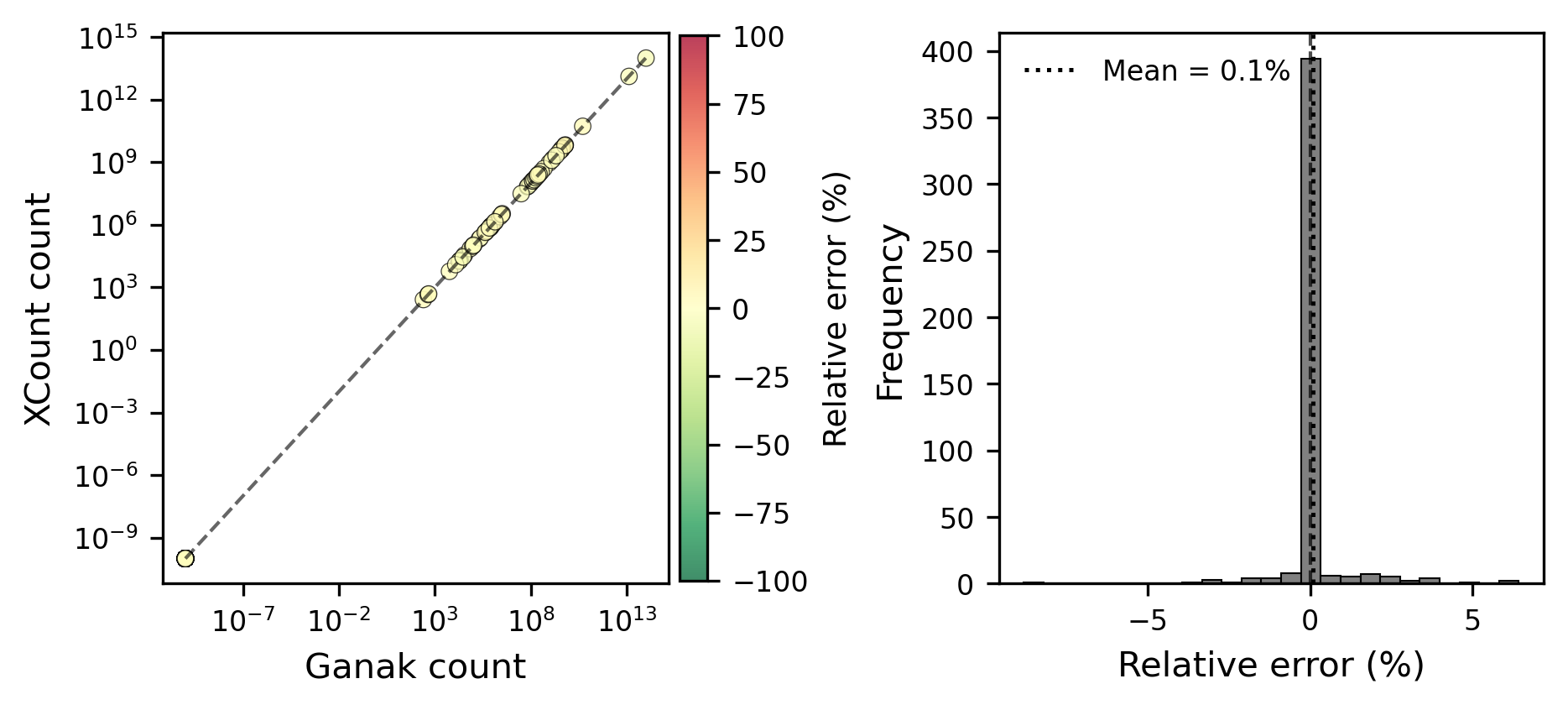}
    \caption{Relative error of $\xcount$ compared to exact counts across all benchmarks.}
    \label{fig:error_plots_acc}
\end{figure}

\noindent{\bf Scalability.} We evaluate the scalability of the tool with increasing size of the ensemble. Figure~\ref{fig:heatmap} depicts a heat map of experiments over all datasets to illustrate a relation between the  number of trees of the ensemble and the max depth of the trees of  the ensemble. The higher shades indicated the greater time taken by the instances that fall in that category. A better metric is to enumerate the total number of guards present in the ensemble. In Figure~\ref{fig:guard_scalability}, we depict a plot the ratio of completed instances to total number of instances for each range of guard values. The plot shows that as the number of guards increase, the performance of $\xcount$ goes down, with less than 50\% of the instances getting completed when guards are more than 400.

\begin{figure}[t]
    \centering
    \begin{subfigure}[t]{0.48\textwidth}
        \centering
        \includegraphics[width=\linewidth]{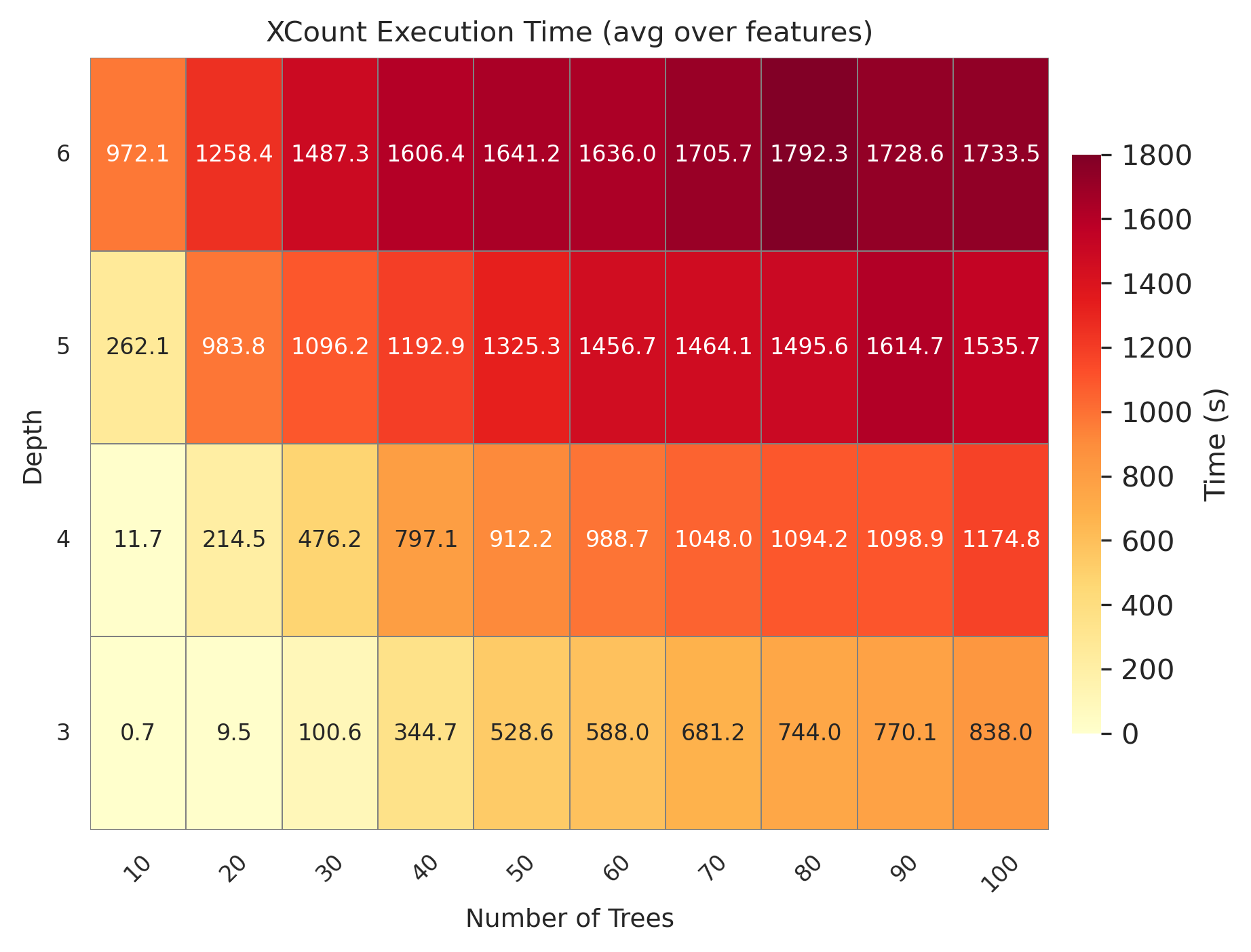}
        \caption{Heat map of trees, depth vs time across all benchmarks}
        \label{fig:heatmap}
    \end{subfigure}
    \hfill
    \begin{subfigure}[t]{0.4\textwidth}
        \centering
        \includegraphics[width=\linewidth]{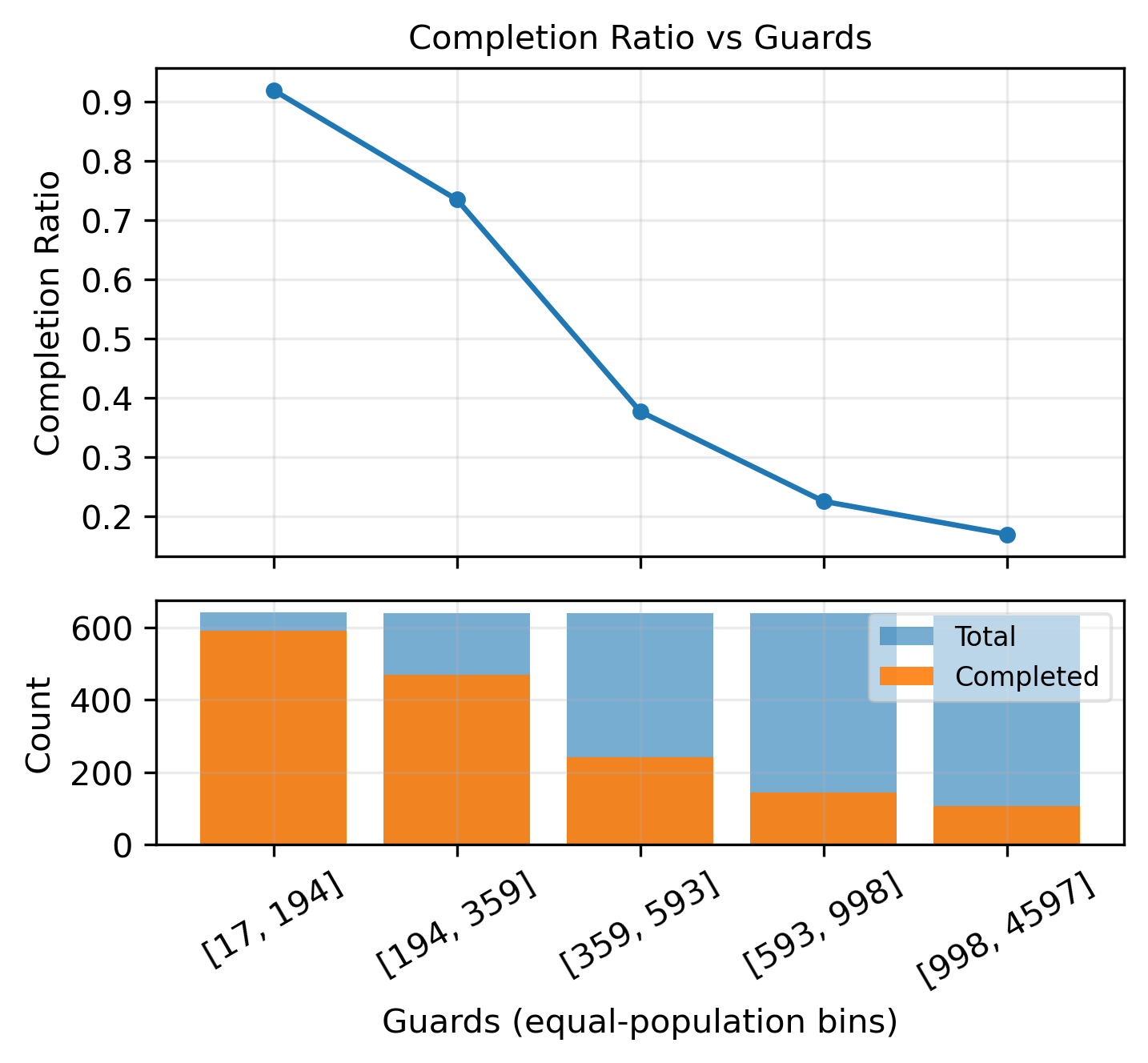}
        \caption{Completion ratio vs guard ranges}
        \label{fig:guard_scalability}
    \end{subfigure}
    \caption{Scalability analysis of $\xcount$ wrt ensemble size across all benchmarks.}
    \label{fig:scalability}
\end{figure}

\noindent{\bf Ablation Studies.} Finally, we discuss the impact of the various optimizations embedded in our tool. The first is encoding the problem as an ADD (\textbf{Opt-1}). Second, when we break the problem into subproblems, and take the union of satisfying assignments over them, one design based optimization is to take pairwise difference of individual ADDs and then add them, instead of adding them first and then taking difference (\textbf{Opt-2}). 
Finally, merging the satisfying solutions of subproblems by taking a probabilistic approach (pepin) instead of exact count, to obtain the final count estimate (\textbf{Opt-3}). To illustrate the relative impact of each of these optimizations we conduct a comprehensive set of experiments across four configurations: 
(i) Encoding the problem as a single monolithic ADD without any of the other optimizations (Opt-1 or $\xbase$).
(ii) Adding all optimizations except Opt-2 i.e first adding ADDs then taking pairwise difference (Opt 1+3).
(iii) Adding all optimizations except Opt-3 i.e doing exact counting instead of probabilistic counting (Opt 1+2).
(iv) Adding all other optimizations except Opt-1 i.e breaking the problem into subproblems and passing the CNF encoding of subproblems to $\ganak$ (Subp + CNF).
Note that for this configuration only the total time for solving individual subproblems is summed up, without calculating the final count as its performance was already worse than other configurations.

We compare the performance of these four with performance of $\xcount$.
For this set of experiments we train models on the \textit{Adult} dataset by varying the number of trees from 10 to 100 with a step size of 10 trees, 
and varying the depth from 3 to 6, leading to 10*4 = 40 tree models. For each of these models, we run the tool taking over all the features that are 
present in the model, leading to a total of 509 instances (\textit{Adult} has 14 features, but the number of instances is slightly less than 14*40 = 560 as 
not all features appear in the trained models if the models are small i.e having less number of trees). 
\begin{figure}[t]
    \centering
    \setlength{\tabcolsep}{4pt}

    \begin{minipage}[t]{0.44\textwidth}
        \vspace{0pt}
        \includegraphics[width=\linewidth]{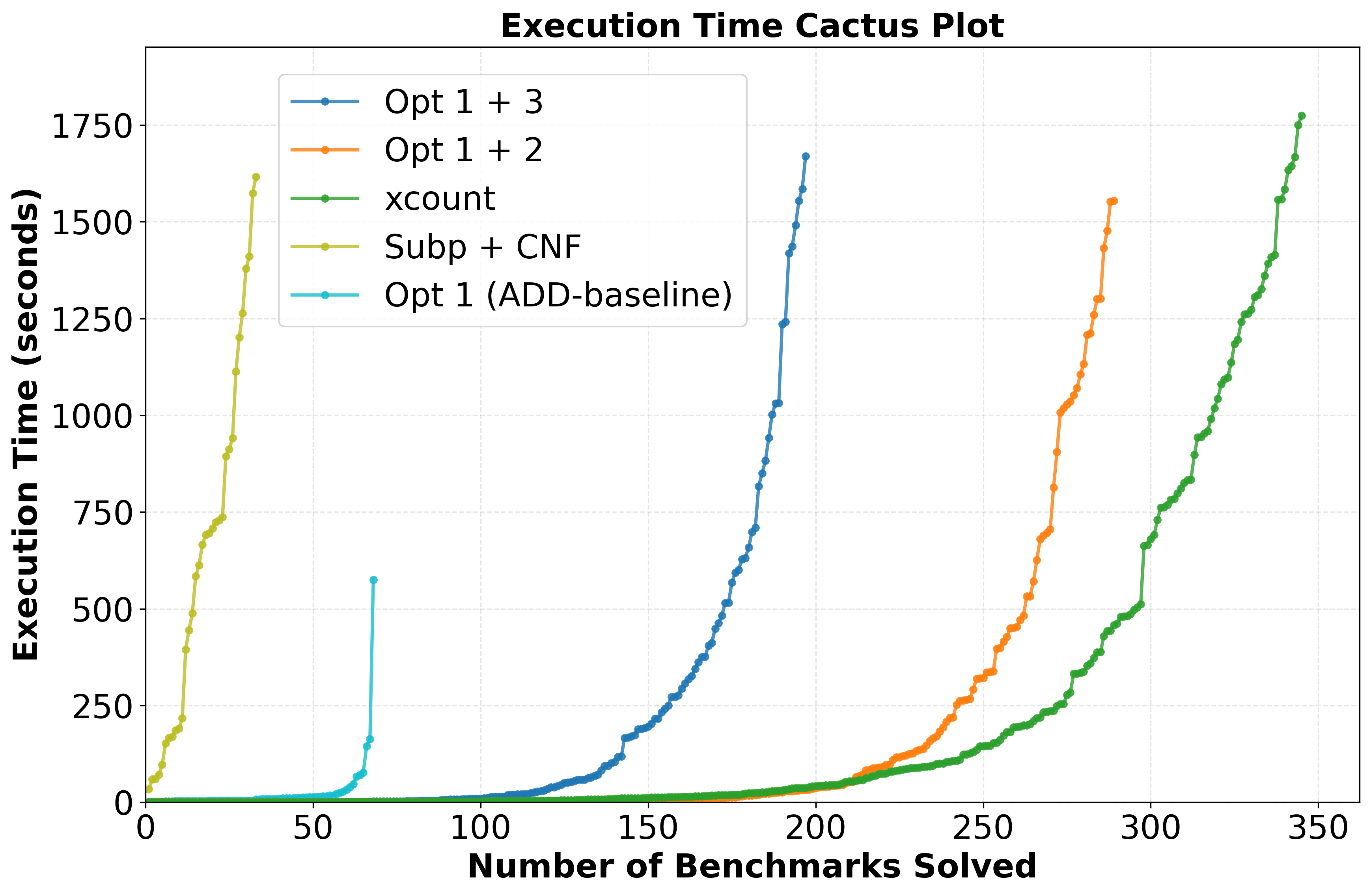}
        \caption{Cactus plots for ablation experiments on \textit{Adult} dataset}
        \label{fig:adult_ablation}
    \end{minipage}
\quad
    \begin{minipage}[t]{0.45\textwidth}
        \vspace{0.5cm}
        \small
        \begin{tabular}{lcc}
            \toprule
            \textbf{Tool} & \textbf{Solved} & \textbf{PAR2} \\
            \midrule
            Opt 1 ($\xbase$) & 68 & 3122.31 \\
            Opt 1+3 & 197 & 2275.374 \\
            Opt 1+2 & 289 & 1635.355 \\
            Subp + CNF & 33 & 3408.22 \\
            $\xcount$ & 345 & 1304.486 \\
            \bottomrule
        \end{tabular}
        \captionof{table}{Aggregate statistics for each configuration. For configuration Subp + Cnf, only subproblems are solved without merging the solutions of subproblems.}
        \label{tab:par2_ablation}
    \end{minipage}
\end{figure}
Figure~\ref{fig:adult_ablation} and Table~\ref{tab:par2_ablation} summarize the results for the experiments on configurations listed above. We see that solving the problem becomes faster when encoded as an ADD, but this does not scale to larger ensembles. Each of the other optimizations lead to benefits with the best performance coming from adding all of them.

{\em In summary}, $\xcount$ outperforms the other baseline approaches by a significant margin especially as sizes of ensembles increases, while the accuracy of the count remains well within the bounds of the theoretical guarantees (for $\epsilon=0.1$, $\delta=0.1$) with 99\% of counts being within 10\% error. We also analyzed scalability with increasing size of the ensembles and role played by each optimization, collectively highlighting the efficacy of our approach. 



\section{Application of $\xcount$ to regularization validation}
\label{sec:motivation}
In this section, we present a case-study to illustrate a potential application of counting sensitivity, by relating it to the level of L1 regularization used during training of the models. A tree ensemble model, such as Gradient Boosted Tree or Random Forest, is trained by splitting a node at a time and prone to overfitting if not properly regularized. L1 regularization, uses a penalty term in the loss function, which encourages fewer splits in the tree. If a node is deeper in a tree the penalty is larger, proportional to the regularization parameter $\alpha$~\cite{andrew2004feature} (see Appendix \ref{app:l1_reg} for more details). Our hypothesis is that smaller $\alpha$ leads to overfitting and increased sensitivity, while larger $\alpha$ promotes generalization and fewer sensitive regions. By quantifying the number of sensitive regions, we can gain insights into the training process of the model, if we do not have access to the model's internal parameters, and assess its potential overfitting.

\begin{wrapfigure}[18]{r}{0.46\textwidth}
  \vspace*{-7mm}
  \centering
  \includegraphics[width=\linewidth]{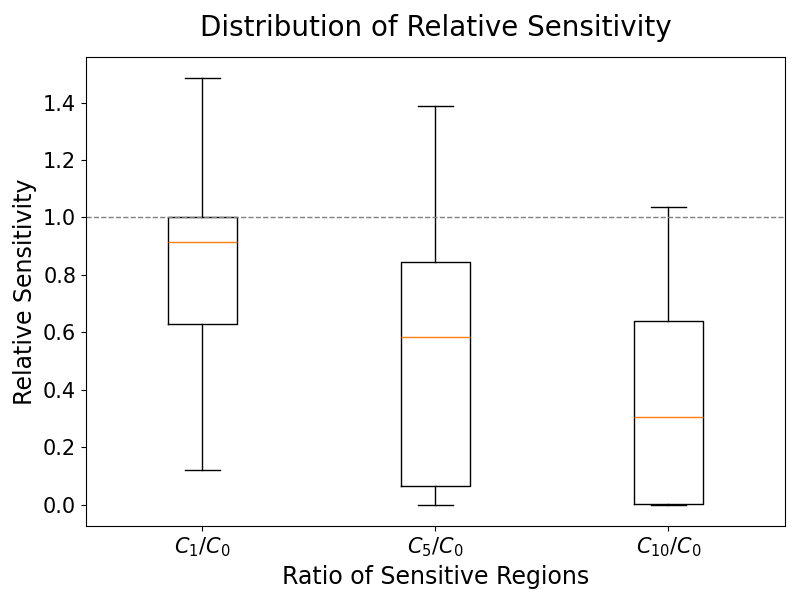}
    \caption{\small Relative sensitivity $C_\alpha/C_0$ for regularized models ($\alpha \in \{1,5,10\}$) against the unregularized baseline. The yellow line marks the median; boxes and whiskers show the IQR and $1.5\times$ IQR range.}
    \label{fig:regularization_impact}
\end{wrapfigure}
We used {\em Adult} dataset for this study.  We trained several models with varying levels of regularization and compare the number of sensitive regions in each model. We repeat this process by varying the other configurations such as depth of trees and number of trees in the ensemble. We consider models of depth 3 and 4, with number of trees as 20, 30 and 40. For each configuration, we trained models with L1 regularization parameters of 0, 1, 5 and 10. Table~\ref{tab:all_metrics} in Appendix shows that our choice of hyperparameters ensures that increasing the L1 regularization parameter $\alpha$ had a negligible impact on predictive performance. This stability allows us to compare the models' sensitivity under the assumption that they are functionally equivalent in terms of standard 
performance metrics.

We selected thirteen sensitive features from the dataset: age, workclass, 
fnlwgt, education, education-num, marital-status, occupation, relationship, race, sex, capital-gain, capital-loss, hours-per-week for our sensitivity analysis.  In total, we have 78 configurations (2 depths $\times$ 3 number of trees $\times$ 13 sensitive features) for our study.

For each configuration and sensitive feature, we computed the number of sensitive regions using $\xcount$, for each of the regularization parameters. To ensure that the discretization of the input space was consistent across four models, we included the guard predicates from all the models during the execution of $\xcount$. We excluded instances where the baseline model yielded zero sensitive regions and the sensitive feature is missing in the models. We used gap 0.5 for our 
study. Let $C_0$ be the count of sensitive regions for the model with no regularization ($\alpha = 0$). We define the relative sensitivity for a model with regularization parameter $\alpha$ as: $C_\alpha/C_0$, where $C_\alpha$ is the count of sensitive regions for the model with regularization parameter $\alpha$. In Figure~\ref{fig:regularization_impact}, we present the results of our study. The box plots show the distribution of relative sensitivity counts across all configurations for each regularization parameter. We observe a clear trend where increasing the regularization parameter leads to a decrease in the number of sensitive regions, confirming our hypothesis that counting sensitive regions provides a useful metric to {\em validate} the impact of regularization on model behavior.


\section{Conclusion}
\label{sec:conclusion}
In this work, we built a quantitative notion of sensitivity over decision tree ensembles and presented $\xcount$, a novel algorithmic technique that can quantify sensitivity over DTEs. The challenges of scale are handled by using a novel compositional and probabilistic approach that provides concrete guarantees. Experimental evaluation over an extensive set of benchmarks show the effectiveness of our method over the baseline and other model counting based approaches.

As future work, we would like use our sensitivity counts to rank DTE models. While our work focused on XGBoost models, our approach and technique work immediately for any decision tree ensemble model. We would like hence to apply this to wider classes of benchmarks, and find other applications for which counting sensitivity could be useful.


\section*{Acknowledgements}
We acknowledge the State Bank of India (SBI) Foundation Hub for Data Science \& Analytics , IIT
Bombay for supporting the work done in this project. We thank Shri Shakeel Ahmed Agasimani,
Deputy General Manager, Analytics Department, C Ramesh Chander, Assistant General Manager
(Statistician), Muthukumaran M S, Chief Manager (Data Scientist), Komaragiri Srinivas Jagannath,
Manager (Data Scientist), Neha Maheswari, Manager (Data Scientist), State Bank of India, for
multiple wide-ranging discussions and feedback.

%
%
%












%
\bibliographystyle{splncs04}
\bibliography{biblio}
\appendix
\section*{Appendix}

This section is divided into five subsections : (i) self-contained proof of
Theorem~\ref{thm:xcount-correct}, (ii) the pseudocode of the subroutines used by the main algorithms,
(iii) the experimental plots for each of the datasets and the plots for performance comparison over varying gap thresholds,
(iv) extra details regarding the regularization application, and (v) 
a detailed explanation of the algorithm through an example.

\section{Expanded Proof of Theorem~\ref{thm:xcount-correct}}
\label{app:xcproof}

For completeness, we present a self-contained proof of
Theorem~\ref{thm:xcount-correct}.
The argument relies on a standard concentration inequality for Poisson random
variables, stated below for reference.

\begin{lemma}[Poisson concentration]\label{lem:poisson-chernoff}
	Let $X \sim \mathrm{Poisson}(\lambda)$ for $\lambda>0$.
	Then for any $x>0$,
	\begin{align*}
		\Pr[|X-\lambda|\ge x]
		&\le
		2\exp\!\left(-\frac{x^2}{2\lambda+x}\right).
	\end{align*}
	In particular, for any $\epsilon\in(0,1]$,
	\begin{align*}
		\Pr[|X-\lambda|\ge \epsilon\lambda]
		&\le
		2\exp\!\left(-\frac{\epsilon^2\lambda}{3}\right).
	\end{align*}
\end{lemma}

We now give the detailed proof of Theorem~\ref{thm:xcount-correct}.

\begin{proof}[Proof of Theorem~\ref{thm:xcount-correct}]
	Fix an ordering of the subproblems processed by {\xcount}.
	For each iteration $i\in[m]$, let $U_i \subseteq \{0,1\}^F$ denote the set of
	assignments generated by the $i$-th subproblem and define
	\begin{align*}
		S_i &:= \bigcup_{t=1}^i U_t,
		&
		S_m &:= \bigcup_{t=1}^m U_t.
	\end{align*}
	By construction,
	\begin{align*}
		|S_m| = C(S,d,G).
	\end{align*}
	
	For $j\ge 0$, let $p_j := 2^{-j}$.
	For $i\in[m]$, define $X^{(i)}_j$ to be the random multiset obtained after
	processing the first $i$ subproblems, conditioned on the sampling probability
	being $p_j$.
	For each $t\le i$, the algorithm draws
	\begin{align*}
		N_t \sim \mathrm{Poisson}(|U_t|\,p_j)
	\end{align*}
	samples uniformly from $U_t$.
	For a fixed $s\in S_i$, the total number of times $s$ is generated is the sum
	of independent Poisson variables, one from each $U_t$ containing $s$.
	This sum is Poisson with mean $p_j$, and the multiplicities are independent
	across distinct $s$.
	Consequently,
	\begin{align*}
		|X^{(i)}_j| \sim \mathrm{Poisson}(|S_i|p_j).
	\end{align*}
	
	For $1\le i\le m$ and $j\ge 0$, define
	\begin{align*}
		E^{(i)}_j
		&:= \text{``after processing the first $i$ subproblems,
			the algorithm has sampling probability $p_j$''}, \\
		A^{(i)}_j
		&:=
		\left\{
		|X^{(i)}_j| \not\in
		\bigl[|S_i|p_j(1-\epsilon),\,|S_i|p_j(1+\epsilon)\bigr]
		\right\}.
	\end{align*}
	Let $j^\star$ be the smallest integer $j$ such that
	\begin{align*}
		p_j < \frac{\thresh}{4|S_m|}.
	\end{align*}
	
	On event $E^{(m)}_j$, the algorithm outputs
	\begin{align*}
		\widehat{C} = \frac{|X^{(m)}_j|}{p_j}.
	\end{align*}
	Hence,
	\begin{align*}
		\Pr\!\left[
		\widehat{C}\not\in[(1-\epsilon)|S_m|,(1+\epsilon)|S_m|]
		\right]
		&\le
		\sum_{j=0}^{j^\star-1}\Pr[A^{(m)}_j]
		+
		\Pr\!\left[\bigcup_{j\ge j^\star}E^{(m)}_j\right].
	\end{align*}
	
	For $j<j^\star$,
	\begin{align*}
		|X^{(m)}_j|
		&\sim \mathrm{Poisson}(\lambda_j),
		&
		\lambda_j
		&:= |S_m|p_j \ge \thresh/4.
	\end{align*}
	Applying Lemma~\ref{lem:poisson-chernoff},
	\begin{align*}
		\Pr[A^{(m)}_j]
		&\le
		2\exp\!\left(-\frac{\epsilon^2\lambda_j}{3}\right)
		\le
		2\exp\!\left(-\frac{\epsilon^2\thresh}{12}\right),
	\end{align*}
	and therefore
	\begin{align*}
		\sum_{j=0}^{j^\star-1}\Pr[A^{(m)}_j]
		&\le
		2j^\star \exp\!\left(-\frac{\epsilon^2\thresh}{12}\right)
		\le
		\frac{\delta}{6}.
	\end{align*}
	
	The event $\bigcup_{j\ge j^\star}E^{(m)}_j$ can occur only if, for some $i\in[m]$,
	\begin{align*}
		|X^{(i)}_{j^\star-1}| \ge \thresh.
	\end{align*}
	For any fixed $i$,
	\begin{align*}
		|X^{(i)}_{j^\star-1}|
		&\sim \mathrm{Poisson}(|S_i|p_{j^\star-1})
		\le
		\mathrm{Poisson}(\thresh/2),
	\end{align*}
	since $S_i\subseteq S_m$ and $p_{j^\star-1}<\thresh/(2|S_m|)$.
	Applying Lemma~\ref{lem:poisson-chernoff},
	\begin{align*}
		\Pr\!\left[|X^{(i)}_{j^\star-1}|\ge \thresh\right]
		&\le
		\exp(-\thresh/6)
		\le
		\frac{\delta}{6m}.
	\end{align*}
	A union bound over $i\in[m]$ yields
	\begin{align*}
		\Pr\!\left[\bigcup_{j\ge j^\star}E^{(m)}_j\right]
		\le
		\frac{\delta}{6}.
	\end{align*}
	
	Combining the bounds,
	\begin{align*}
		\Pr\!\left[
		\widehat{C}\not\in[(1-\epsilon)|S_m|,(1+\epsilon)|S_m|]
		\right]
		\le \delta.
	\end{align*}
	Since $|S_m|=C(S,d,G)$, the theorem follows.
\end{proof}
\section{Details of Section~\ref{sec:algorithm}}
\label{app:algo}
This section contains the details of the procedures used by the primary 
algorithms listed in Section~\ref{sec:algorithm}.

The subroutine $\bitmaskgen$ takes as input an integer $k$ and gives 
$(k+1)$ bit vectors, each of size k, such that for each bit vector 
the monotonicity property is followed - bit preceding a given bit must be greater 
that or equal to it.  

\begin{algorithm}[H]
\caption{$\bitmaskgen$($k$)}
\label{alg:bitmaskgen}
\begin{algorithmic}[1]
\State $MaskSet \gets \emptyset$
\State $BitMask[k] \gets 1$
\For{$i=1$ to $k$}
  \State $BitMask[i] \gets 0$
  \State $MaskSet \gets BitMask$
\EndFor
\State \Return $MaskSet$
\end{algorithmic}
\end{algorithm}



The subroutine $\mathsf{PruneTree}$ takes a node of the tree as input 
and recursively prunes it by traversing 
the tree and checking decision at each node against the given assignment set $m$.

\begin{algorithm}[H]
\caption{$\mathsf{PruneTree}$($n,m$)}
\label{alg:prunetree}
\begin{algorithmic}[1]
\If{$n$ is leaf}
    \Return
\EndIf
\State $n.yes \gets \mathsf{PruneTree}(n.yes,m)$ 
\State $n.no \gets \mathsf{PruneTree}(n.no,m)$ 
\If{$(n.feature,n.threshold) \notin m$}
    \Return $n$
\EndIf
\State $i \gets \text{index of }(n.feature,n.threshold) \text{ in m}$ 
\State nodes $keep \gets nullptr$, $discard \gets nullptr$
\If{$m[i] == 1$}
    \State $keep \gets n.yes$
    \State $discard \gets n.no$
\Else 
    \State $keep \gets n.no$
    \State $discard \gets n.yes$
\EndIf
\State \Return $keep$
\end{algorithmic}
\end{algorithm}

The subroutine $\EncodeADD$ takes a decision tree node and 
recursively encodes it as an ADD by traversing through the tree 
by making the tree nodes as the nodes of the ADD, and the leaf value 
as the terminal node of the ADD which stores the real value.
 
\begin{algorithm}[H]
\caption{$\EncodeADD$($n$)}
\label{alg:encodeadd}
\begin{algorithmic}[1]
\State init ADD node $A$
\If{$n$ is a leaf node}
    \Return $A \gets n.val$
\Else
    \State $A.true \gets \EncodeADD$($n.yes$)
    \State $A.false \gets \EncodeADD$($n.no$)
\EndIf
\State \Return A
\end{algorithmic}
\end{algorithm}

The subroutine $\ToBDD$ converts an ADD to a BDD by taking a 
threshold value and constructing the terminal nodes of the BDD 
according to the decision whether the real value in the terminal 
node of the ADD  is greater than the threshold.

\begin{algorithm}[H]
\caption{$\ToBDD$($n$)}
\label{alg:tobdd}
\begin{algorithmic}[1]
\If{$n$ is a leaf node}
    \If{$n.val > G$}
        \Return 1
    \Else 
        \Return 0
    \EndIf
\Else
    \State $\ToBDD$($n.yes$)
    \State $\ToBDD$($n.no$)
\EndIf
\State \Return
\end{algorithmic}
\end{algorithm}

The subroutine $\mathsf{SampleBDD}$ takes a BDD, an integer $N$ 
and a pair of bit-masks, and samples satisfying assignments 
from the BDD uniformly, appends it with either one of the 
masks (randomly selected) to form the complete assignment and 
returns all such complete assignments. 

\begin{algorithm}[H]
\caption{$\mathsf{SampleBDD}$($BDD$,$N$,$pair$)}
\label{alg:samplebdd}
\begin{algorithmic}[1]
\State $S \gets \emptyset$
    \For{$i = 1$ to $N$}
        \State Sample assignment $x$ via top-down traversal using transition probabilities $p(u \to v)$
        \State $m \gets $ Sample one from bitmask pair $(pair.m_1, pair.m_2)$ with 0.5 probability
        \State S.Append$((x, m))$
    \EndFor
\State \Return $S$
\end{algorithmic}
\end{algorithm}

\section{Details of Section~\ref{sec:experiments}}
\label{app:expts}

\subsection{Plots for all datasets}

In this section, we present the plots of experimental evaluation 
for each of the datasets in figure~\ref{fig:allplots1}.

\begin{figure}[h!]
    \centering
    \begin{subfigure}{0.49\textwidth}
        \includegraphics[width=\textwidth]{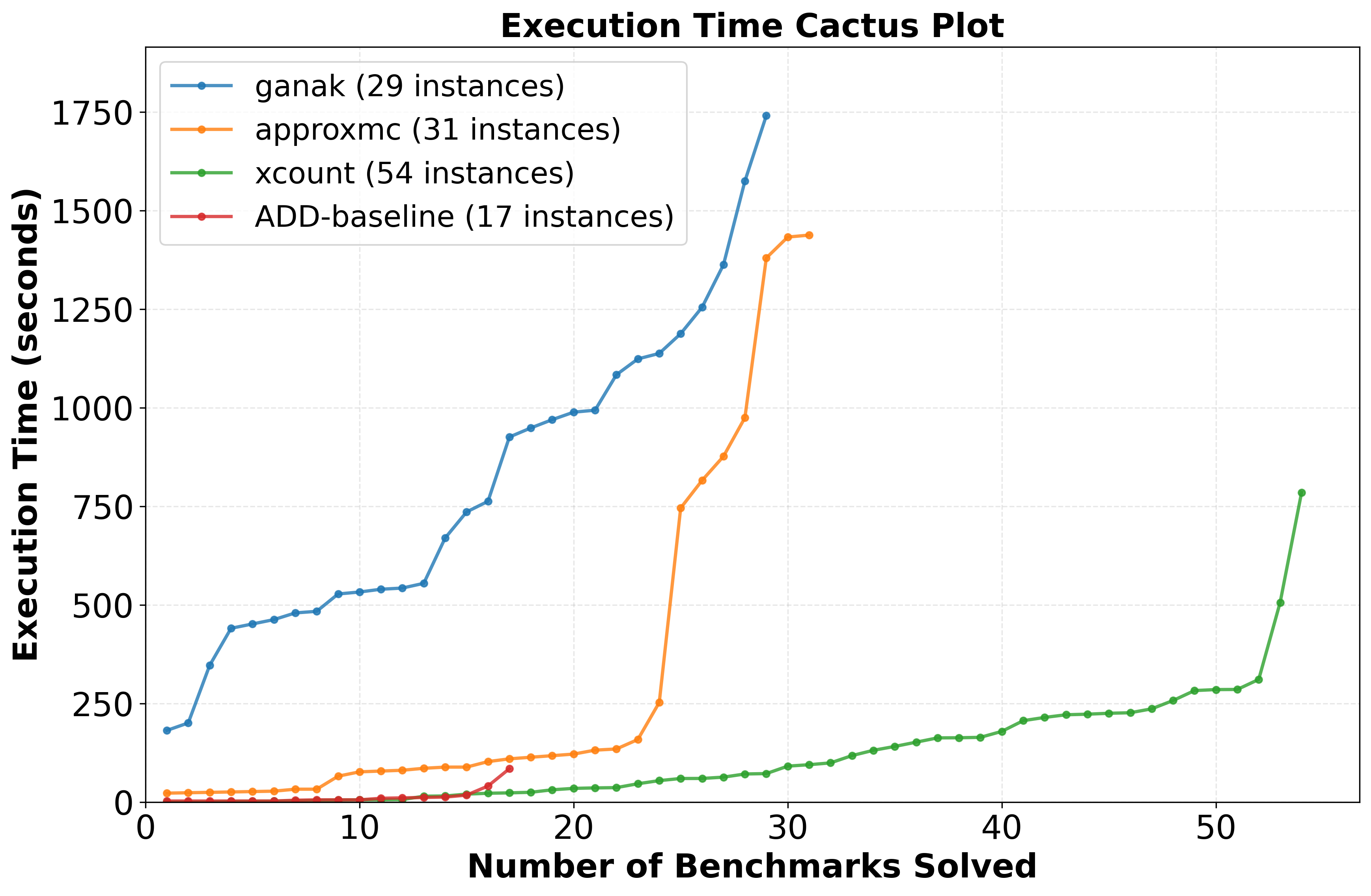}
        \caption{Diabetes}
        \label{fig:cactus_diabetes}
    \end{subfigure}
    \hfill
    \begin{subfigure}{0.49\textwidth}
        \includegraphics[width=\textwidth]{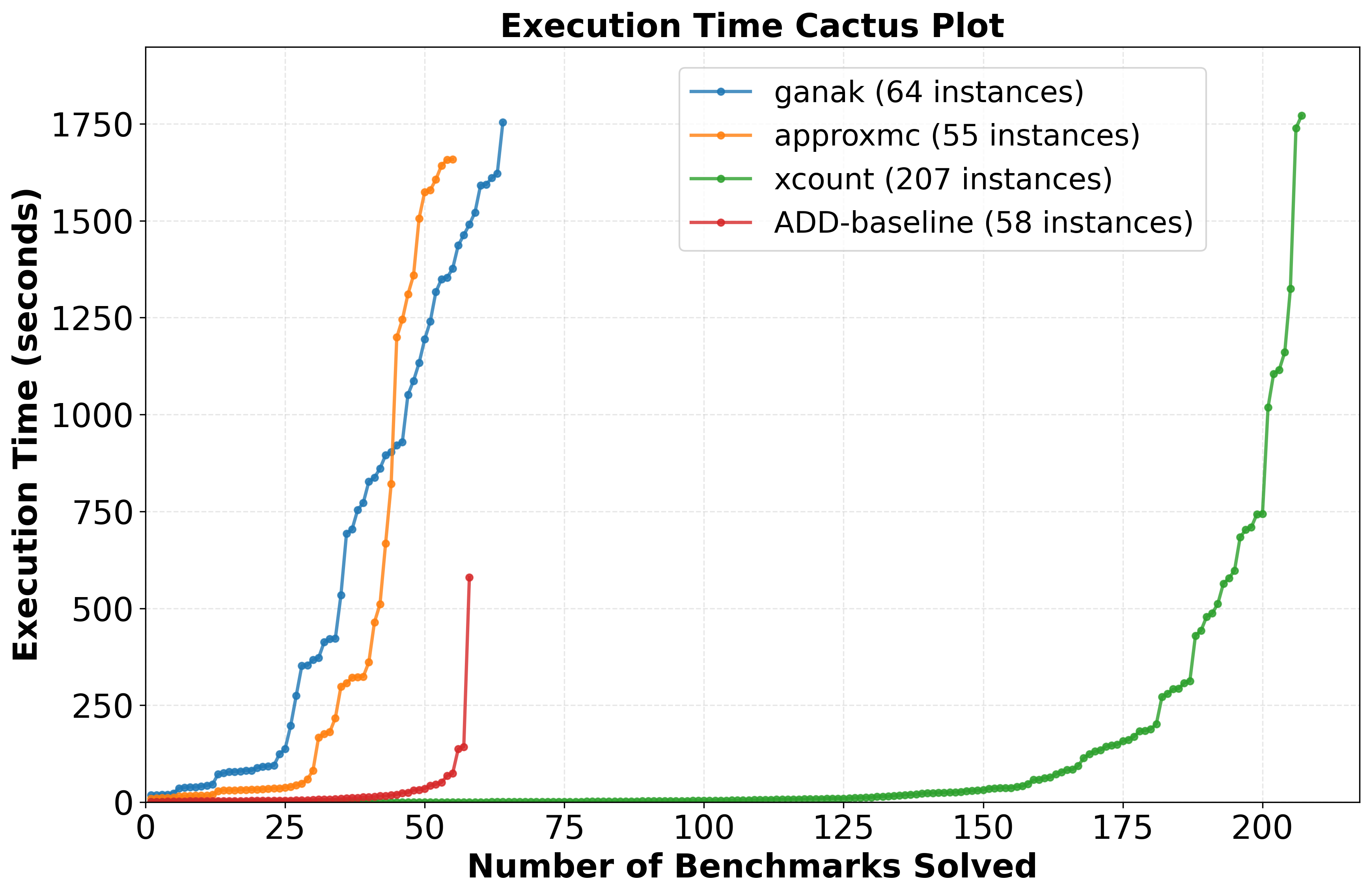}
        \caption{Adult}
        \label{fig:cactus_adult}
    \end{subfigure}

    \vspace{0.6em}

    \centering
    \begin{subfigure}{0.49\textwidth}
        \includegraphics[width=\textwidth]{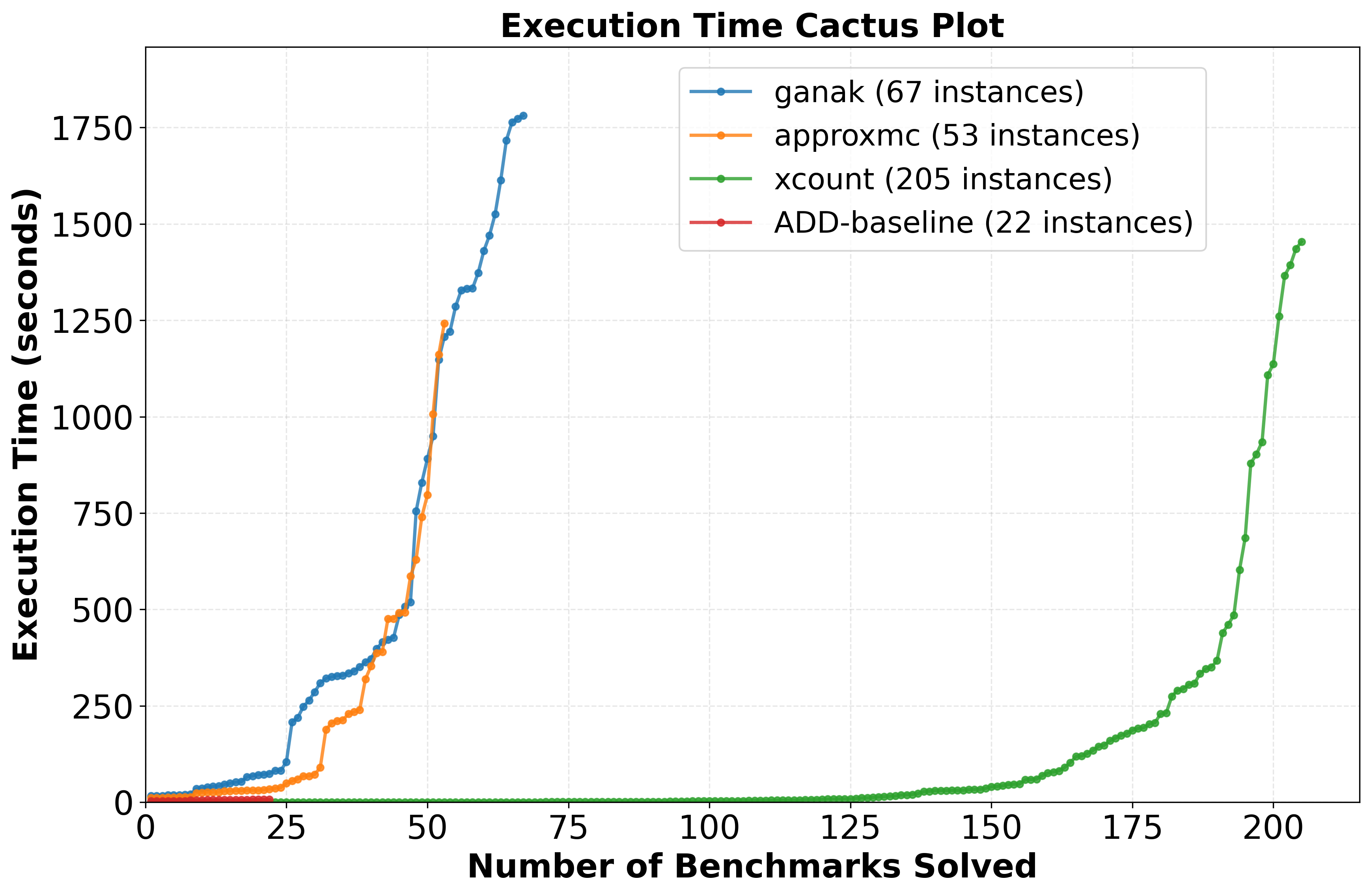}
        \caption{Covtype}
        \label{fig:cactus_diabetes}
    \end{subfigure}
    \hfill
    \begin{subfigure}{0.49\textwidth}
        \includegraphics[width=\textwidth]{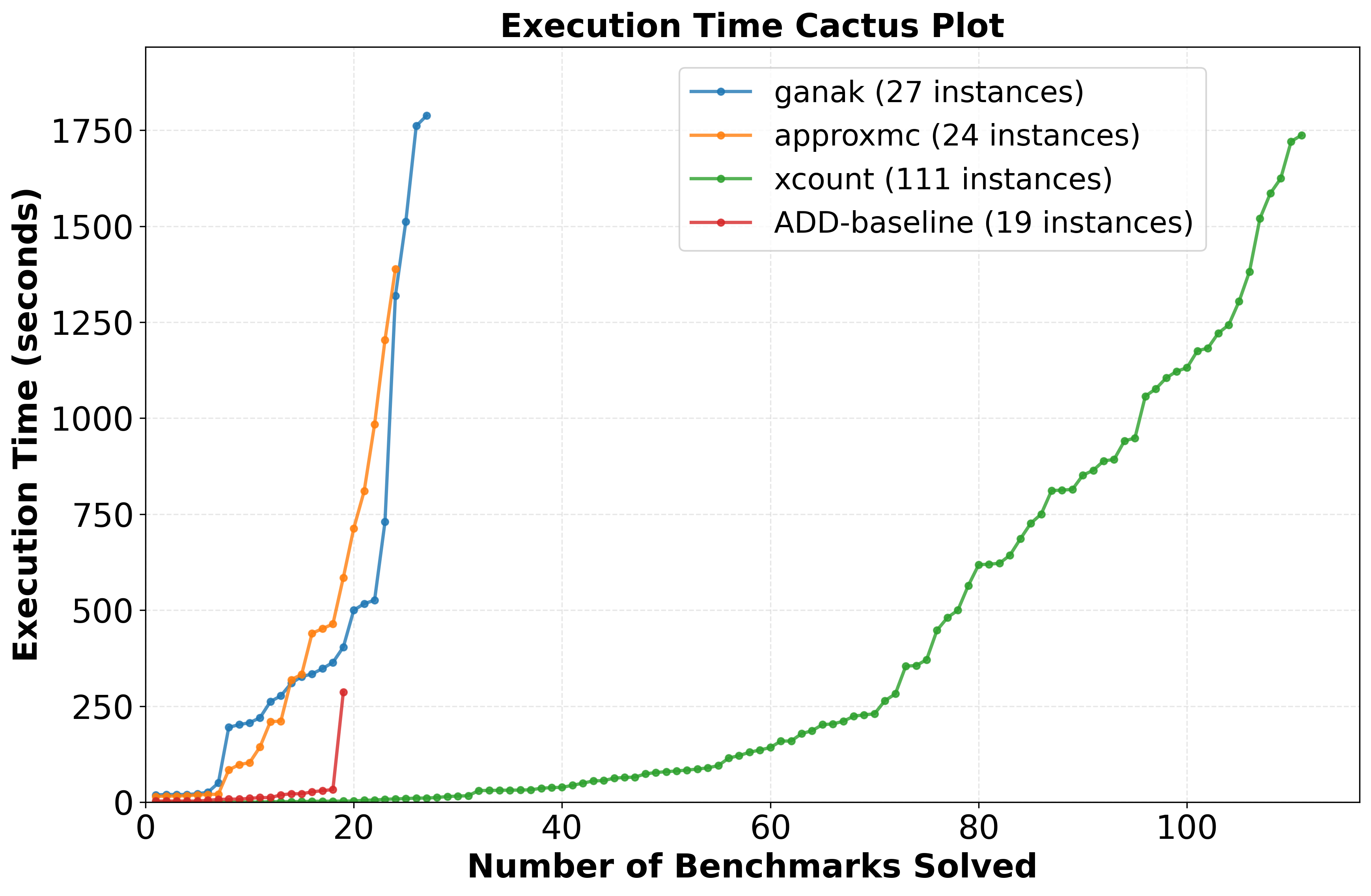}
        \caption{Protein structure}
        \label{fig:cactus_adult}
    \end{subfigure}

    \vspace{0.6em}

    \centering
    \begin{subfigure}{0.49\textwidth}
        \includegraphics[width=\textwidth]{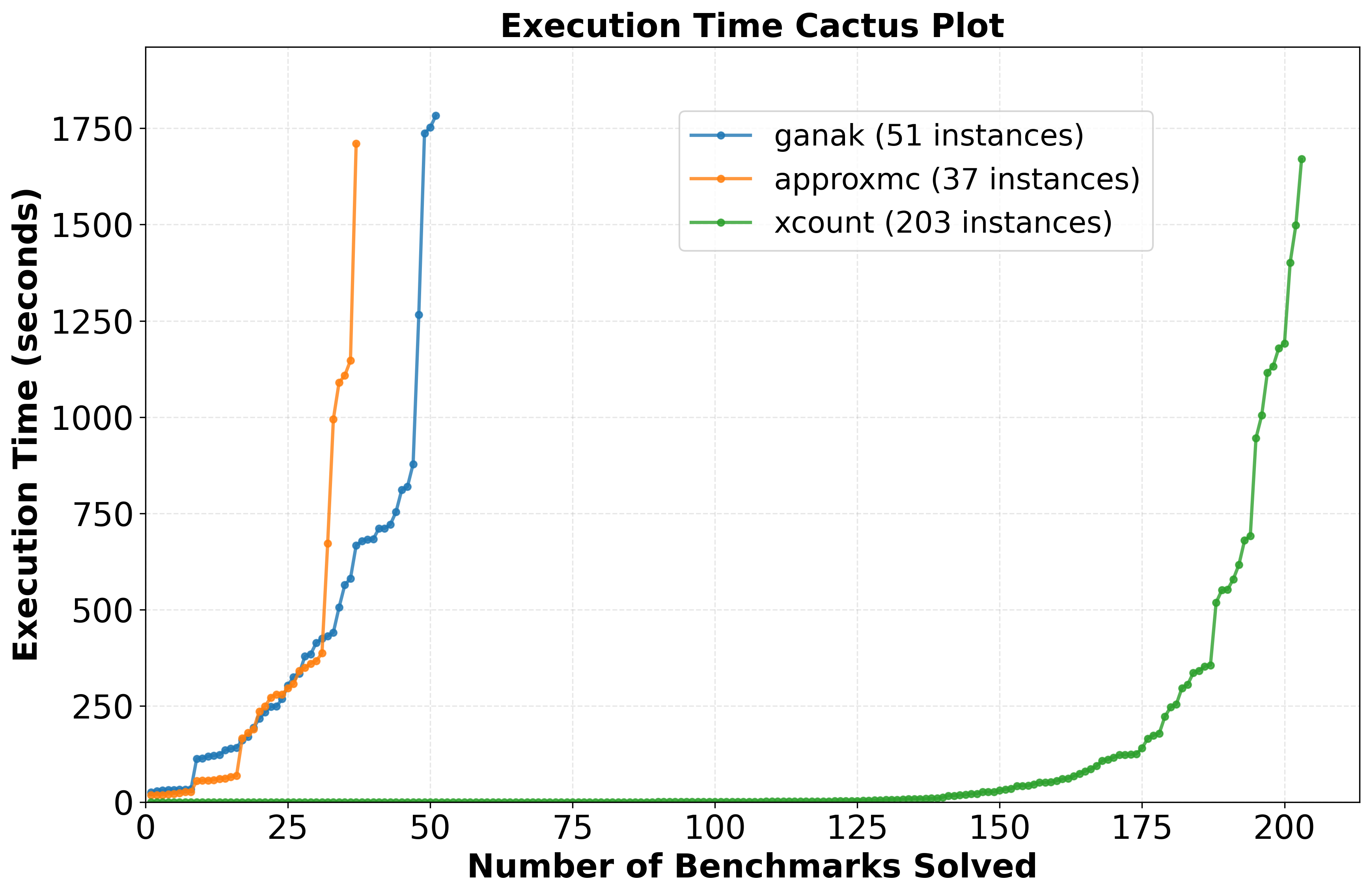}
        \caption{Mnist}
        \label{fig:cactus_diabetes}
    \end{subfigure}
    \hfill
    \begin{subfigure}{0.49\textwidth}
        \includegraphics[width=\textwidth]{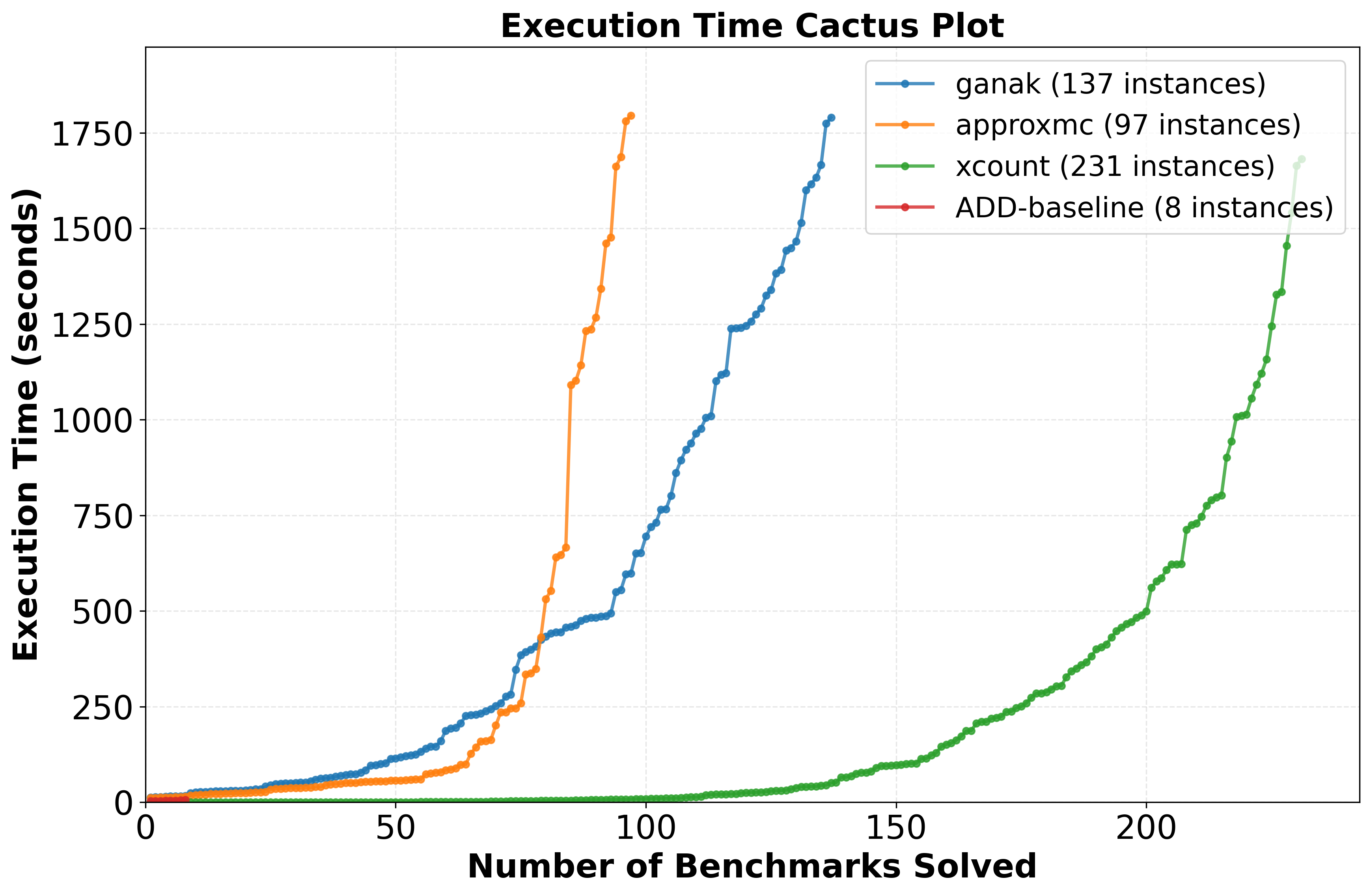}
        \caption{Webspam}
        \label{fig:cactus_adult}
    \end{subfigure}

    \vspace{0.6em}
    \caption{Cactus plots comparing execution time. Lower curves indicate better performance.}
    \end{figure}

\begin{figure}[h!]
    \ContinuedFloat
    \centering
    \begin{subfigure}{0.49\textwidth}
        \includegraphics[width=\textwidth]{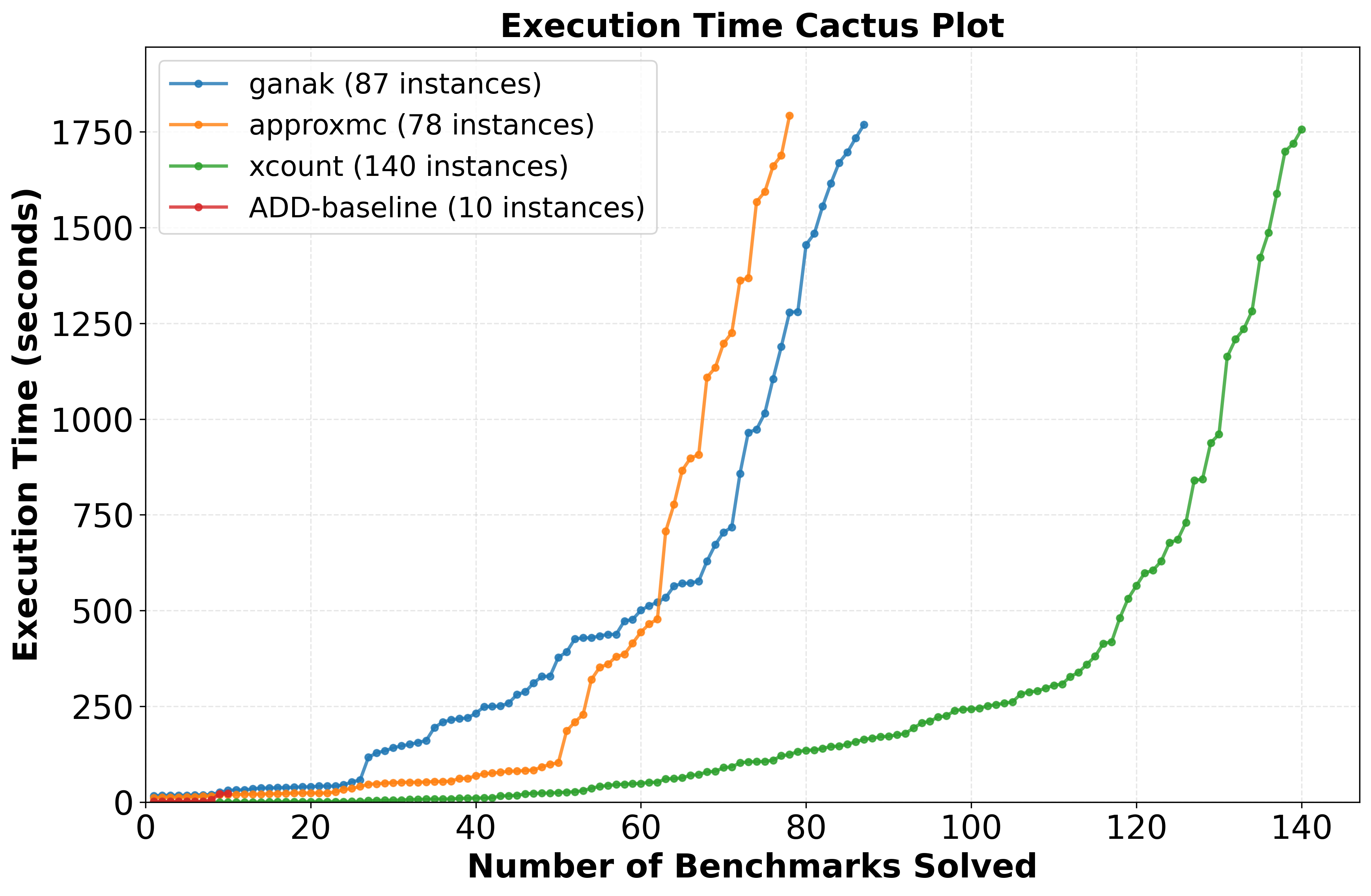}
        \caption{Wine-quality}
        \label{fig:cactus_diabetes}
    \end{subfigure}
    \hfill
    \begin{subfigure}{0.49\textwidth}
        \includegraphics[width=\textwidth]{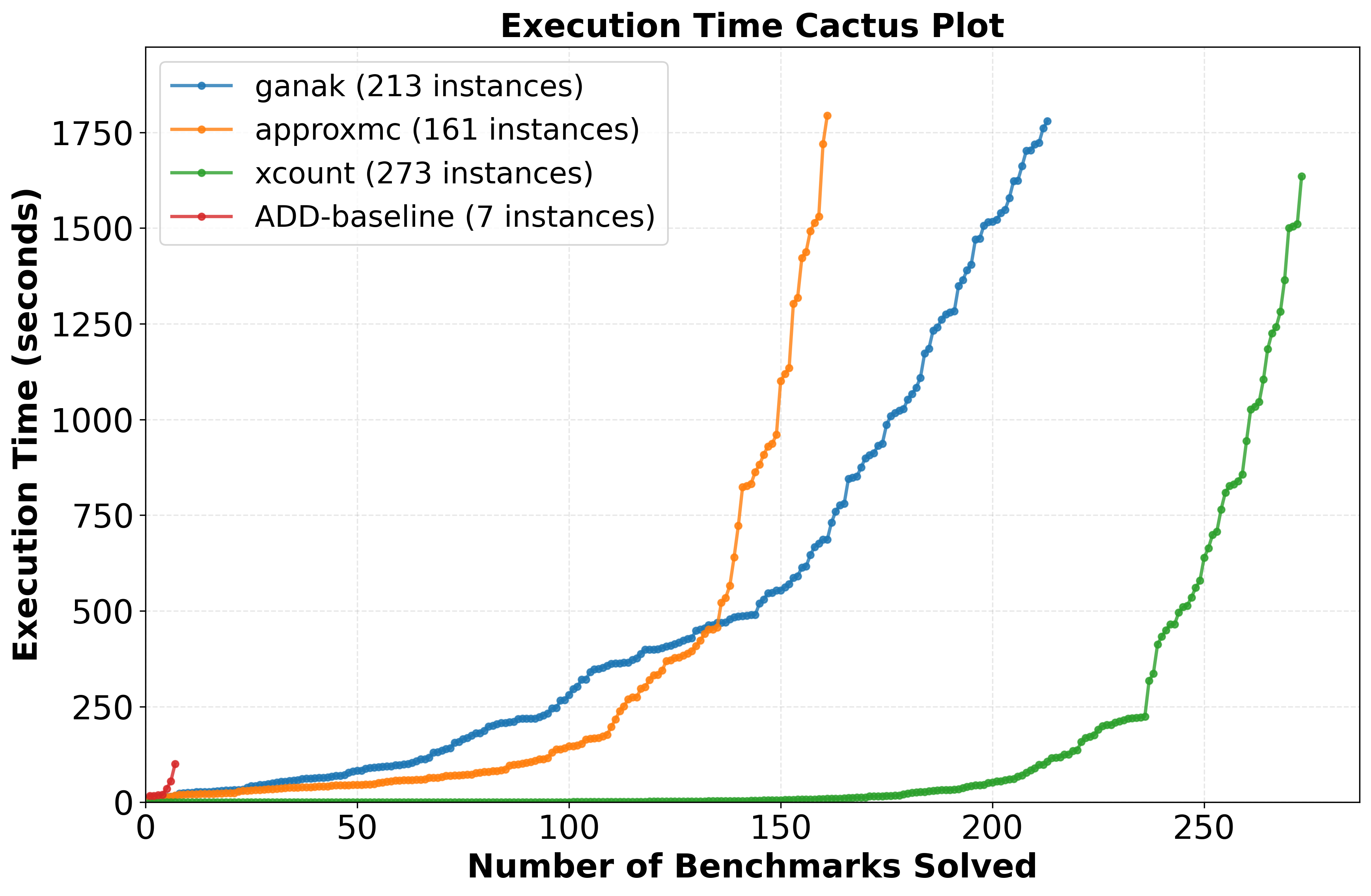}
        \caption{Fashion-Mnist}
        \label{fig:cactus_adult}
    \end{subfigure}
    \centering
    \begin{subfigure}{0.49\textwidth}
        \includegraphics[width=\textwidth]{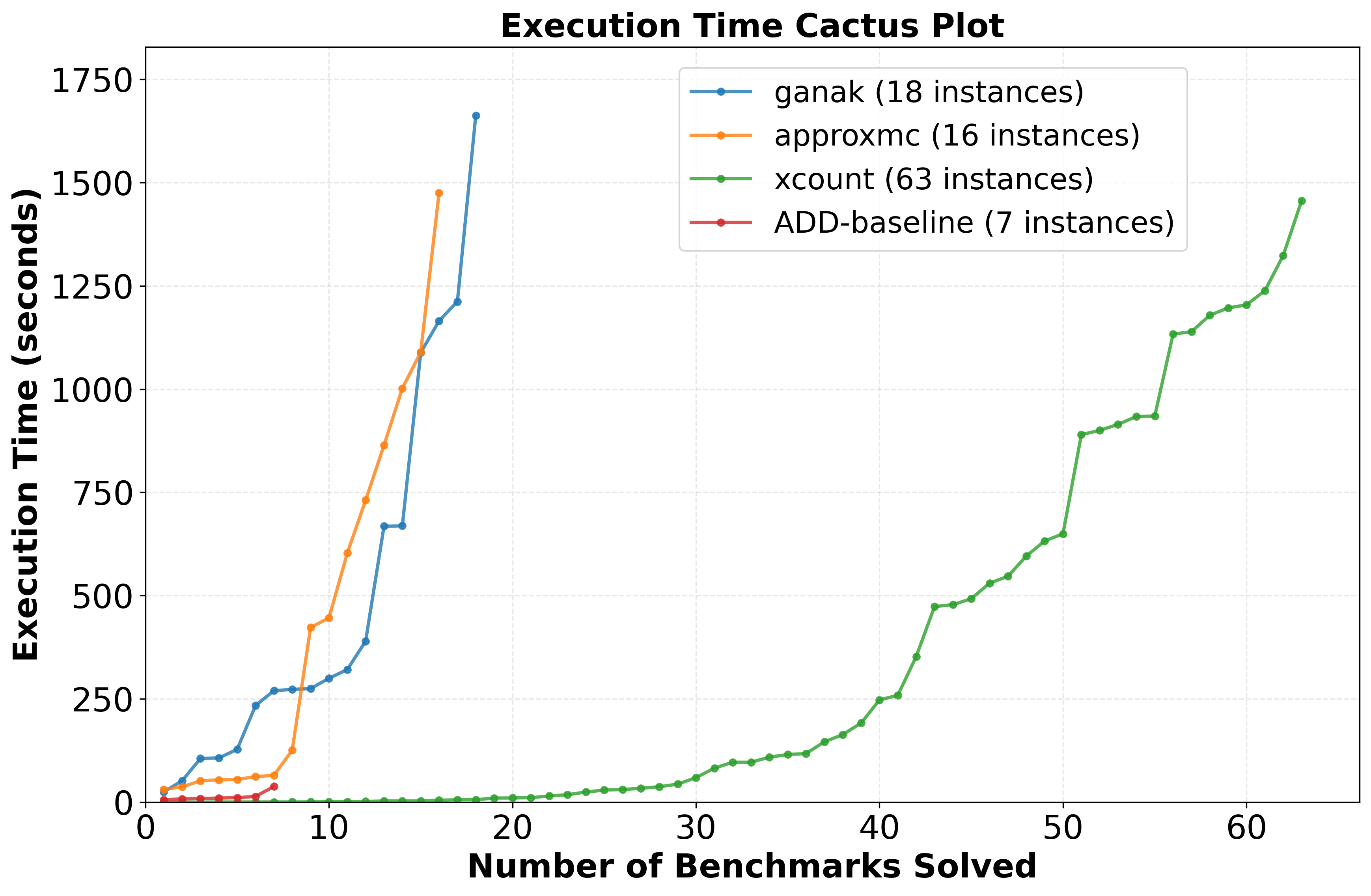}
        \caption{Higgs}
        \label{fig:cactus_diabetes}
    \end{subfigure}
    \hfill
    \begin{subfigure}{0.49\textwidth}
        \includegraphics[width=\textwidth]{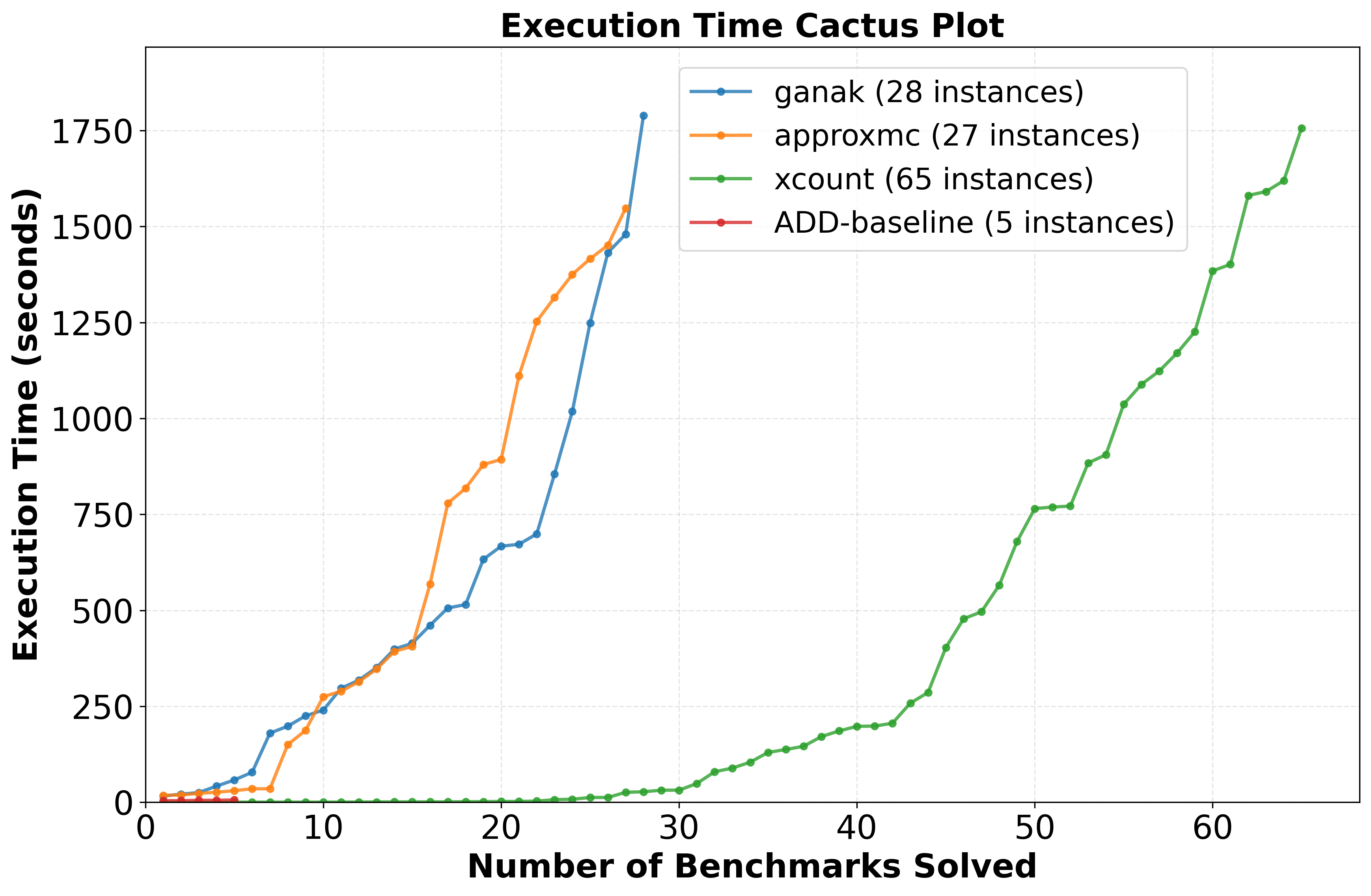}
        \caption{Supersymmetry}
        \label{fig:cactus_adult}
    \end{subfigure}

    \caption{Cactus plots comparing execution time. Lower curves indicate better performance.}
    \label{fig:allplots1}
\end{figure}

\subsection{Performance across varying gap thresholds}

\begin{figure}[h!]
    \centering
    \includegraphics[width=0.6\textwidth]{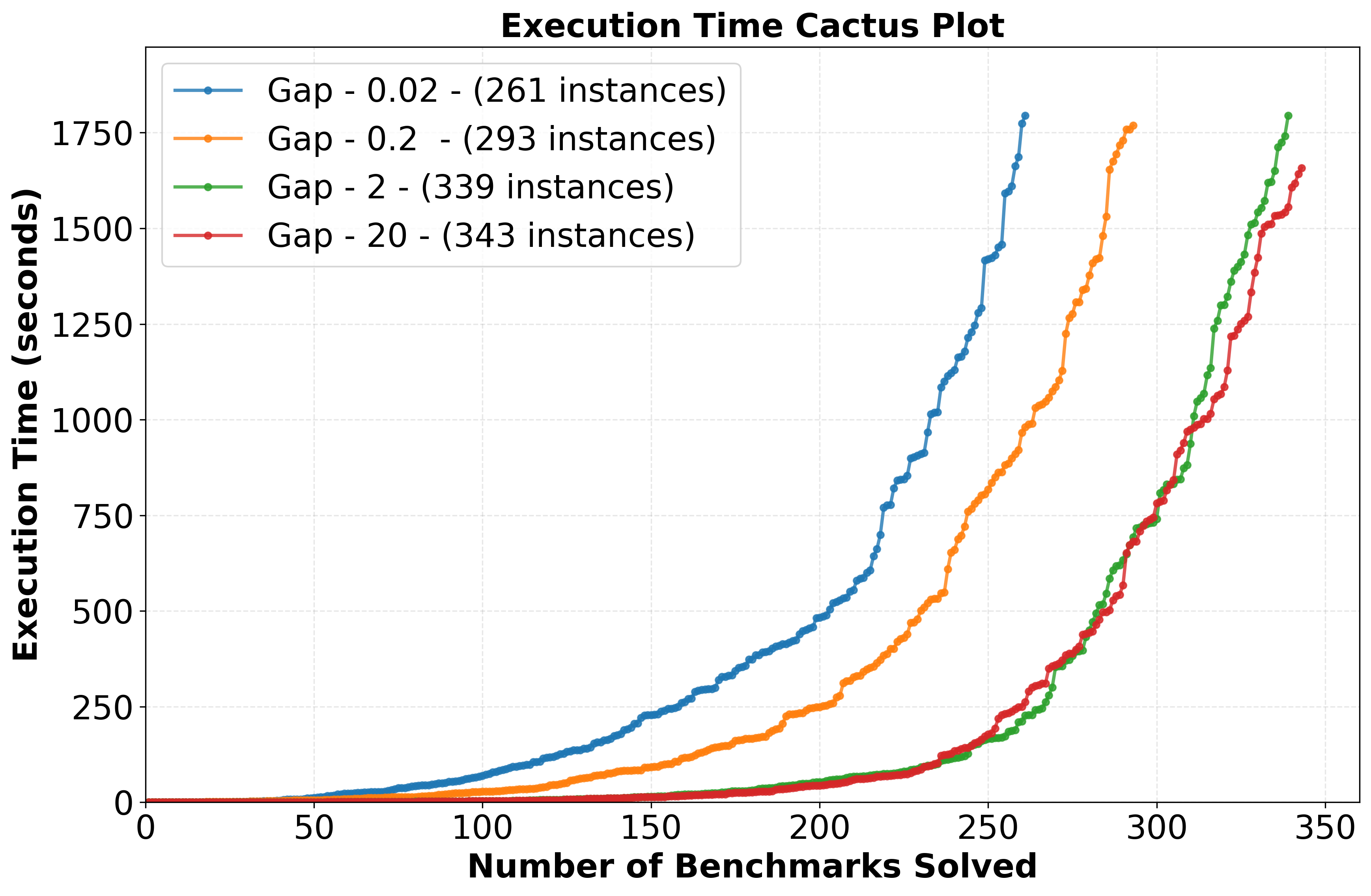}
    \caption{Performance of $\xcount$ across gap values for models of \textit{Adult} dataset.}
    \label{fig:error_plots}
\end{figure}

To test the performance of $\xcount$ across varying scales of gap 
thresholds, we tested it on 509 benchmark instances of models trained on 
\textit{Adult} dataset, generated as per methodology outlined in ablation in~\ref{sec:experiments} and 
running the instances for varying gap thresholds through 0.02, 0.2, 2 and 20, while keeping precision of leaves fixed at 3 decimal places.
We observe that for lower gap threhold the performance of $\xcount$ slows 
down. This is expected as lowering the gap threshold effectively translates to counting assignments that lie 
closer to each other in terms of their ensemble output, and as we approach the scale of smallest difference possible between two 
ensemble outputs this would effectively count all possible assignments as satisfying the gap threshold.

\section{Details on Section \ref{sec:motivation}}
\subsection{L1 Regularization Parameter ($\alpha$)}
\label{app:l1_reg}

To understand the role of $\alpha$, we first consider the standard objective function for tree ensemble models.

\subsubsection{Objective Function without Regularization}
In a typical gradient boosting framework, the goal is to minimize a loss function $l$ that measures the difference between the predicted 
value $\hat{y}_i$ and the actual value $y_i$ for each instance $i$. Without regularization, the objective function at iteration $t$ aims 
solely to minimize this empirical loss~\cite{xgboost}:

\begin{equation}
    \mathcal{L}^{(t)}_{\text{unreg}} = \sum_{i=1}^{n} l(y_i, \hat{y}_i^{(t)})
\end{equation}

where $\hat{y}_i^{(t)}$ is the prediction at step $t$. Minimizing only this term often leads to complex models that fit the noise in the training data (overfitting).

\subsubsection{Introducing Regularization}
To improve generalization, a regularization term $\Omega$ is added to the objective function. The regularized objective becomes~\cite{xgboost}:

\begin{equation}
    \mathcal{L}^{(t)} = \sum_{i=1}^{n} l(y_i, \hat{y}_i^{(t)}) + \sum_{k=1}^{K} \Omega(f_k)
\end{equation}

Here, $\Omega(f_k)$ penalizes the complexity of the $k$-th tree, discouraging overly complex models.

The specific complexity penalty $\Omega(f)$ used in our study includes both $L_1$ and $L_2$ components and is defined as:

\begin{equation}
    \Omega(f) = \gamma T + \alpha \sum_{j=1}^{T} |w_j| + \frac{1}{2}\lambda \sum_{j=1}^{T} w_j^2
\end{equation}
\noindent This regularization term follows the standard XGBoost objective~\cite{xgboost,xgboost_docs}

where:

\begin{itemize}
    \item $T$ is the number of leaves in the tree.
    \item $w_j$ is the weight (score) assigned to the $j$-th leaf.
    \item \textbf{$\gamma$ (Minimum Split Loss):}  This parameter penalizes the number of leaves in a tree by adding a fixed cost for each leaf. A split is only performed if it results in a loss reduction greater than $\gamma$.
    \item \textbf{$\alpha$ (L1 Regularization Parameter):} This parameter controls the strength of the $L_1$ penalty applied to the leaf weights.
    \item \textbf{$\lambda$ (L2 Regularization Parameter):} This parameter controls the $L_2$ penalty on the leaf weights. Unlike $L_1$ regularization, $L_2$ regularization does not enforce sparsity but instead shrinks large weights.
\end{itemize}

The $L_1$ term grows with the number of leaves and the magnitude of their weights.
Since each additional split typically increases the number of leaves $T$, a larger
$\alpha$ discourages overly complex trees and encourages fewer splits unless they
significantly reduce the loss.

Although the penalty does not explicitly depend on depth, deeper trees often
contain more leaves, which indirectly increases the total regularization cost.

\subsection{Evaluation Metrics for the Models}
\label{app:model_metrics}
To assess the predictive performance of the tree ensemble models, we utilize standard binary 
classification metrics derived from the confusion matrix. Let $TP$ (True Positives) be the number 
of positive instances correctly classified, $TN$ (True Negatives) be the number of negative instances 
correctly classified, $FP$ (False Positives) be the number of negative instances incorrectly classified
as positive, and $FN$ (False Negatives) be the number of positive instances incorrectly classified as 
negative.
\begin{table}[t]
\centering
\caption{Accuracy, Precision, Recall, F1, AUC for all 24 models across varying Depths ($d$), Number of Trees ($n$), and L1 Regularization parameters ($\alpha$).}
\label{tab:all_metrics}
\resizebox{\textwidth}{!}{%
\begin{tabular}{ll|ccccc|ccccc|ccccc|ccccc}
\toprule
 &  & \multicolumn{5}{c|}{\textbf{$\alpha = 0$}} & \multicolumn{5}{c|}{\textbf{$\alpha = 1$}} & \multicolumn{5}{c|}{\textbf{$\alpha = 5$}} & \multicolumn{5}{c}{\textbf{$\alpha = 10$}} \\
\textbf{Depth} & \textbf{Trees} & \textbf{Acc} & \textbf{Prec} & \textbf{Rec} & \textbf{F1} & \textbf{AUC} & \textbf{Acc} & \textbf{Prec} & \textbf{Rec} & \textbf{F1} & \textbf{AUC} & \textbf{Acc} & \textbf{Prec} & \textbf{Rec} & \textbf{F1} & \textbf{AUC} & \textbf{Acc} & \textbf{Prec} & \textbf{Rec} & \textbf{F1} & \textbf{AUC} \\ \midrule
\multirow{3}{*}{\textbf{d=3}} & \textbf{n=20} & 0.859 & 0.671 & 0.229 & 0.341 & 0.857 & 0.859 & 0.671 & 0.226 & 0.338 & 0.857 & 0.859 & 0.672 & 0.225 & 0.337 & 0.857 & 0.858 & 0.675 & 0.218 & 0.330 & 0.857 \\
 & \textbf{n=30} & 0.858 & 0.680 & 0.213 & 0.325 & 0.863 & 0.857 & 0.665 & 0.210 & 0.319 & 0.864 & 0.858 & 0.669 & 0.217 & 0.328 & 0.864 & 0.859 & 0.681 & 0.222 & 0.335 & 0.863 \\
 & \textbf{n=40} & 0.859 & 0.684 & 0.218 & 0.330 & 0.867 & 0.858 & 0.672 & 0.216 & 0.327 & 0.867 & 0.857 & 0.668 & 0.209 & 0.319 & 0.867 & 0.858 & 0.668 & 0.218 & 0.329 & 0.867 \\ \midrule
\multirow{3}{*}{\textbf{d=4}} & \textbf{n=20} & 0.859 & 0.676 & 0.227 & 0.339 & 0.864 & 0.859 & 0.676 & 0.228 & 0.341 & 0.865 & 0.858 & 0.654 & 0.229 & 0.339 & 0.865 & 0.859 & 0.666 & 0.231 & 0.343 & 0.864 \\
 & \textbf{n=30} & 0.859 & 0.677 & 0.229 & 0.342 & 0.869 & 0.859 & 0.666 & 0.239 & 0.352 & 0.869 & 0.859 & 0.652 & 0.237 & 0.348 & 0.869 & 0.859 & 0.666 & 0.234 & 0.346 & 0.868 \\
 & \textbf{n=40} & 0.859 & 0.656 & 0.254 & 0.366 & 0.872 & 0.859 & 0.656 & 0.252 & 0.364 & 0.872 & 0.859 & 0.652 & 0.250 & 0.361 & 0.872 & 0.859 & 0.654 & 0.245 & 0.356 & 0.871 \\ \bottomrule
\end{tabular}%
}
\end{table}
The metrics presented in Table~\ref{tab:all_metrics} are defined as follows:

\begin{itemize}
    \item \textbf{Accuracy (Acc):} The ratio of correctly predicted observations to the total observations. It measures the overall effectiveness of the classifier.
    \begin{equation}
        \text{Accuracy} = \frac{TP + TN}{TP + TN + FP + FN}
    \end{equation}

    \item \textbf{Precision (Prec):} Also known as the positive predictive value, this metric measures the accuracy of the positive predictions. It indicates how many of the instances predicted as positive are actually positive.
    \begin{equation}
        \text{Precision} = \frac{TP}{TP + FP}
    \end{equation}

    \item \textbf{Recall (Rec):} Also known as sensitivity or the true positive rate (TPR), this measures the ability of the classifier to find all the positive samples. It indicates what proportion of actual positive instances were identified correctly.
    \begin{equation}
        \text{Recall} = \frac{TP}{TP + FN}
    \end{equation}

    \item \textbf{F1-Score (F1):} The harmonic mean of Precision and Recall. It is particularly useful when the class distribution is imbalanced, as it seeks a balance between precision and recall.
    \begin{equation}
        \text{F1} = 2 \cdot \frac{\text{Precision} \cdot \text{Recall}}{\text{Precision} + \text{Recall}}
    \end{equation}

    \item \textbf{AUC (Area Under the ROC Curve):} The Receiver Operating Characteristic (ROC) curve plots the True Positive Rate (Recall) against the False Positive Rate ($FPR = \frac{FP}{TN + FP}$) at various threshold settings. The AUC represents the degree of separability between classes. Probabilistically, the AUC corresponds to the probability that the classifier ranks a randomly chosen positive instance higher than a randomly chosen negative instance. An AUC of 0.5 suggests no discrimination (random guessing), while an AUC of 1.0 implies perfect discrimination.
\end{itemize}
\label{app:motivation}

\section{Example Run: Step-by-Step Execution of $\xcount$}
\label{app:example_run}
This section provides a detailed walkthrough of the algorithm using a example. 
We demonstrate how $\xcount$ processes a small decision tree ensemble to count sensitive regions,
referencing the procedures outlined in Section~\ref{sec:algorithm} and the actual data structures used in the implementation.

\subsection{Input Configuration}

We consider an XGBoost regression model with 3 trees of depth 2 trained on the Diabetes dataset. The configuration
 for the sensitivity query is a sensitive feature $f_2$ (feature index 2), a bit distance $d = 1$, and a gap threshold $G = 0.1$.

\subsection{Original Ensemble}
Figure~\ref{fig:original_ensemble} shows the three decision trees ($\mathcal{T}$). Note that the trees split on the sensitive feature we chose (i.e. $f_2$), 
as well as the non-sensitive features ($f_3$ and $f_8$).

\begin{figure}[!htbp]
\centering
\begin{subfigure}{0.33\textwidth}
    \includegraphics[width=\textwidth]{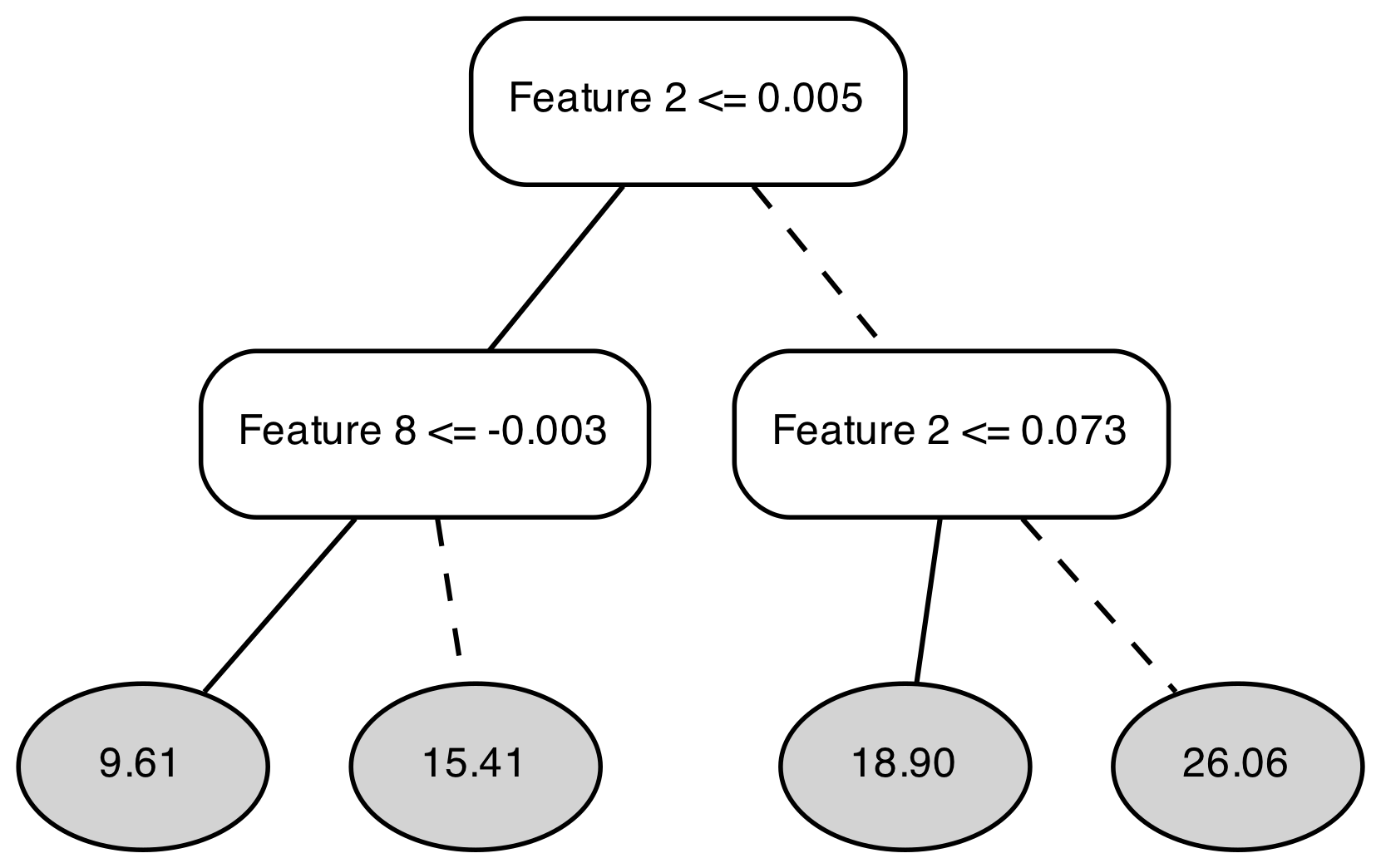}
    \caption{Tree 0}
\end{subfigure}
\hfill
\begin{subfigure}{0.33\textwidth}
    \includegraphics[width=\textwidth]{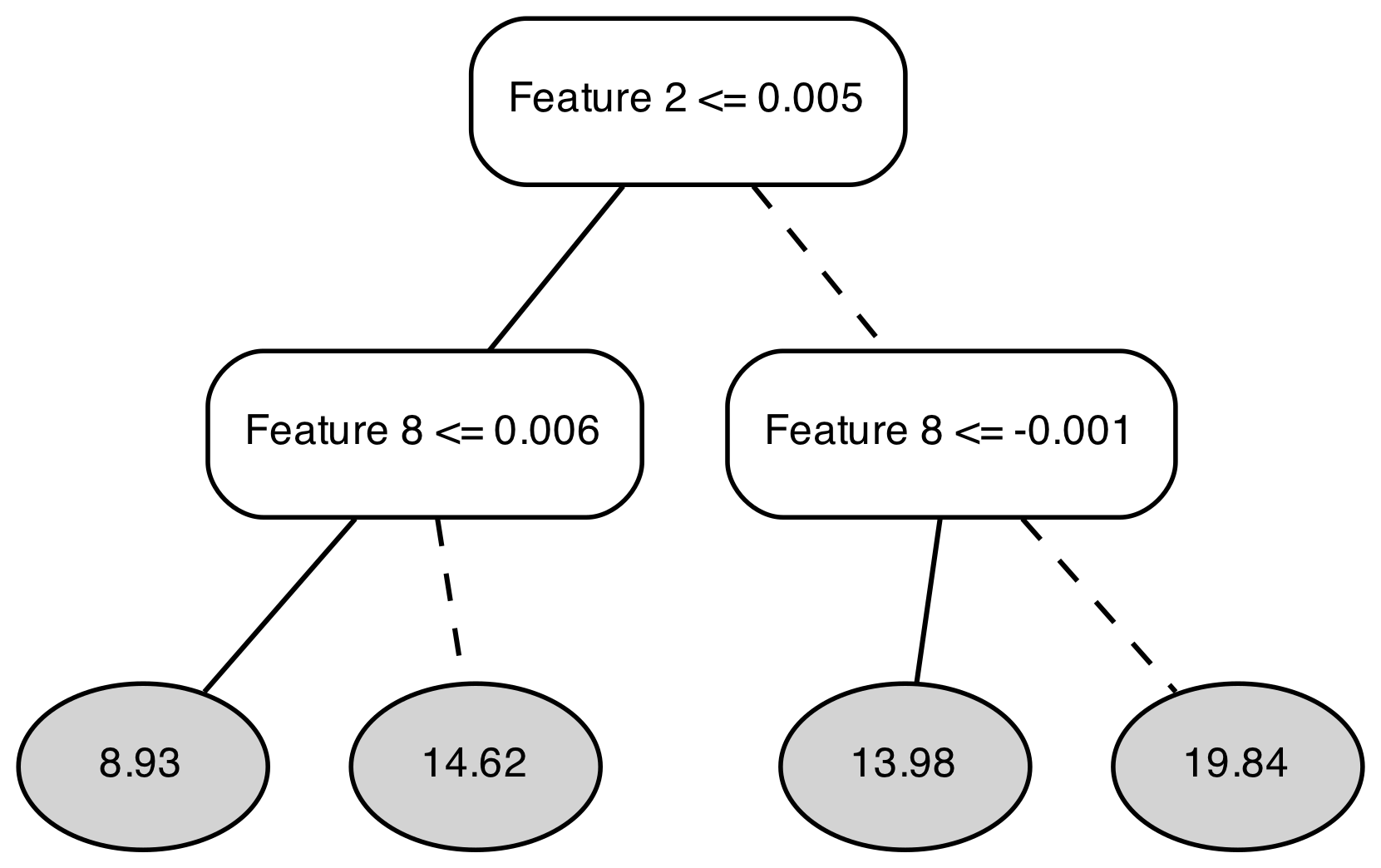}
    \caption{Tree 1}
\end{subfigure}
\hfill
\begin{subfigure}{0.33\textwidth}
    \includegraphics[width=\textwidth]{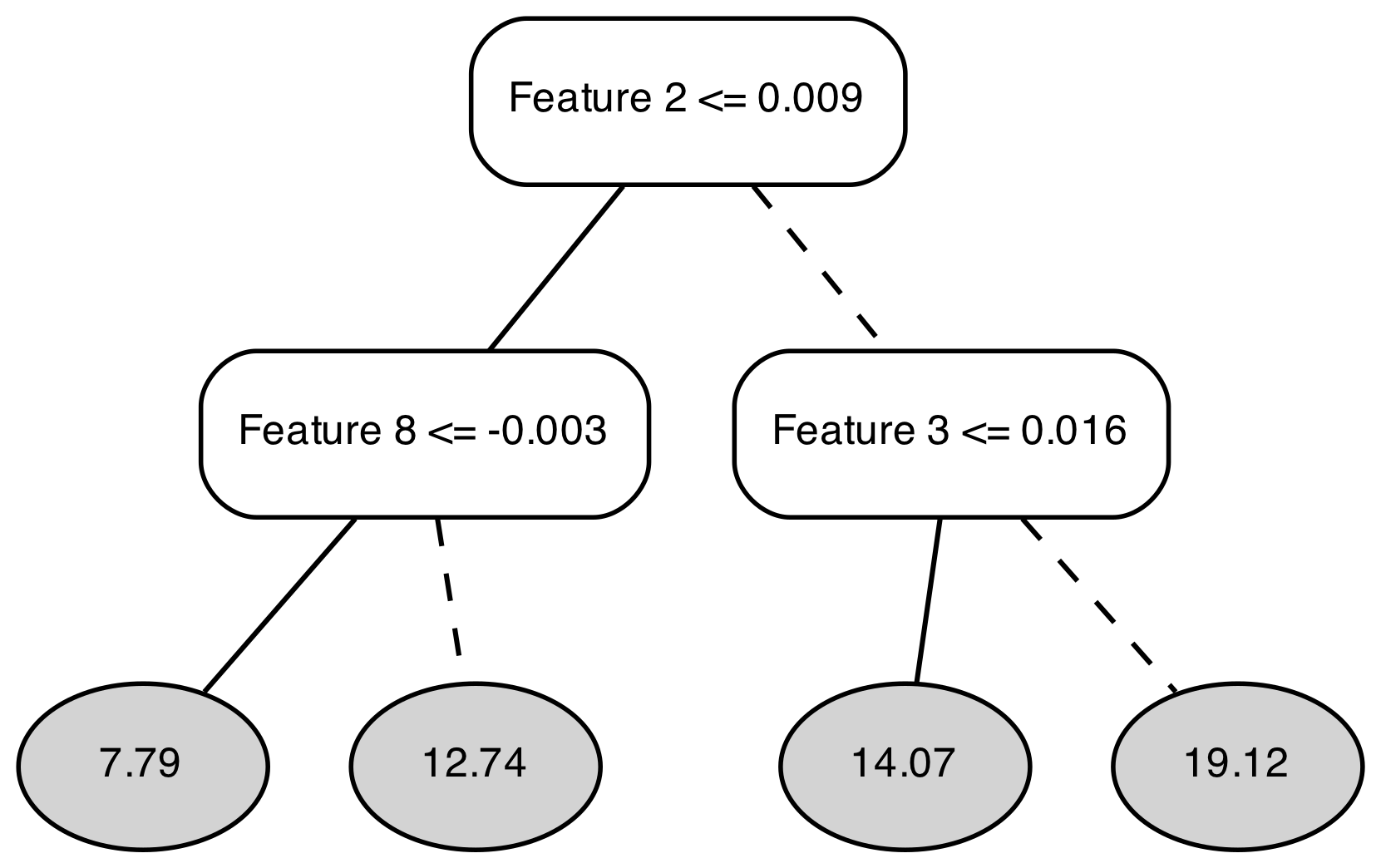}
    \caption{Tree 2}
\end{subfigure}
\caption{The three decision trees in the ensemble. Internal nodes show split conditions of the form ``$f_i \leq \theta$''.}
\label{fig:original_ensemble}
\end{figure}

\subsection{Extracting Splits and Generating Bitmasks}

The execution begins by traversing the ensemble and identify all split conditions. 
The extracted thresholds on the feature $f_2$ are $\theta_1 = 0.005$, $\theta_2 = 0.009$, and $\theta_3 = 0.073$.

Following Algorithm~\ref{alg:bitmaskgen} ($\bitmaskgen$), we enforce the monotonicity constraint ($b_{\theta_1} \Rightarrow b_{\theta_2} \Rightarrow b_{\theta_3}$) to generate the valid discretized regions. In the implementation, this set of valid masks ($MaskSet$) is stored as a vector of bitmasks:

\begin{center}
\begin{tabular}{|c|c|l|}
\hline
\textbf{Mask ID} & \textbf{BitMask} $(b_{\theta_1}, b_{\theta_2}, b_{\theta_3})$ & \textbf{Interpretation ($f_2$ range)} \\
\hline
$m_0$ & $(0, 0, 0)$ & $f_2 > 0.073$ \\
$m_1$ & $(0, 0, 1)$ & $0.009 < f_2 \leq 0.073$ \\
$m_2$ & $(0, 1, 1)$ & $0.005 < f_2 \leq 0.009$ \\
$m_3$ & $(1, 1, 1)$ & $f_2 \leq 0.005$ \\
\hline
\end{tabular}
\end{center}

\subsection{Subproblem Generation}

The algorithm iterates through the $MaskSet$ to generate subproblems.
The total number of subproblems depends on the number of valid masks ($M = |MaskSet|$) and the bit distance $d$.
For a single sensitive feature with monotonic constraints, the masks form a linear sequence where the Hamming distance between any two masks $m_i$ and $m_j$ corresponds to their index difference $|i-j|$.
A subproblem is generated for every ordered pair $(m_i, m_j)$ where $1 \le |i-j| \le d$.

For this example, with $M=4$ and $d=1$, we consider only immediate neighbors ($|i-j|=1$). The total count is calculated as:
\[
\text{Total Subproblems} = \sum_{k=1}^{d} 2 \cdot (M - k) = 2 \cdot (4 - 1) = 6
\]
These are collected into the \texttt{subpArray} structure, shown below as a vector of pairs:

\begin{center}
\begin{tabular}{|c|c|}
\hline
\textbf{Index} & \textbf{Element} (\texttt{pair<Mask, Mask>}) \\
\hline
0 & $\langle (1,1,1), (0,1,1) \rangle$ \\
1 & $\langle (0,1,1), (1,1,1) \rangle$ \\
2 & $\langle (0,1,1), (0,0,1) \rangle$ \\
3 & $\langle (0,0,1), (0,1,1) \rangle$ \\
4 & $\langle (0,0,1), (0,0,0) \rangle$ \\
5 & $\langle (0,0,0), (0,0,1) \rangle$ \\
\hline
\end{tabular}
\end{center}

We focus our trace on \textbf{Subproblem 1} (Index 1), which contains the pair $\langle (0,1,1), (1,1,1) \rangle$, and for the rest of the subproblems, we summarize the final results.

It is crucial to note that within the context of a specific subproblem, the ADDs generated are \textbf{no longer functions of the sensitive feature $f_2$}. 
The subproblem explicitly fixes the range of $f_2$, causing all $f_2$ nodes to resolve during pruning.

\subsection{Processing Subproblem 1}
In Algorithm \ref{alg:processsubp}, we present the $\mathsf{ProcessSubproblem}$ procedure.
It initializes the $\DiffSum$ ADD to zero and iterates through the trees, as can be seen in line \ref{line:diffsum_init} of 
Algorithm \ref{alg:processsubp}.

\subsubsection{Tree Pruning and Encoding} 
In Algorithm \ref{alg:processsubp}, the $\mathsf{PruneTree}$ function ( Algorithm \ref{alg:prunetree}), is called to prune each tree based on 
the masks in the subproblem. For each tree $T \in \mathcal{T}$, the algorithm prunes it twice: once for Mask $m_2$ and once for Mask $m_3$.
For Mask $m_2$ ($0.005 < f_2 \leq 0.009$), the check $f_2 \leq 0.005$ evaluates to False.
Conversely, for Mask $m_3$ ($f_2 \leq 0.005$), the same check evaluates to True. 
Figure~\ref{fig:pruned_trees_before} displays the ADD structures for all three trees immediately after this pruning step. 
At this stage (before constraints are enforced), the ADDs may still contain paths that are logically inconsistent with the 
feature monotonicity constraints. To ensure the consistency, the algorithm applies the constraints to the ADDs. 
This results in the structures shown in Figure~\ref{fig:pruned_trees}.

\begin{figure}[!htbp]
\centering
\begin{subfigure}{0.2\textwidth}
    \includegraphics[width=\textwidth]{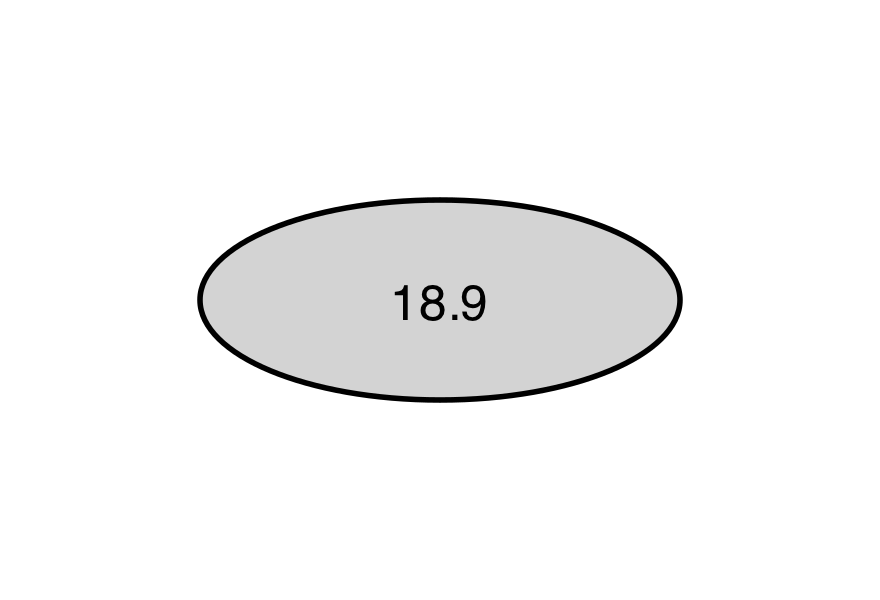}
    \caption{Tree 0, Mask $m_2$}
\end{subfigure}
\hfill
\begin{subfigure}{0.32\textwidth}
    \includegraphics[width=\textwidth]{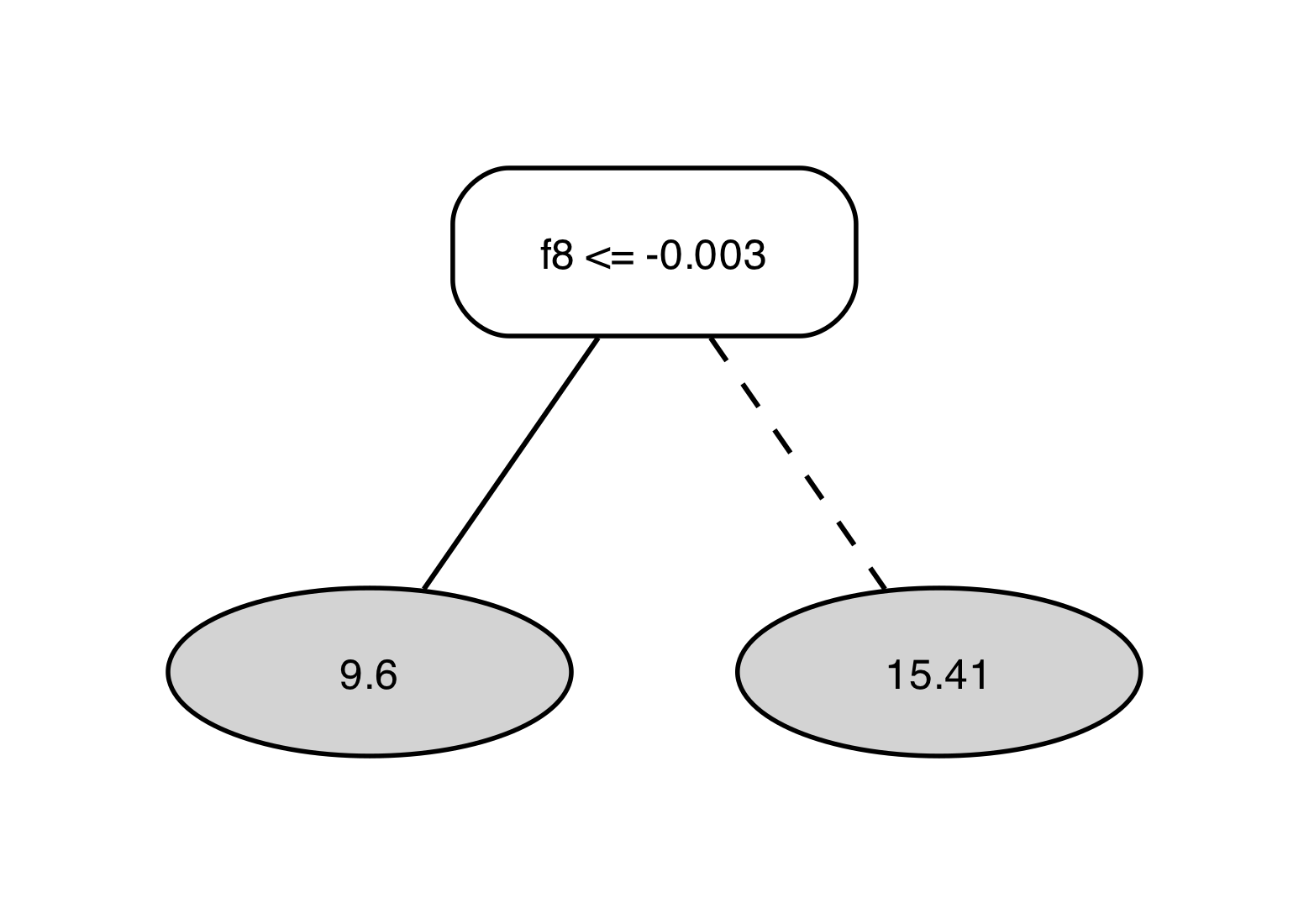}
    \caption{Tree 0, Mask $m_3$}
\end{subfigure}
\hfill
\begin{subfigure}{0.32\textwidth}
    \includegraphics[width=\textwidth]{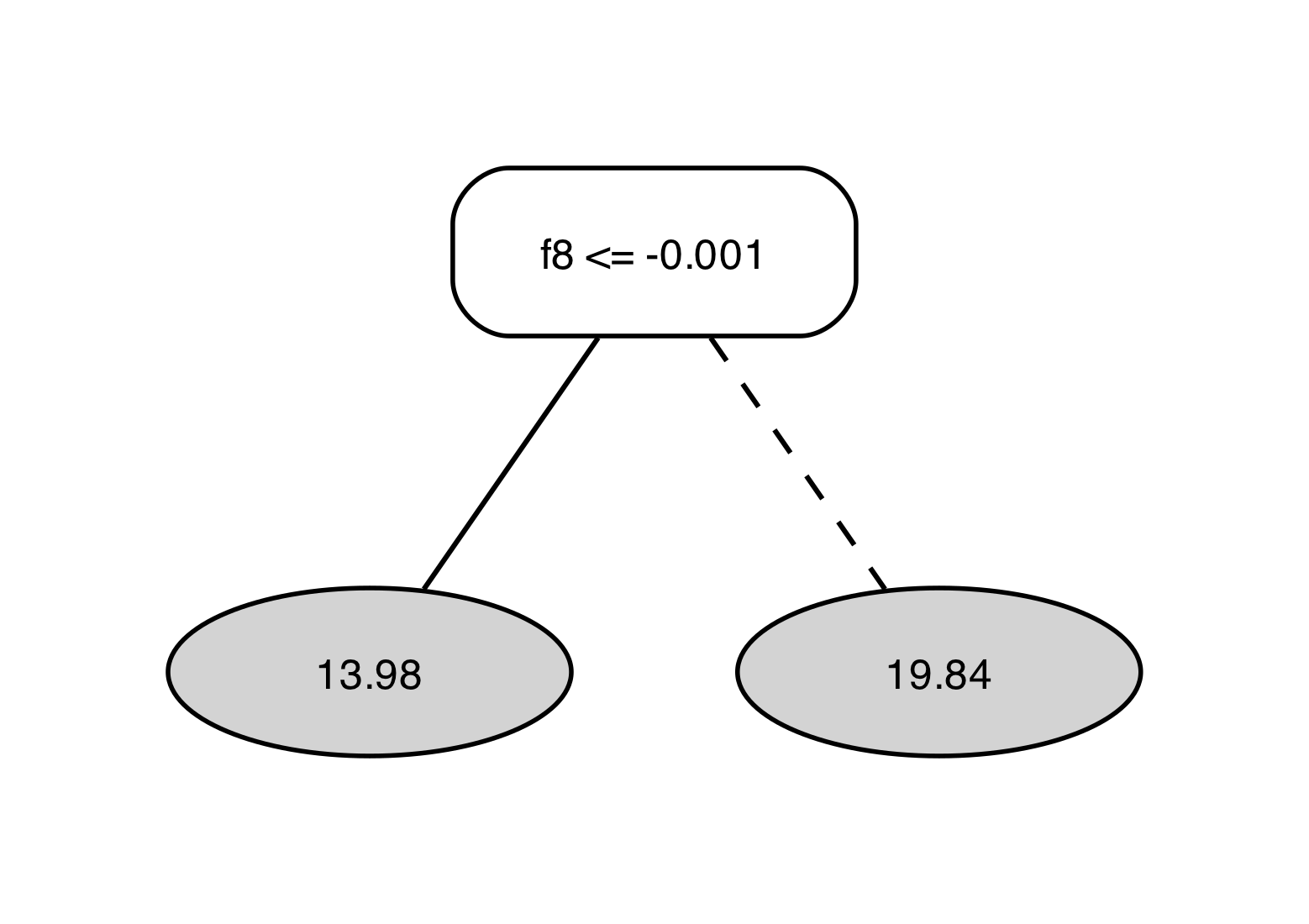}
    \caption{Tree 1, Mask $m_2$}
\end{subfigure}

\begin{subfigure}{0.32\textwidth}
    \includegraphics[width=\textwidth]{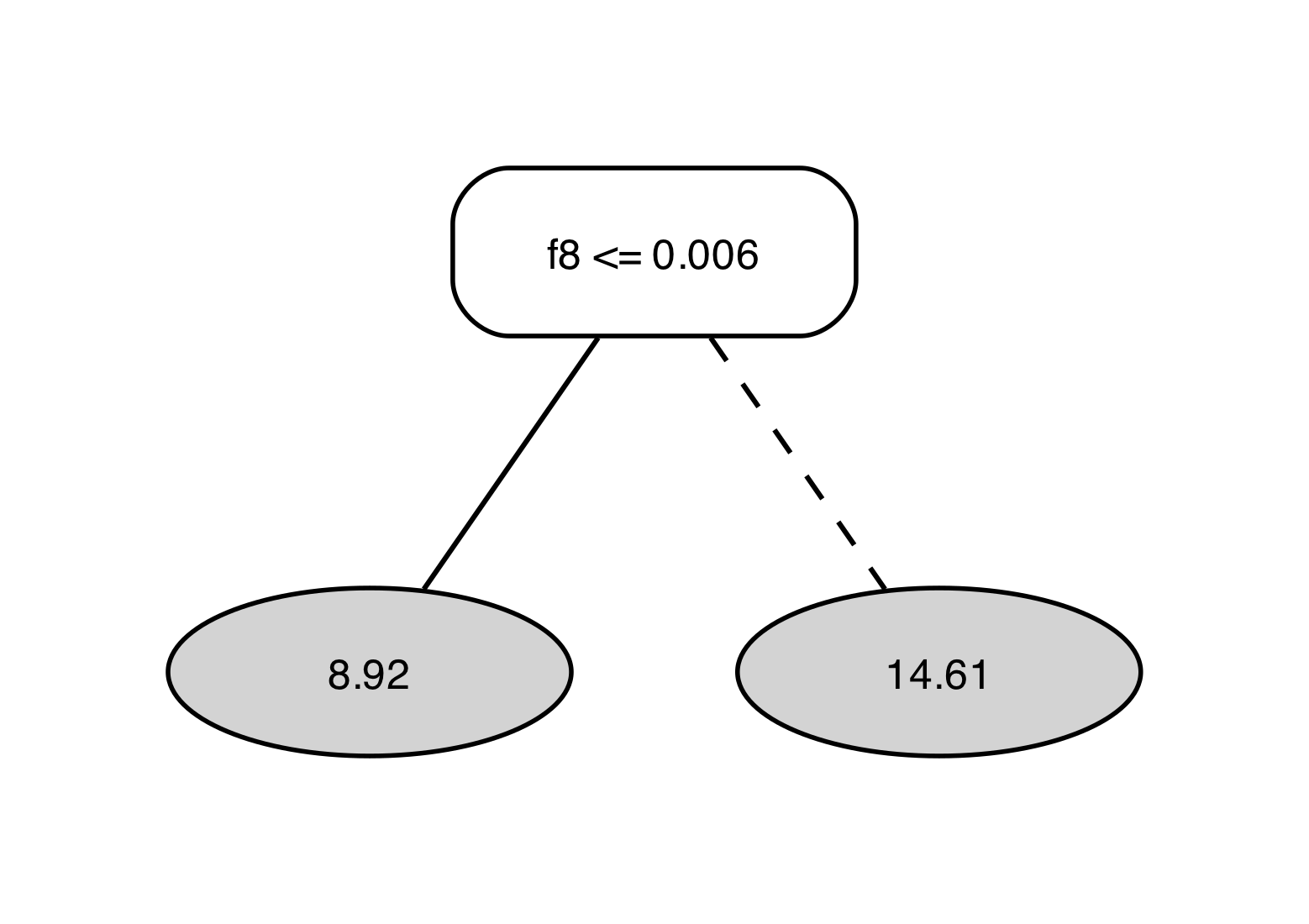}
    \caption{Tree 1, Mask $m_3$}
\end{subfigure}
\hfill
\begin{subfigure}{0.32\textwidth}
    \includegraphics[width=\textwidth]{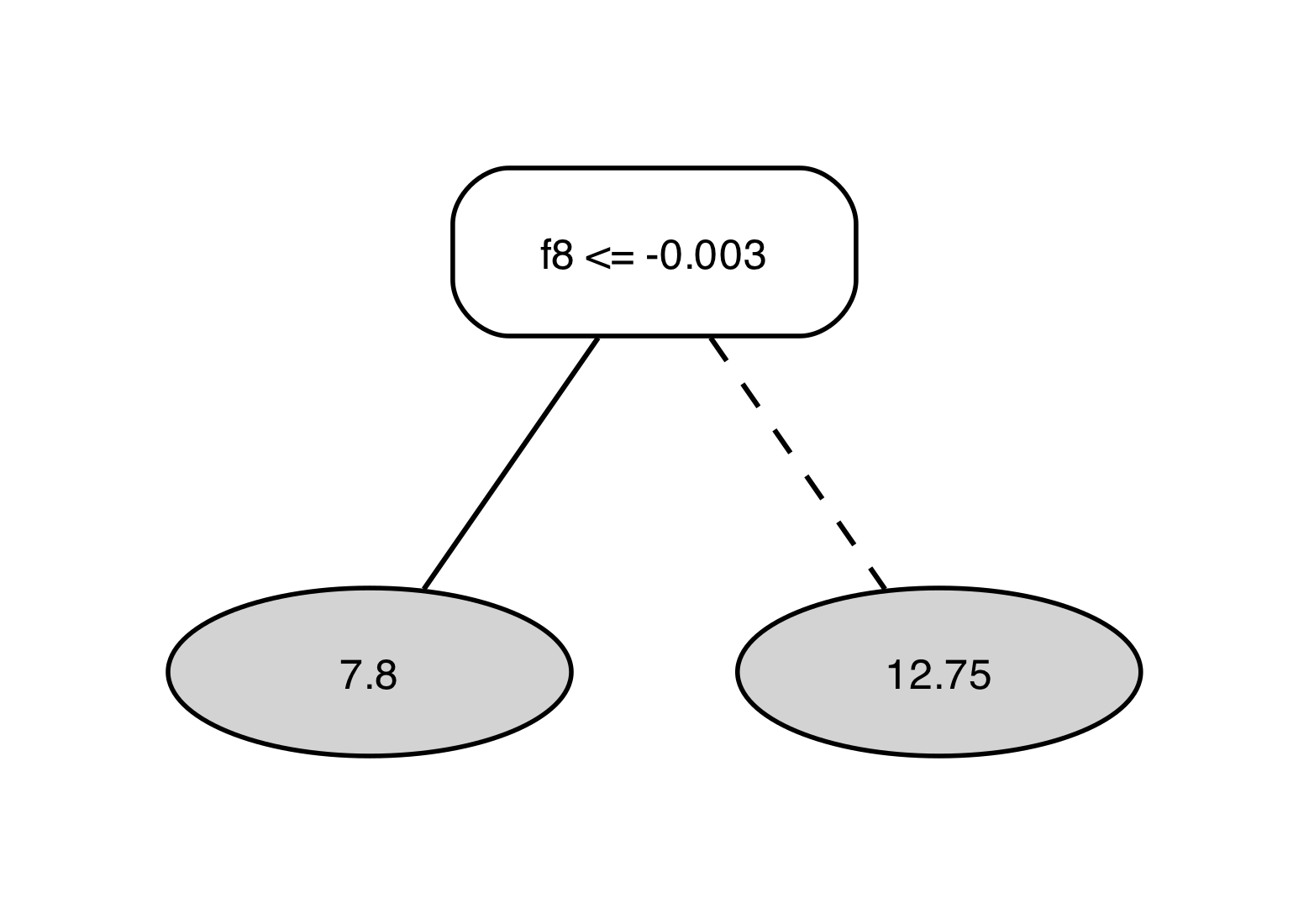}
    \caption{Tree 2, Mask $m_2$}
\end{subfigure}
\hfill
\begin{subfigure}{0.32\textwidth}
    \includegraphics[width=\textwidth]{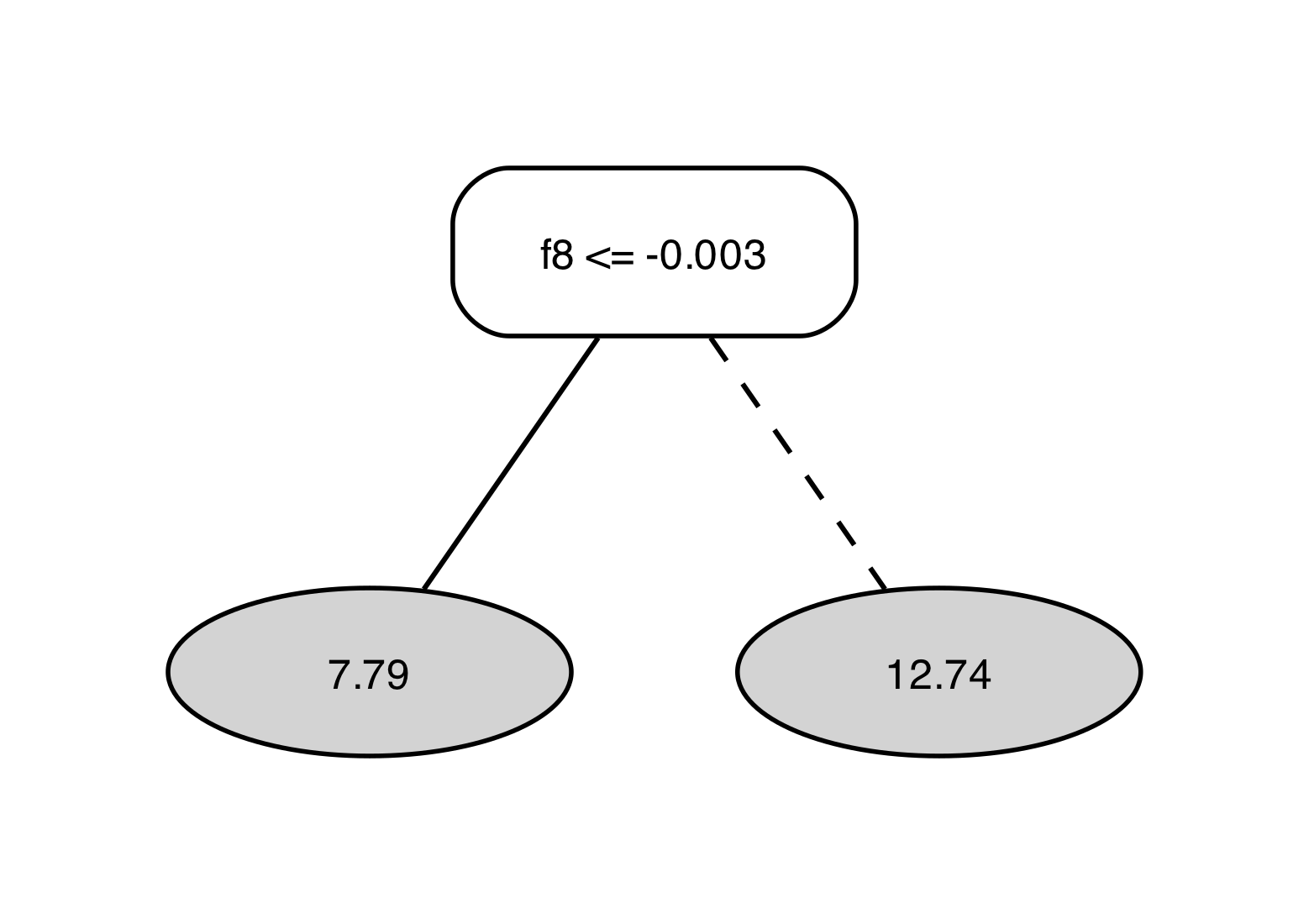}
    \caption{Tree 2, Mask $m_3$}
\end{subfigure}
\caption{ADDs for Trees 0, 1, and 2 immediately after pruning with Masks $m_2$ and $m_3$, but before standardizing constraints are applied.}
\label{fig:pruned_trees_before}
\end{figure}

\begin{figure}[!htbp]
\centering
\begin{subfigure}{0.35\textwidth}
    \includegraphics[width=\textwidth]{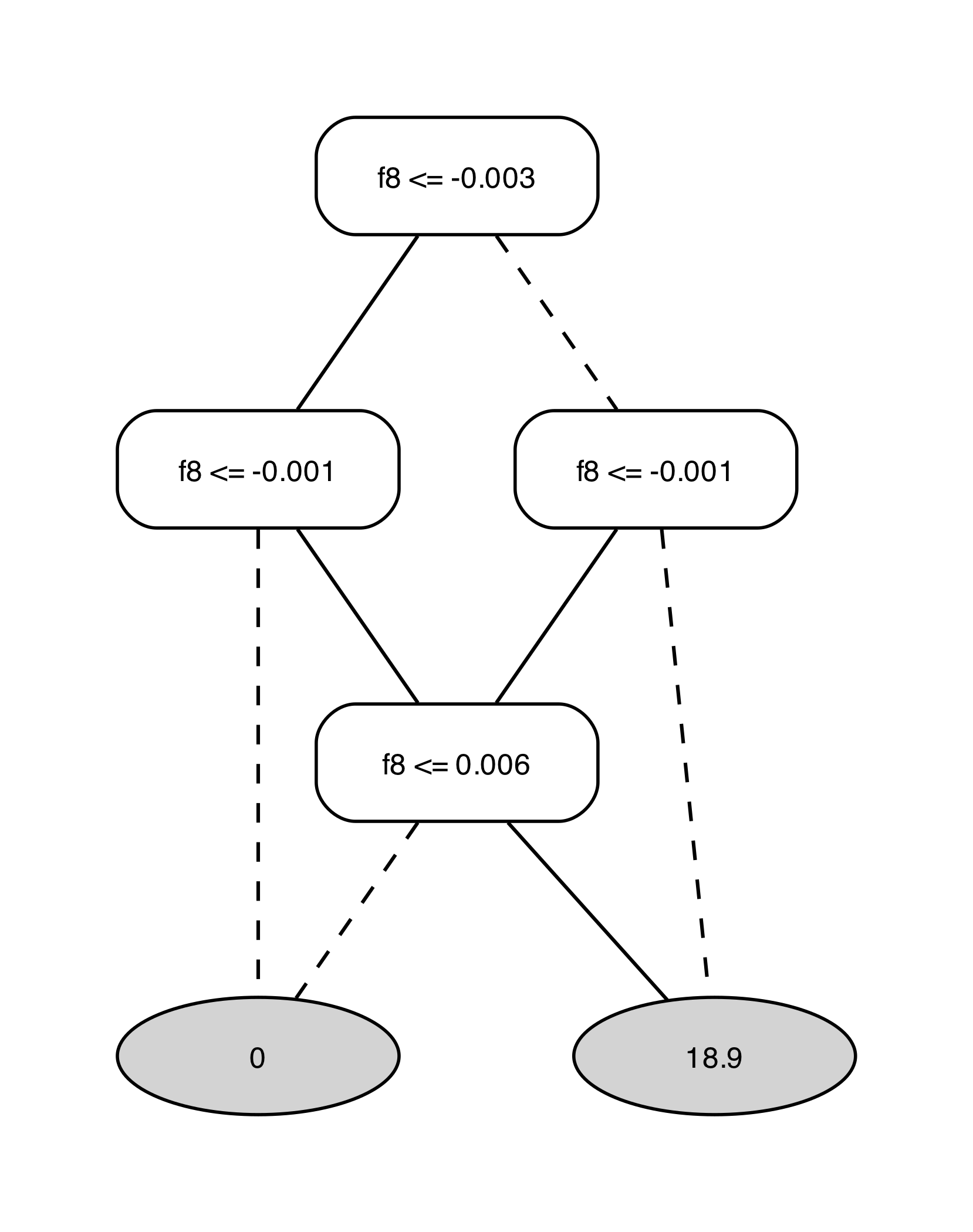}
    \caption{Tree 0, Mask $m_2$}
\end{subfigure}
\hfill
\begin{subfigure}{0.49\textwidth}
    \includegraphics[width=\textwidth]{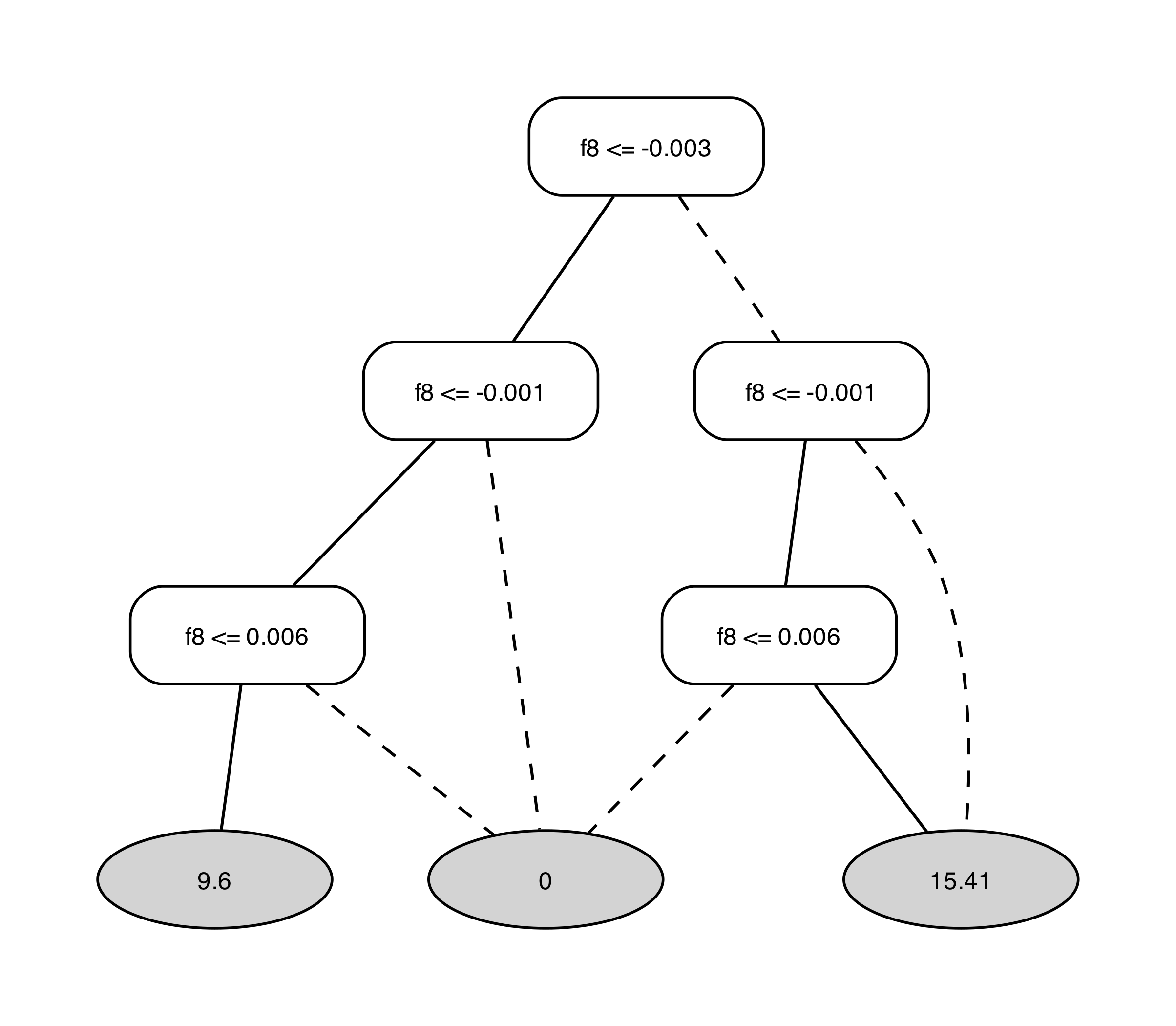}
    \caption{Tree 0, Mask $m_3$}
\end{subfigure}

\begin{subfigure}{0.49\textwidth}
    \includegraphics[width=\textwidth]{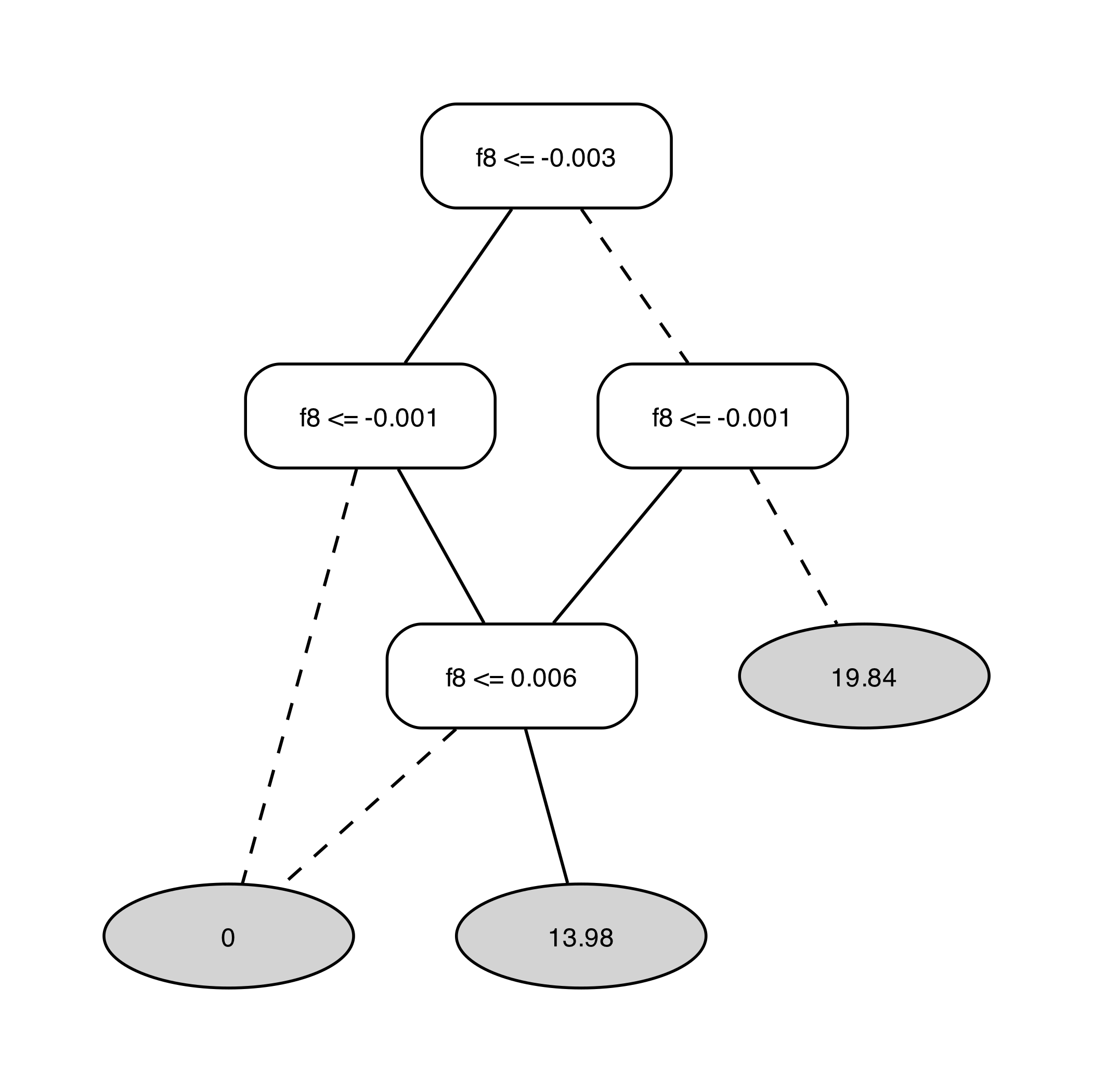}
    \caption{Tree 1, Mask $m_2$}
\end{subfigure}
\hfill
\begin{subfigure}{0.49\textwidth}
    \includegraphics[width=\textwidth]{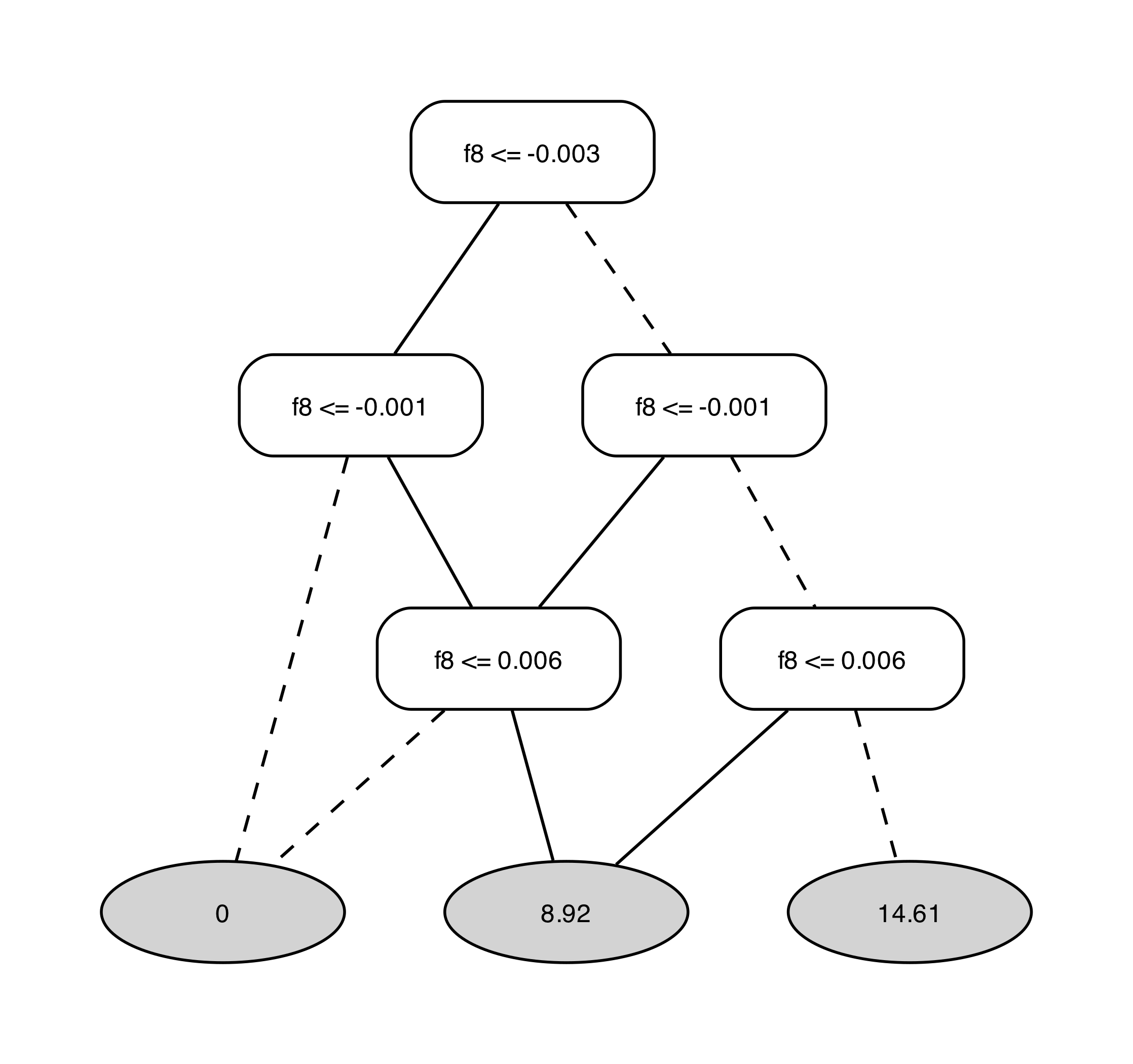}
    \caption{Tree 1, Mask $m_3$}
\end{subfigure}

\begin{subfigure}{0.49\textwidth}
    \includegraphics[width=\textwidth]{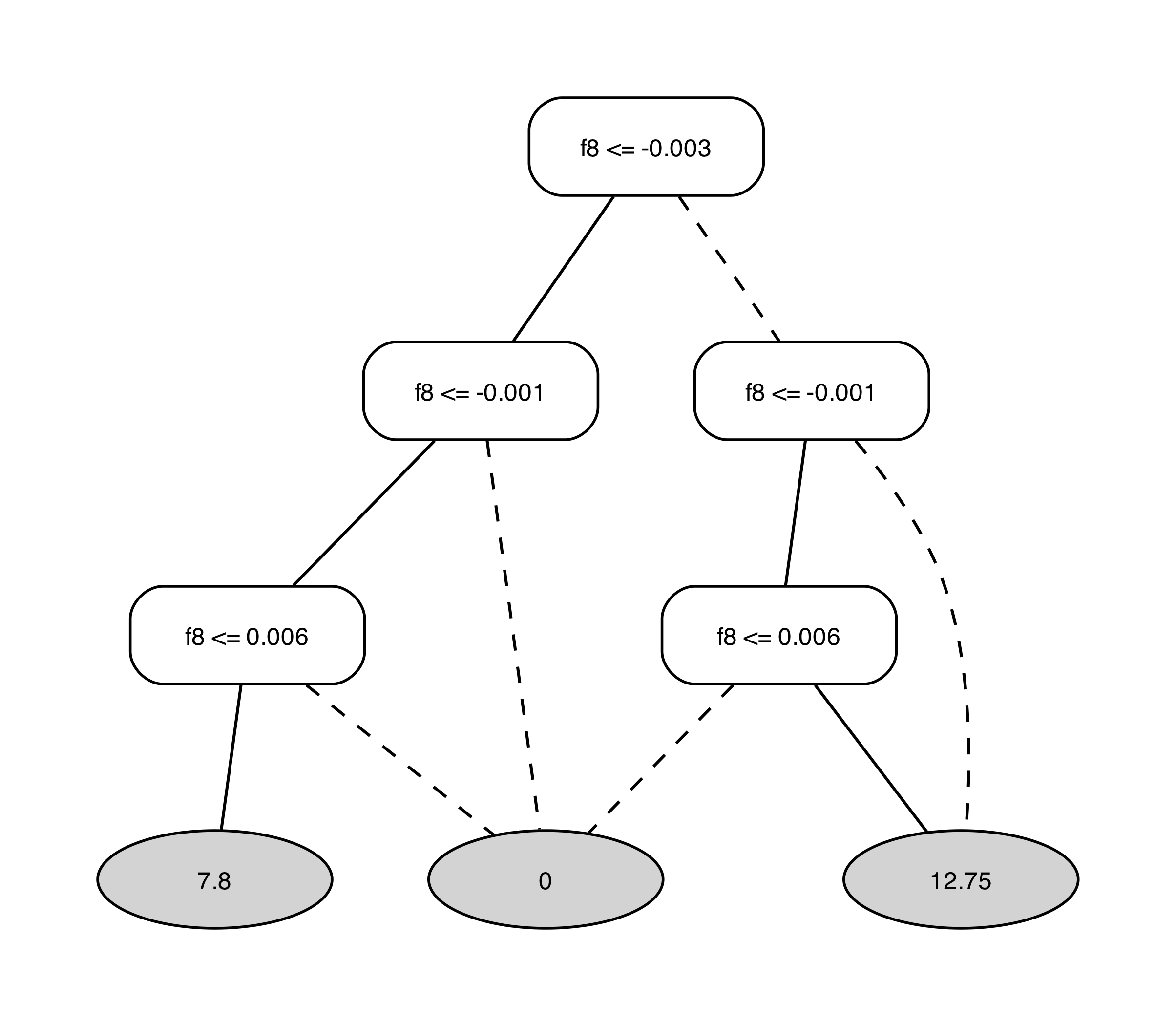}
    \caption{Tree 2, Mask $m_2$}
\end{subfigure}
\hfill
\begin{subfigure}{0.49\textwidth}
    \includegraphics[width=\textwidth]{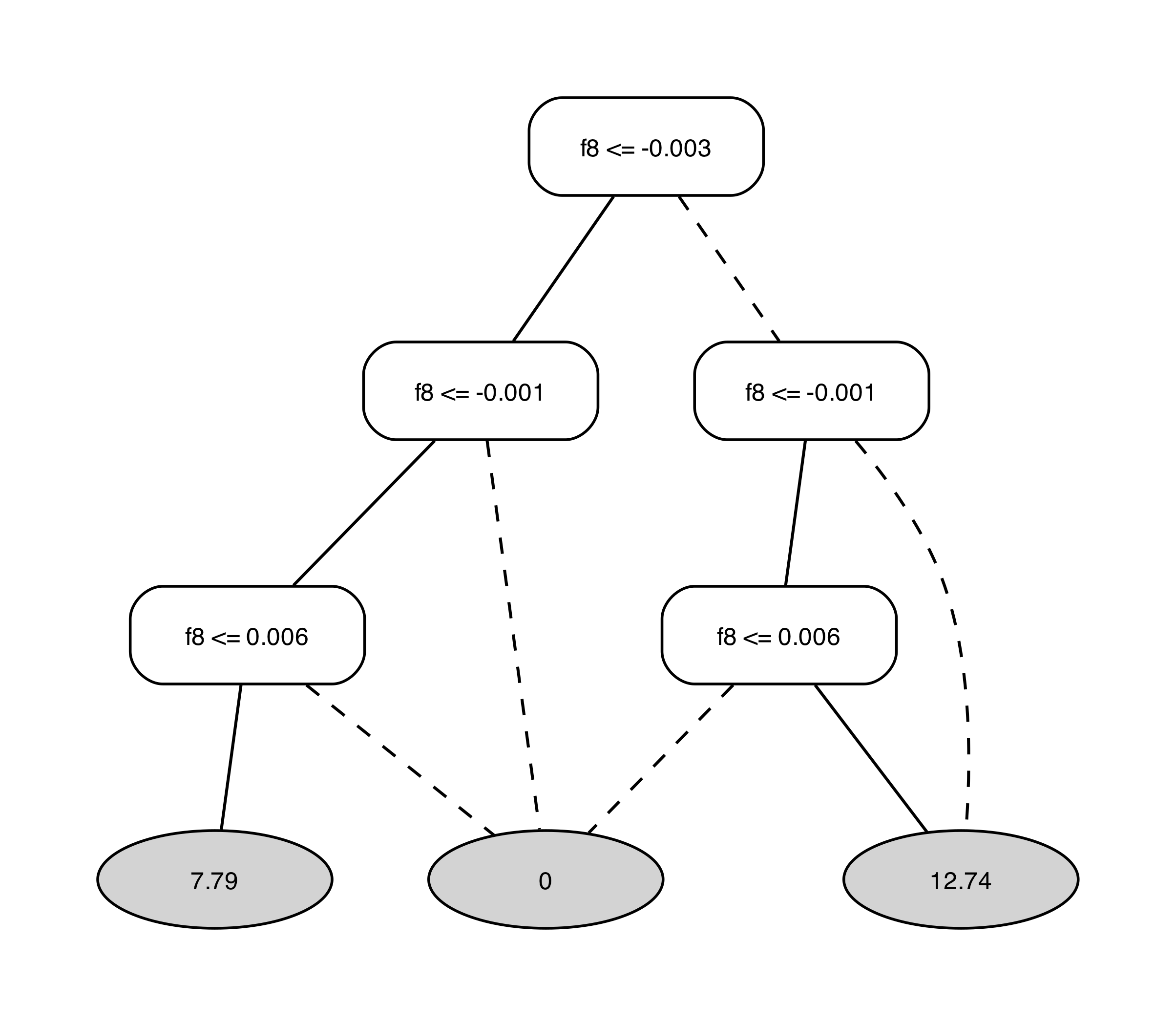}
    \caption{Tree 2, Mask $m_3$}
\end{subfigure}
\caption{ADDs for Trees 0, 1, and 2 after pruning and applying constraints.}
\label{fig:pruned_trees}
\end{figure}

\subsubsection{Subtraction and Accumulation} 
In Line~\ref{line:diffadd} of Algorithm~\ref{alg:processsubp}, the algorithm computes $\DiffADD$, which is the difference between the pair of ADDs obtained by pruning the same tree with the two differing masks.
In Line~\ref{line:diffsum_update}, this difference is accumulated into the $\DiffSum$ ADD. Figure~\ref{fig:diffsum_evolution} shows the growth of the $\DiffSum$ ADD as we aggregate the differences from Tree 0,
 Tree 1, and Tree 2.
\begin{figure}[!htbp]
\centering
\begin{subfigure}{0.35\textwidth}
    \centering
    \includegraphics[width=\textwidth]{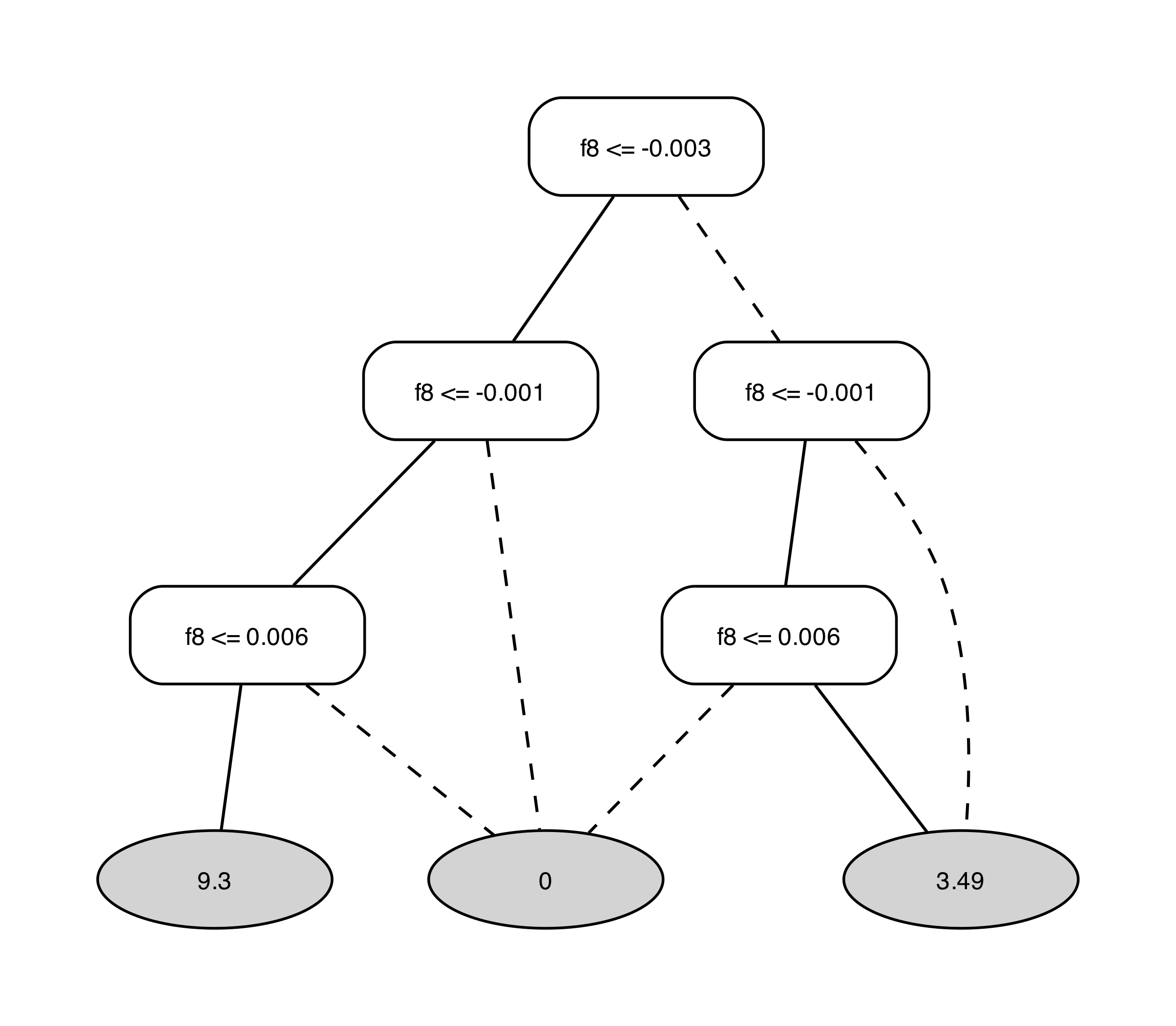}
    \caption{After Tree 0}
\end{subfigure}
\hfill
\begin{subfigure}{0.6\textwidth}
    \centering
    \includegraphics[width=\textwidth]{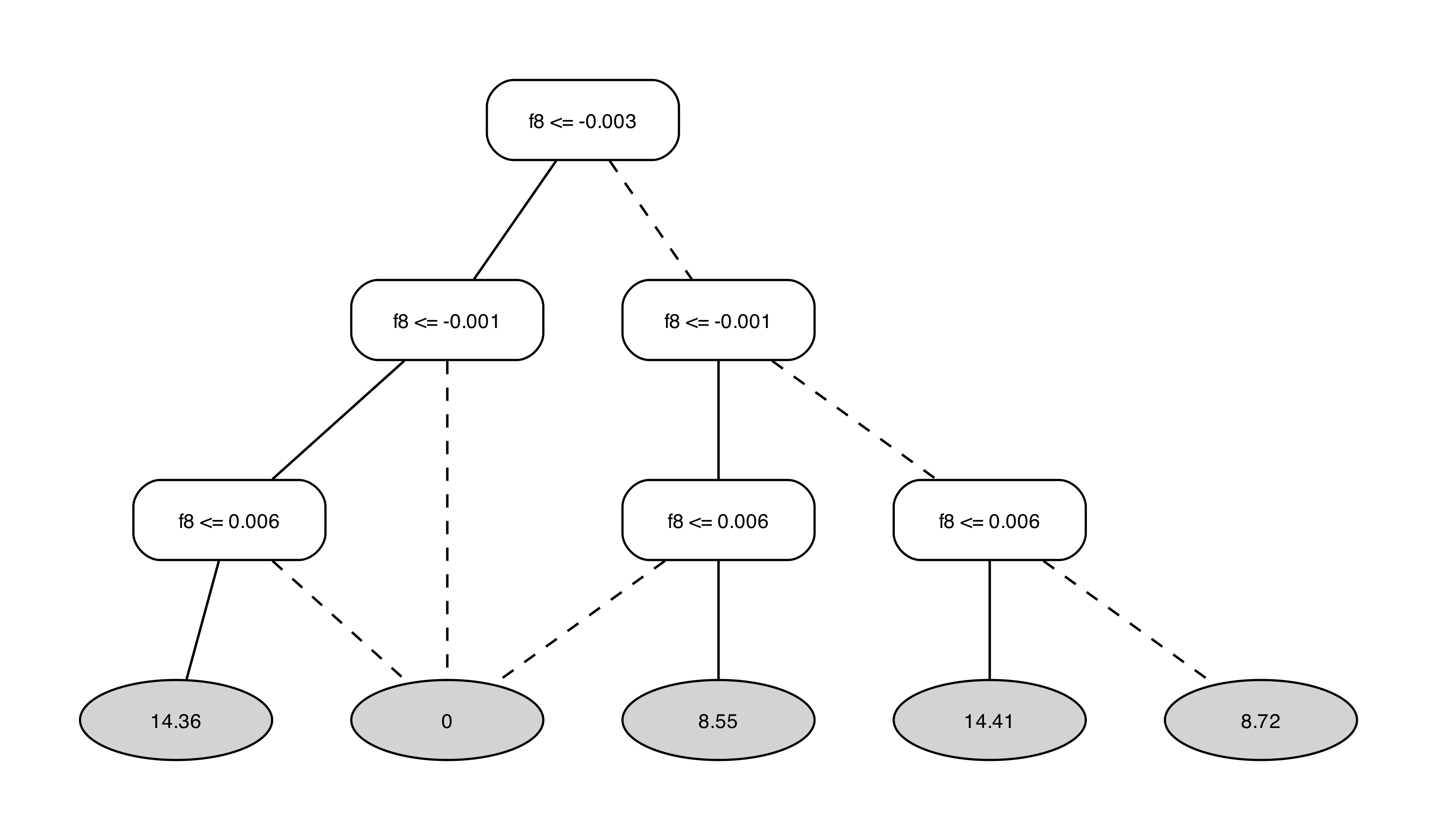}
    \caption{After Tree 1}
\end{subfigure}
\hfill
\begin{subfigure}{0.6\textwidth}
    \centering
    \includegraphics[width=\textwidth]{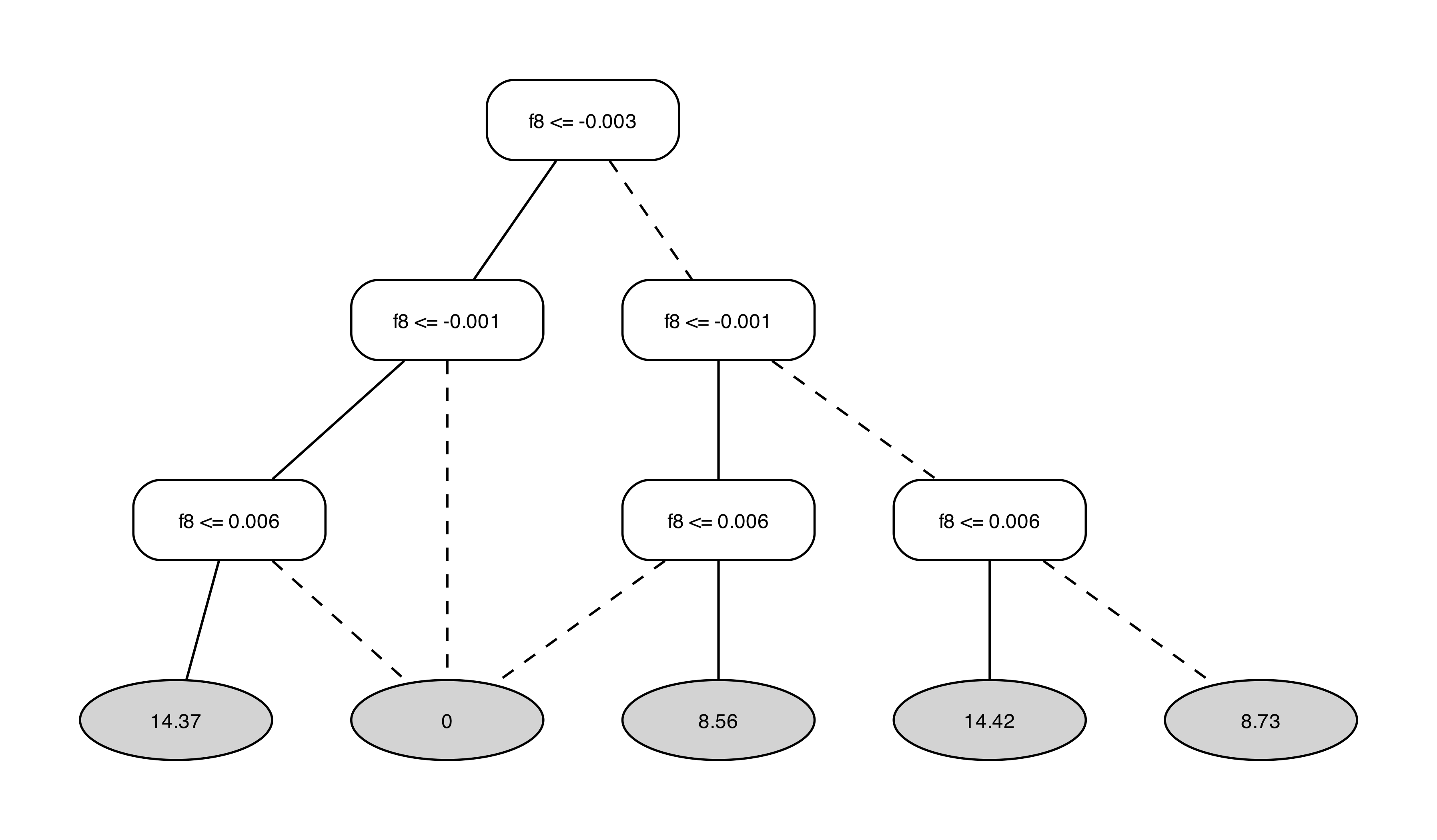}
    \caption{After Tree 2}
    \label{fig:final_tree_sub1}
\end{subfigure}
\caption{Evolution of the cumulative $\DiffSum$ ADD for Subproblem 1. The ADD in \ref{fig:final_tree_sub1} shows the final $\DiffSum$ for the Subproblem.}
\label{fig:diffsum_evolution}
\end{figure}

\subsection{Merging Subproblems and Counting}

After processing all trees, we obtain the symbolic representation of the prediction difference for a subproblem.
 Figure~\ref{fig:final_adds} displays these final ADDs for all 6 subproblems generated by the decomposition.

\begin{figure}[!htbp]
\centering
\begin{subfigure}{0.45\textwidth}
    \centering
    \includegraphics[width=\textwidth]{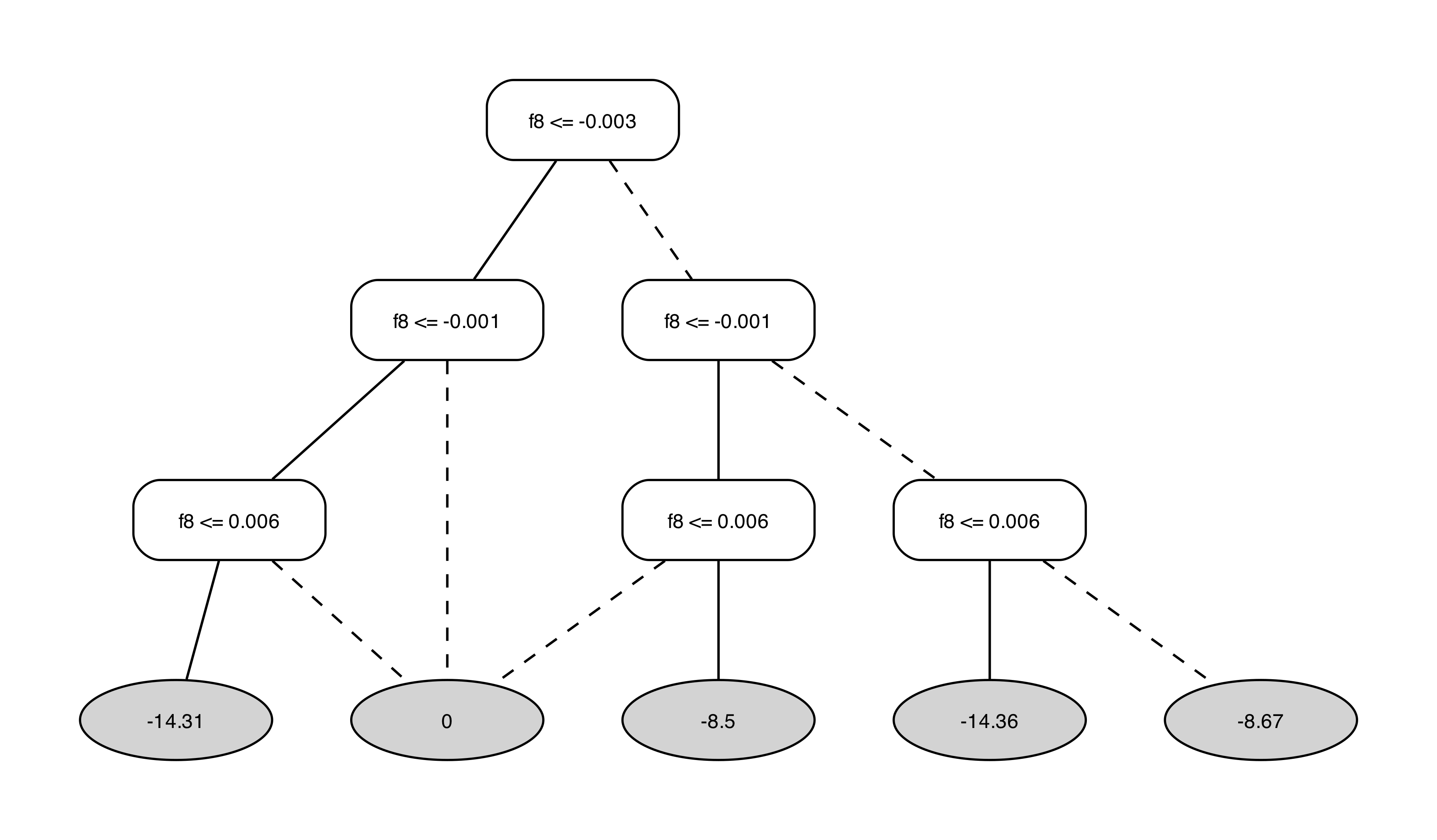}
    \caption{Subproblem 0}
\end{subfigure}
\hfill
\begin{subfigure}{0.45\textwidth}
    \centering
    \includegraphics[width=\textwidth]{example_run_figs/subproblem_1/intermediate_adds/diff_sum_after_tree_2.png}
    \caption{Subproblem 1}
\end{subfigure}

\begin{subfigure}{0.45\textwidth}
    \centering
    \includegraphics[width=\textwidth]{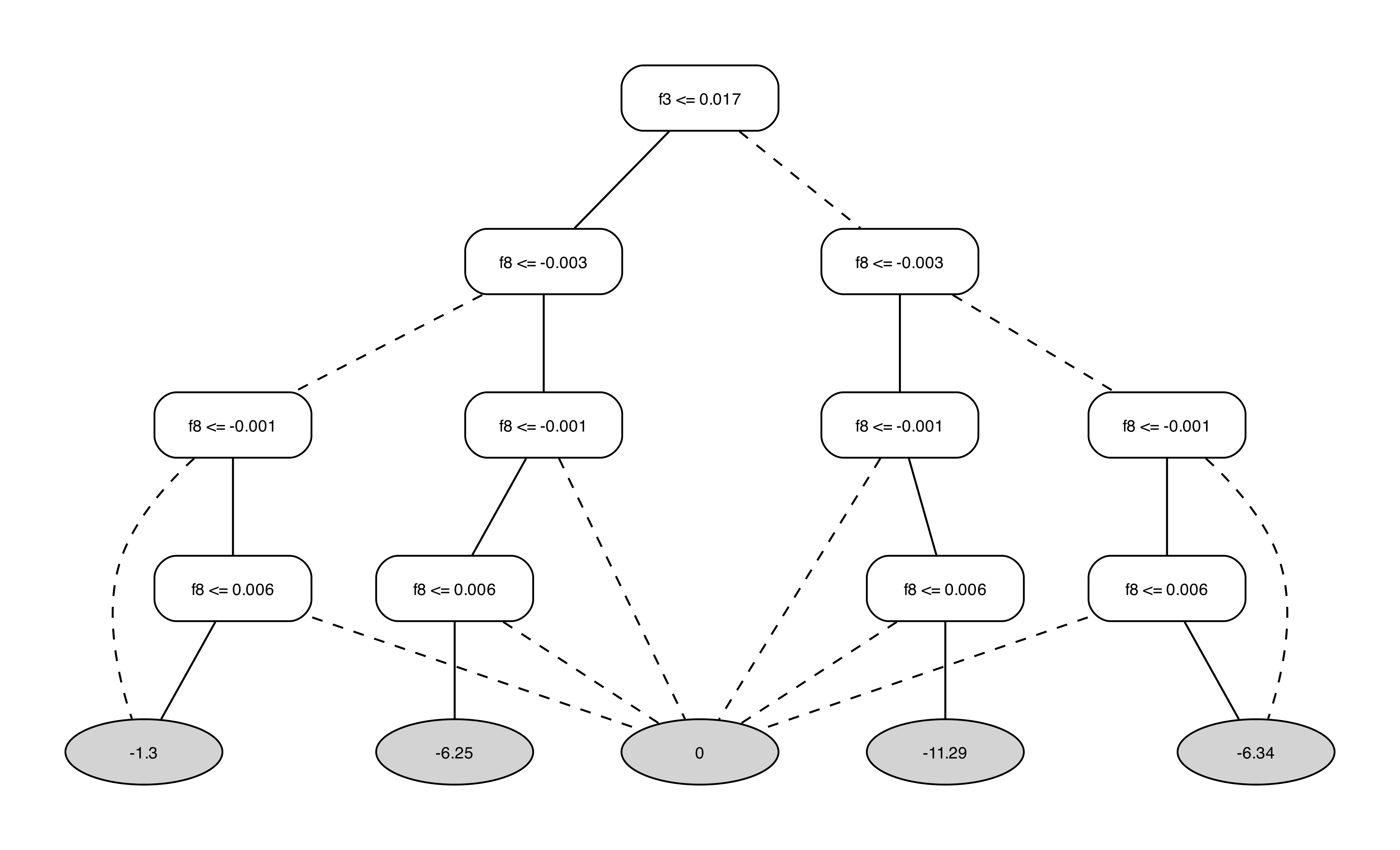}
    \caption{Subproblem 2}
\end{subfigure}
\hfill
\begin{subfigure}{0.45\textwidth}
    \centering
    \includegraphics[width=\textwidth]{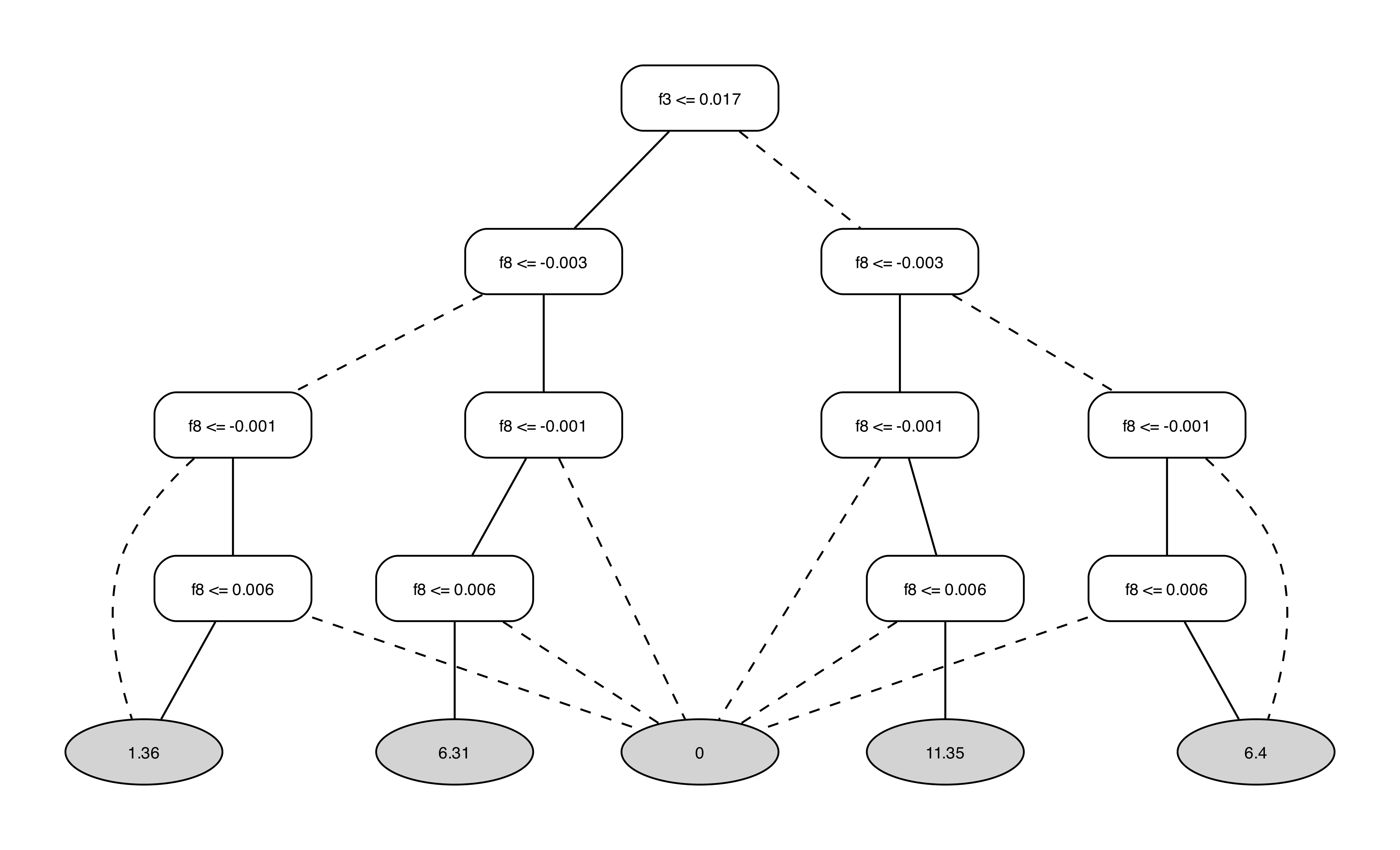}
    \caption{Subproblem 3}
\end{subfigure}

\begin{subfigure}{0.35\textwidth}
    \centering
    \includegraphics[width=\textwidth]{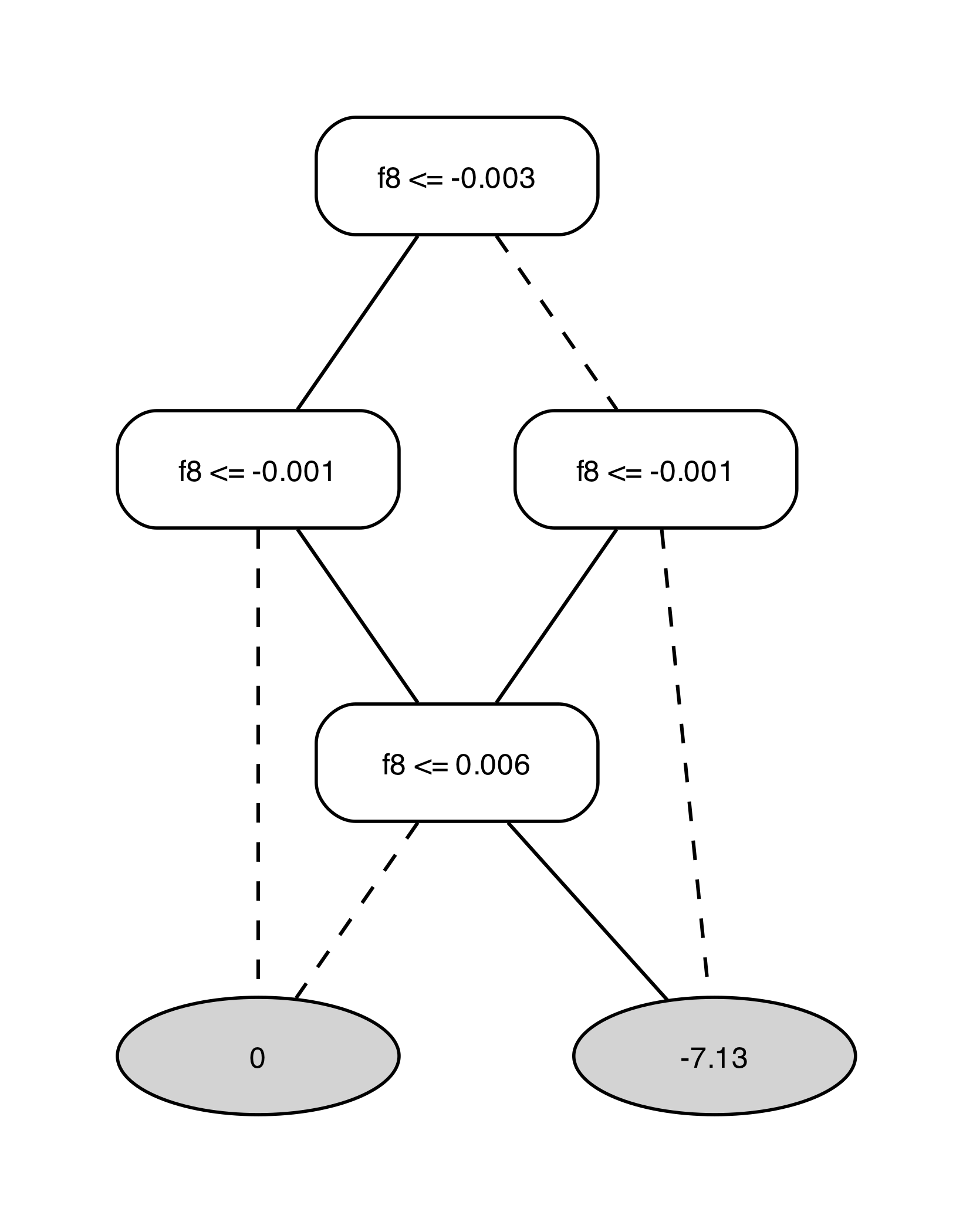}
    \caption{Subproblem 4}
\end{subfigure}
\hfill
\begin{subfigure}{0.35\textwidth}
    \centering
    \includegraphics[width=\textwidth]{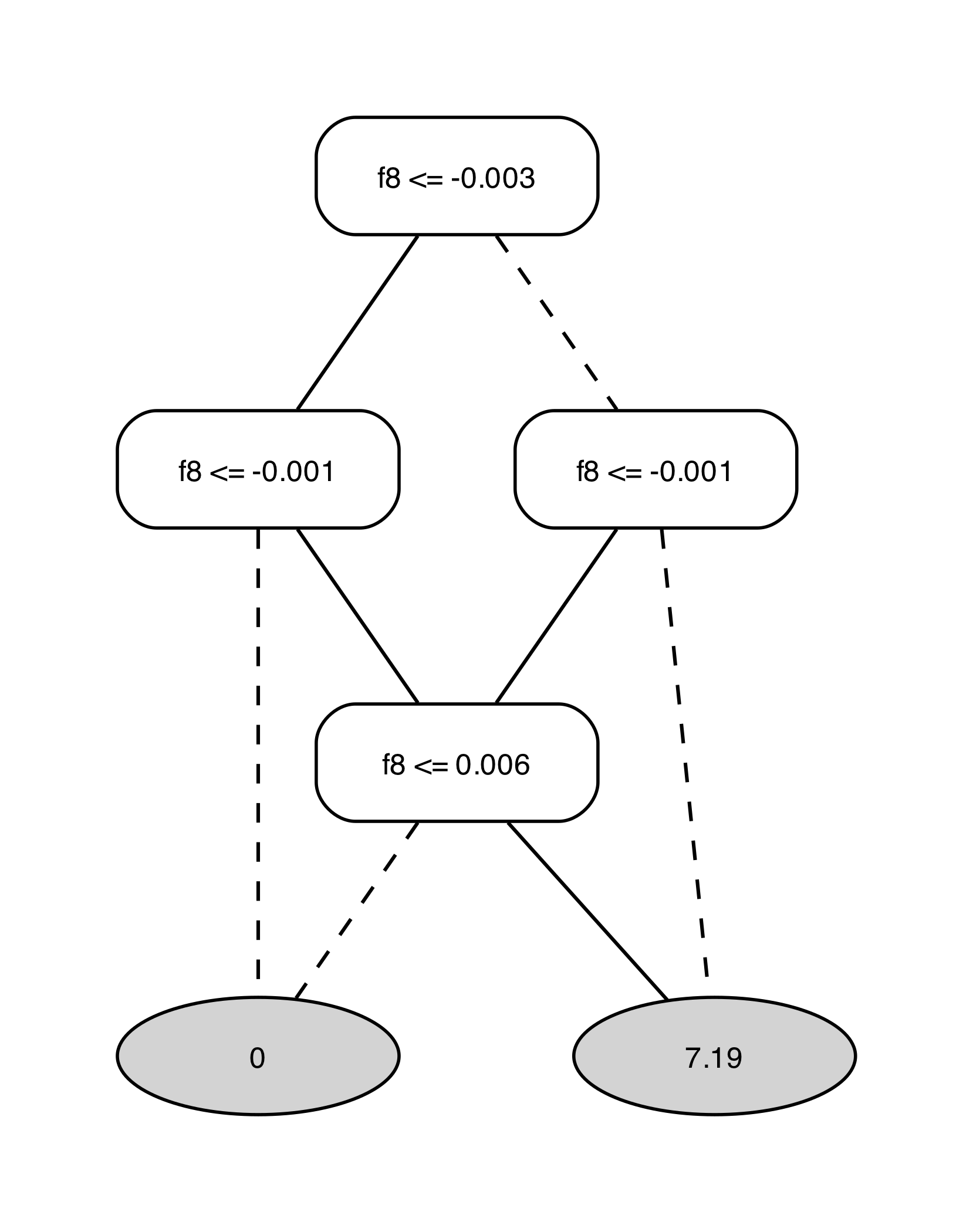}
    \caption{Subproblem 5}
\end{subfigure}
\caption[]{Final $\DiffSum$ ADDs for subproblems 0--5. These graphs represent the accumulated ensemble prediction difference for each mask pair before applying the gap threshold.}
\label{fig:final_adds}

\end{figure}

We then transform these ADDs into a boolean representation (BDDs) by applying the gap threshold $G$. 
Terminal nodes in the ADD are checked: if value $> G$, the path maps to $1$ (sensitive).
The resulting BDDs are shown in Figure~\ref{fig:final_bdds}.

It is important to recognize that a single satisfying assignment in this resulting BDD corresponds to \emph{two} distinct 
sensitive regions in the original input space—one for each mask involved in the subproblem pair.
Since the decomposition process generates subproblems for adjacent mask pairs, the resulting sets of sensitive regions 
frequently overlap.
Figure~\ref{fig:final_bdds} demonstrates this phenomenon (e.g., between Subproblem 1 and Subproblem 3, which share mask $m_2$).

\begin{figure}[!htbp]
\centering
\begin{subfigure}{0.32\textwidth}
    \includegraphics[width=\textwidth]{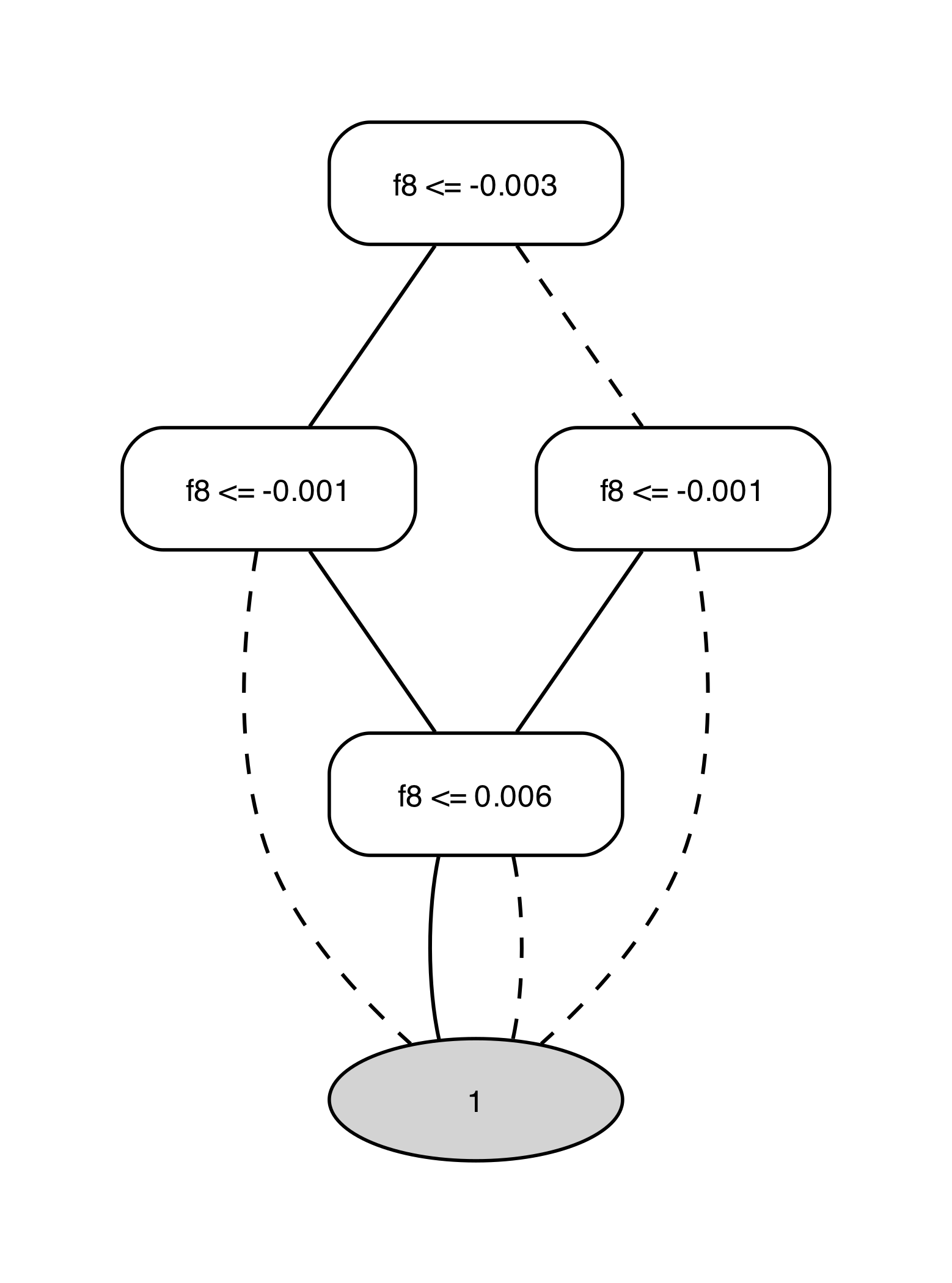}
    \caption{Subproblem 1}
\end{subfigure}
\hfill
\begin{subfigure}{0.32\textwidth}
    \includegraphics[width=\textwidth]{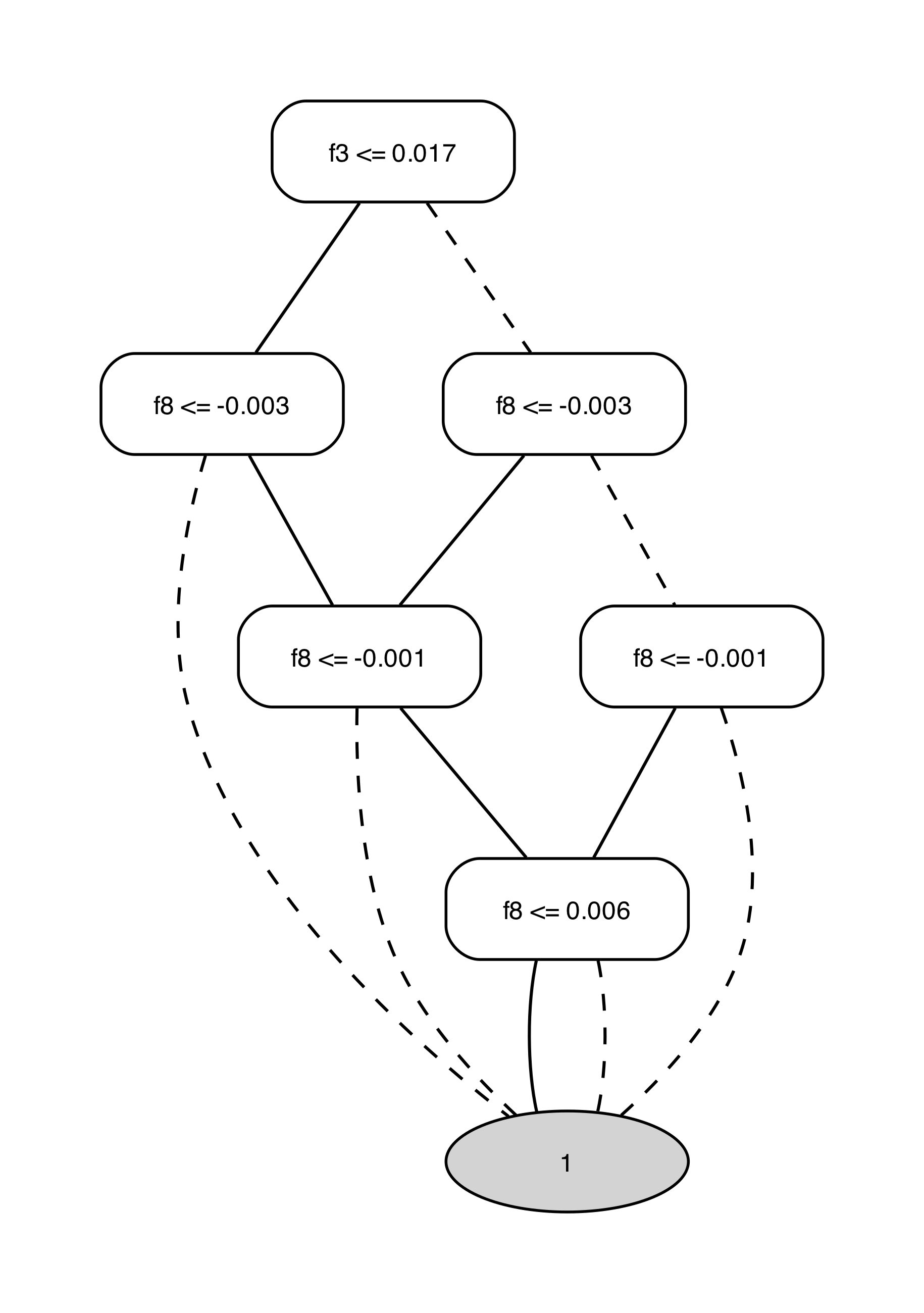}
    \caption{Subproblem 3}
\end{subfigure}
\begin{subfigure}{0.32\textwidth}
    \includegraphics[width=\textwidth]{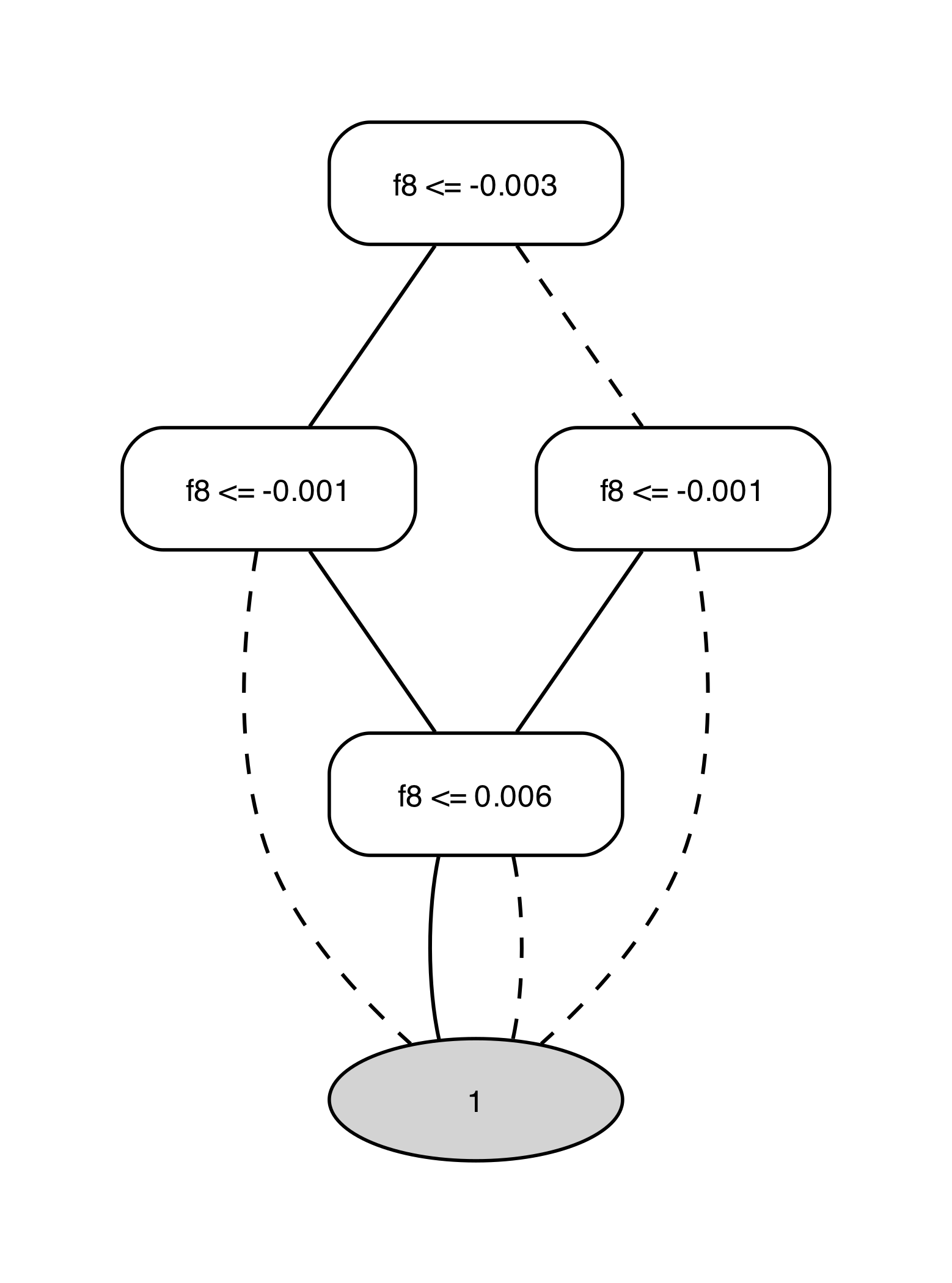}
    \caption{Subproblem 5}
\end{subfigure}
\caption{Final BDDs for subproblems 1, 3 and 5, where the ADD did not evaluate to 0 upon applying the threshold.
The rest of the ADDs resulted in BDDs that evaluate to 0 as all the leaf values were below the gap threshold $G=0.1$.
Each BDD represents the sensitive regions found for a specific pair of bitmasks. Overlaps between these BDDs 
necessitate the use of the Pepin counting technique.}
\label{fig:final_bdds}
\end{figure}

To address the overlap, $\xcount$ utilizes the sampling-based Pepin counting technique (detailed in Section~\ref{sec:algorithm}).
Conceptually, this involves a global solution set (referred to as $X$ in the algorithm) where satisfying assignments from the BDDs of different subproblems are aggregated.
The algorithm statistically estimates the size of the union of these sets to produce the final count of \emph{unique} sensitive regions.

\end{document}